\documentclass{article}

\usepackage{amsmath}

\usepackage{microtype}
\usepackage{graphicx}
\usepackage{subfigure}
\usepackage{booktabs} 
\usepackage{enumitem}
\usepackage{bm}
\usepackage{comment}

\usepackage{natbib}

\usepackage{hyperref}
\usepackage[accepted]{icml2025}

\usepackage{amsmath}
\usepackage{amssymb}
\usepackage{mathtools}
\usepackage{amsthm}

\usepackage[capitalize,noabbrev]{cleveref}

\theoremstyle{plain}
\newtheorem{thm}{Theorem}

\newtheorem{lemma}{Lemma}
\newtheorem{defs}{Definition}
\newtheorem{asmp}{Assumption}
\newtheorem{cor}{Corollary}

\usepackage[textsize=tiny]{todonotes}
\def\1{{\bf 1}}
\def\0{{\bf 0}  }

\newcommand{\Mean}{{\mathbb{E}}}
\newcommand{\Var}{{\mbox{Var}}}
\newcommand{\Cov}{{\mbox{Cov}}}

\newcommand{\prob}{{\mathbb{P}}}
\newcommand{\ATE}{\textrm{ATE}}

\usepackage{comment}

\icmltitlerunning{Unraveling the Interplay between Carryover Effects and Reward Autocorrelations in Switchback Experiments}

\begin{document}

\twocolumn[{
\icmltitle{Unraveling the Interplay between Carryover Effects and Reward Autocorrelations in Switchback Experiments}

\icmlsetsymbol{equal}{*}

\begin{icmlauthorlist}
\icmlauthor{Qianglin Wen}{equal,yunnanU}
\icmlauthor{Chengchun Shi}{equal,Lse}
\icmlauthor{Ying Yang}{camfdU}
\icmlauthor{Niansheng Tang}{yunnanU}
\icmlauthor{Hongtu Zhu}{unc}
\end{icmlauthorlist}

\icmlaffiliation{Lse}{Department of Statistics, London School of Economics and Political Science}
\icmlaffiliation{yunnanU}{Yunnan Key Laboratory of Statistical Modeling and Data Analysis, Yunnan University}
\icmlaffiliation{camfdU}{Center for Applied Mathematics, Shanghai Key Laboratory for Contemporary Applied Mathematics, Fudan University}
\icmlaffiliation{unc}{Department of Biostatistics, The University of North Carolina at Chapel Hill}
\icmlcorrespondingauthor{Hongtu Zhu}{htzhu@email.unc.edu}

\icmlkeywords{A/B testing; policy evaluation; reinforcement learning; switchback designs}
\vskip 0.2in
}]
\begin{abstract}
A/B testing has become the gold standard for policy evaluation in modern technological industries. Motivated by the widespread use of switchback experiments in A/B testing, this paper conducts a comprehensive comparative analysis of various switchback designs in Markovian environments. 
Unlike many existing works which derive the optimal design based on specific and relatively simple estimators, our analysis covers a range of state-of-the-art estimators developed in the reinforcement learning (RL) literature. It reveals that the effectiveness of different switchback designs depends crucially on (i) the size of the carryover effect and (ii) the auto-correlations among reward errors over time. Meanwhile, these findings are estimator-agnostic, i.e., they apply to most RL estimators.  Based on these insights, we provide a workflow to offer guidelines for practitioners on designing switchback experiments in A/B testing.
\end{abstract}
\printAffiliationsAndNotice{\icmlEqualContribution}
\section{Introduction}\label{sec::intro}
\textit{\textbf{Motivation.}}  Policy evaluation has become increasingly important in applications such as economics \citep{athey2017state}, medicine \citep{luedtke2016statistical}, environmental science \citep{reich2021review} and epidemiology \citep{hudgens2008toward}. In the technology sector, companies such as Google, Amazon, Netflix, and Microsoft extensively use A/B testing to measure and improve the effectiveness of new products or strategies against established ones \citep[e.g.,][]{johari2017peeking, waudby2022anytime}. 
For instance, ridesharing platforms like Uber, Lyft and Didi continuously refine their policies for order dispatching and subsidies to optimize key metrics such as supply-demand balance, driver earnings, response rates, and order completion rates \citep{qin2025reinforcement}. A/B testing is a common tool used by these companies for policy evaluation and is vital for identifying the most effective strategies to enhance the efficiency and convenience of the transportation system \citep{Xu2018,Tang2019,zhou2021graph}.

\textit{\textbf{Challenges.}} In numerous applications, treatments are assigned sequentially over time \citep{robins1986new, bojinov2019time}, posing unique challenges for A/B testing:
\begin{enumerate}[leftmargin=*]
\item A primary challenge is the \textbf{carryover effect} -- the effect of previous treatments on future outcomes. 
Such effects are ubiquitous in many applications. For instance, in ridesharing, the implementation of a specific order dispatch policy can change the spatial distribution of drivers in a city, impacting subsequent outcomes \citep[see][Figure 2]{li2023evaluating}. These carryover effects substantially challenge A/B testing. Standard solutions like the two-sample $t$-test often fail to capture these effects, frequently resulting in insignificant $p$-values \citep{shi2023dynamic}. 
\item Another challenge is the \textbf{limited sample size}, coupled with generally \textbf{weak treatment effects}, making it extremely difficult to determine the most effective policy. This issue is particularly prevalent in online experiments in ridesharing, which seldomly last more than two weeks and typically exhibit effect sizes ranging from $0.5\%$ to $2\%$ \citep{Xu2018,Tang2019}.
\end{enumerate}
\textit{\textbf{Contributions. }}This paper conducts a quantitative analysis to understand the effects of various switchback designs on the precision of their resulting policy value estimators. Switchback designs alternate between a baseline and a new policy at fixed intervals. Each policy is implemented for a specified duration before transitioning to the other. 
When the duration of each policy extends to a full day, the design becomes an ``alternating-day'' (AD) design, involving daily policy switches. 
These designs are increasingly utilized in large-scale ridesharing platforms \citep{Xiong2023,qin2025reinforcement}. \citet{luo2022policy} has empirically demonstrated that more frequent policy alternations can reduce mean squared error (MSE) in estimating the average treatment effect (ATE). However, the mechanisms driving the improvement in estimation accuracy 
lack thorough exploration. 
Our study fills this gap by offering a comprehensive comparative analysis of switchback experiments within a reinforcement learning \citep[RL,][]{sutton2018reinforcement} framework, where the experimental data follows a Markov decision process \citep[MDP,][]{puterman2014markov} model. 

Importantly, our analysis unravels the interplay between carryover effects and reward auto-correlations in determining the optimal switchback experiment. In particular, when the \textbf{carryover effect is weak}, we show that: 
\begin{enumerate}[leftmargin=*]
	\item[(i)] In scenarios with \textbf{positively correlated} reward errors, the precision of the ATE estimator tends to be improved with more frequent alternations between policies. This leads us to an interesting conclusion: the off-policy ATE estimator under switchback designs outperforms its on-policy counterpart under the AD design in terms of estimation efficiency. This conclusion remains valid even in the presence of some negatively correlated errors. The superior efficiency of the switchback design is attributed to the switchback design’s inherent capability to neutralize the influence of autocorrelated errors over time, leading to a more accurate estimator. This insight has not been systematically documented in existing literature, to our knowledge. We also remark that positively autocorrelated errors are commonly observed in practice, as demonstrated in Figure \ref{fig0}.
	\item[(ii)] When reward errors are \textbf{uncorrelated}, all designs become asymptotically equivalent in theory. Our numerical studies indicate that AD generally exhibits superior performance in finite samples.
	\item[(iii)] When the majority of errors are \textbf{negatively correlated}, AD becomes the most efficient.
\end{enumerate}
Additionally, with \textbf{a large carryover effect}, AD or switchback designs with less frequently switches work the best. 

These findings  apply to a range of policy value estimators developed in the RL literature, such as model-based estimators, least squares temporal difference (LSTD) estimators, and double reinforcement learning (DRL) estimators \citep[see][for a review]{uehara2022review}. While existing works also studied switchback designs (refer to the next section), they primarily focused on the use of simple importance sampling (IS) estimators for policy evaluation. According to our numerical studies, IS estimators suffer from much larger MSEs than ours (see Figure \ref{fig:compare_logmse_linear_nonlinear} for details). 

\section{Related Works}\label{sec:relatedwork} 
Our paper intersects with three related lines of research: A/B testing, off-policy evaluation and experimental designs.

\textit{\textbf{A/B testing.}} A/B testing has been widely adopted across tech companies \citep[see][for reviews of methodologies]{larsen2023statistical, quin2023b}. It relies on causal inference to estimate treatment effects, typically assuming ``no interference'' or the stable unit treatment value assumption \citep[SUTVA, see e.g.,][]{imbens2015causal}. However, as noted earlier, this assumption can be violated in temporally-dependent experiments, rendering most A/B testing methods unsuitable for switchback designs.

\begin{figure}[t]
	\centering
	\vspace{-6pt} 
	\includegraphics[width=0.8\linewidth]{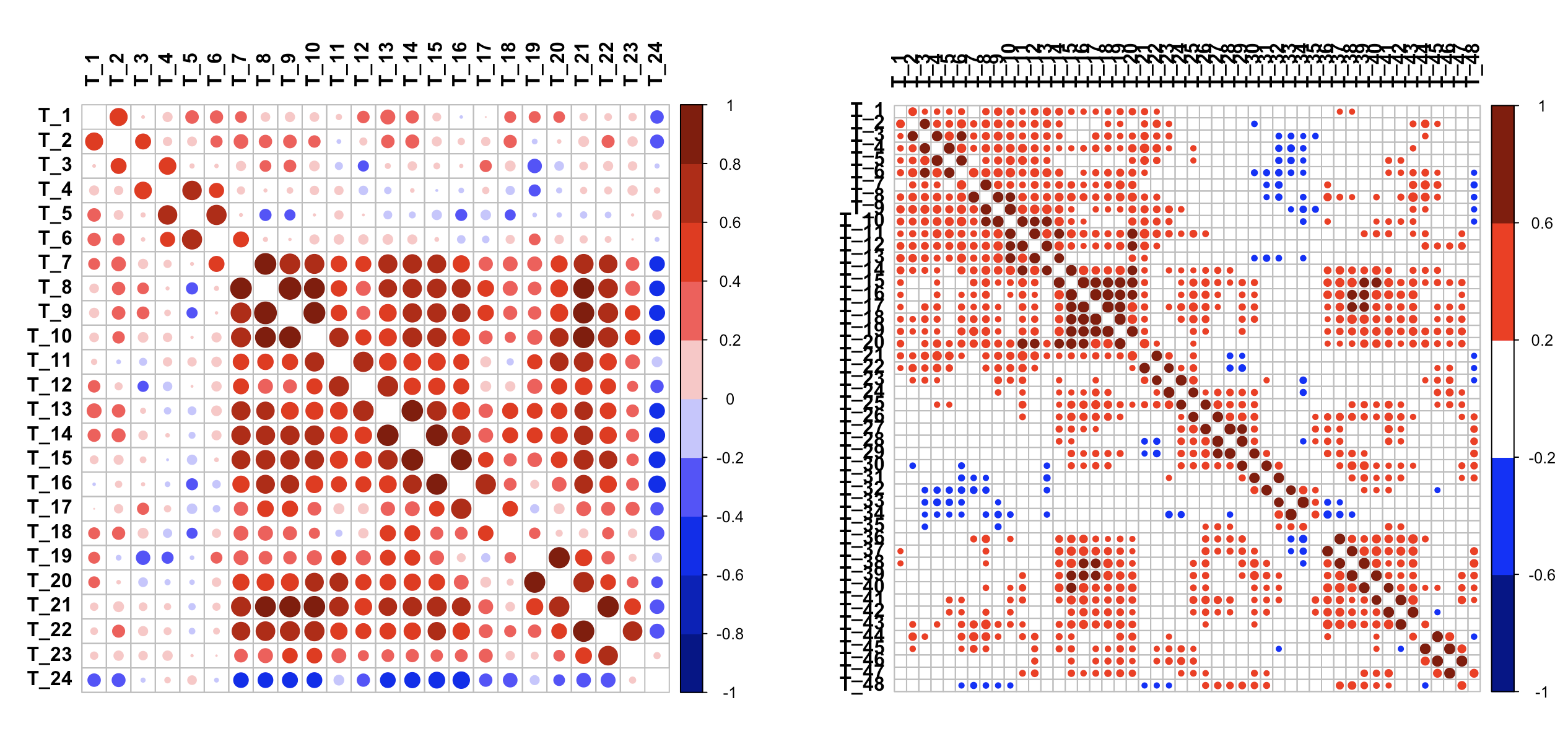} 
	\vspace{-6pt}
	\caption{\small The estimated correlation coefficients between pairs of fitted reward residuals, based on two datasets provided by a ridesharing company. Most residual pairs are non-negatively correlated, with a large proportion exhibiting positive correlation. The diagonal components have been omitted to enhance clarity.}
	\label{fig0}
\end{figure}

Recently, there is a growing interest in developing A/B testing and/or causal inference solutions in the presence of interference effects. Depending on the type of interference, these papers can be grouped into four categories:
\begin{itemize}[leftmargin=*]
    \item The first category studies spatial or network interference effects where the policy implemented in one location can affect outcomes at other locations  \citep[see][for some recent proposals]{pollmann2020causal,tchetgen2021auto,leung2022graph,bhattacharya2024causal,dai2024causal,jia2024multi,causal_mess_network,zhang2024spatial}.
    \item The second category focuses on temporal carryover effects and is the most relevant to our proposal \citep[see e.g.,][]{robins1986new,sobel_causal_2014,boruvka2018assessing,liang2025randomization,viviano2023synthetic}. Notably, there is a line of papers that proposed to employ the RL framework that models the observed data via MDPs to capture the carryover effects \citep{farias2022markovian,farias2023correcting,shi2023dynamic}.  
    \item The third category handles both interference effects over time and space \citep[see e.g.,][]{jia2023clustered,shi2023multiagent}.
    \item Finally, the last category focuses on interference effects that appear in two-sided markets or recommender systems \citep{munro2021treatment,johari2022experimental,zhan2024estimating}.
\end{itemize}
However, the effectiveness of different experimental designs are less explored in these papers, which is our primary focus.

\textit{\textbf{Off-policy evaluation (OPE).}} OPE methods generally fall into two categories: model-based and model-free approaches. Model-based methods estimate an MDP model from offline data and compute policy value based on the estimated model \citep{gottesman2019combining,yin2020asymptotically,wang2024off,yu2024two}. On the other hand, model-free methods can be further classified into three subtypes:
(i) Value-based approaches, such as LSTD, focus on estimating the policy value via an estimated Q-function
or value function \citep{bradtke1996linear,sutton2008convergent,luckett2019,hao2021bootstrapping,liao2021ope_jasa,chen2022well,shi2022statistical,li2023sharp,liu2023online,cao2024orthogonalized,bian2025off}; (ii) IS-type methods adjust rewards by the density ratio between target and behavior policies \citep{thomas2015,liu2018breaking,nachum2019dualdice,xie2019towards,dai2020coindice,wang2021projected,Hu2023off,thams2023statistical,zhou2025IS}; (iii) Doubly-robust methods such as DRL combine value-based and IS methods for more robust policy evaluation \citep{zhang2013robust,jiang2016doubly,thomas2016data,bibaut2019more,uehara2020minimax,kallus2020double,kallus2022efficiently,liao2022batch,xu2023instrumental,shi2024off}. 

Despite the popularity of developing advanced OPE estimators, the strategies for generating
offline data to maximize their estimation efficiency have not been thoroughly investigated.
Existing works either focus on a contextual bandit setting without carryover effects \citep{wan2022safe} or  do not study switchback designs \citep{hanna2017data,mukherjee2022revar,zhong2022robust,li2023optimal}. Our paper fills this gap by investigating how switchback designs affect the efficiency of various OPE estimators, including model-based, value-based, and doubly-robust estimators, providing a comprehensive analysis that enriches the OPE literature.

\textit{\textbf{Experimental design}}. There is a rich literature on experimental designs tailored for clinical trials, with a range of proposed optimal designs  \citep{begg1980treatment,wong2008optimum,jones2009d,atkinson2017optimum,rosenblum2020optimal} and sequential adaptive designs \citep{hu2009efficient,baldi2011covariate,atkinson2013randomised,hu2015unified,kato2024active} to guide treatment allocation strategies. However, these methods are developed under settings where data are identically and independently distributed and are thus not applicable to our settings in the presence of carryover effects. 

Recent developments have expanded the scope to accommodate spatial or network spillover effects \citep{ugander2013graph,li2019randomization,kong2021approximate,leung2022rate,ni_design_2023,viviano2023causal,yang2024spatially,zhang2024online,zhu2025causalgraphcut} and to address the complex interactions inherent in two-sided marketplaces  \citep{bajari2021multiple,li2022interference,bajari2023experimental}. Despite these advancements, a gap remains concerning designs that adequately account for temporal carryover effects in sequential decision making.

\citet{glynn2020adaptive}, \citet{hu2022switchback}, \citet{bojinov2023design}, \citet{basse2023minimax}, \citet{sun2024optimal} and \citet{Xiong2023} studied the design of temporally-dependent experiments. In particular, \citet{Xiong2023} made an important step forward for understanding the trade-offs among switchback designs by deriving a rigorous bias-variance decomposition of the ATE estimator and summarizing four key factors that determine the estimation error. However, their analysis is confined to simple IS estimators derived within a bandit framework, which can be severely biased with large carryover effects. Our empirical studies in Appendix \ref{sec:addresults} also confirm that IS-type estimators are prone to large MSEs. Moreover, \citet{Xiong2023}'s analysis did not adopt our MDP framework, which is commonly utilized for policy learning and evaluation in the motivating ridesharing application \citep{Xu2018,shi2023dynamic}. 

Finally, recent works in the machine learning literature have developed deep learning or RL algorithms to numerically compute optimal designs \citep{foster2021deep,blau2022optimizing,lim2022policy}.

\section{Preliminaries}\label{sec:model}
In this section, we describe the data, detail our model, introduce switchback designs and formulate our objective.

\textit{\textbf{Data. }} Suppose a technology company conducts an online experiment over $n$ days to evaluate the effectiveness of a new policy compared to a baseline policy. Each day is divided into $T$ non-overlapping intervals. For each day $i=1,\ldots, n$ and for each time interval $t$, let $S_{i,t}\in \mathbb{R}^d$ denote certain market features (e.g., the number of available drivers and pending ride requests in ridesharing) observed at the beginning of the interval. 
The policy in effect during each interval $t$ is represented by $A_{i,t}$, which, in the context of A/B testing, is a binary variable indicating one of two policies. Finally, $R_{i,t}\in \mathbb{R}$ denotes the immediate outcome or reward observed at the end of each interval $t$ (e.g., the total revenue at time $t$). We assume all trajectories $\{(S_{i,t}, A_{i,t}, R_{i,t}):1\le t\le T\}_{i=1}^n$ are i.i.d. instances of a stochastic process $\{(S_t,A_t,R_t):1\le t\le T\}$. 
That is, the data are independent across different days. This assumption is likely to hold in applications such as marketing auctions where each company’s budget resets at the end of the day, eliminating any carryover effects across days  \citep{basse2016randomization, liu2020trustworthy}, and in ridesharing where order volume typically wanes between 1 am and 5 am \citep[Figure 1]{luo2022policy}, making it plausible that each day's observations can be treated as independent realizations. Additionally, in online advertising, impression allocation often follows a daily schedule, reinforcing the assumption of independent data across days. 

\textit{\textbf{Model.}}
We model the experimental data by a finite MDP with autocorrelated errors, based on three assumptions:
\vspace{-3mm}
\begin{enumerate}[leftmargin=*]
	\item[(i)] First, we assume that the state satisfies a Markov assumption. Specifically, we require
\begin{equation*}
\mathbb{P}(S_{t+1}=s'|A_t,S_t,\{S_j,A_j,R_j\}_{j<t})=p_t(s'|A_t,S_t),
\end{equation*}
for any $s'$ and $t$. This assumption requires the future state to be conditionally independent of the past data history given the current state-action pair, and is consistent with a wide body of work in RL \citep{sutton2018reinforcement}. 
	\item[(ii)] Second, we assume the reward satisfies a conditional mean independence assumption: there exists a sequence of reward functions $\{r_t\}_t$ such that for any $a$ and $s$,
	\begin{equation*}
		\Mean (R_t|A_t=a,S_t=s,\{S_j,A_j\}_{j<t})=r_t(a,s).
	\end{equation*}
	Such an assumption is commonly imposed in the literature \citep{chernozhukov2022automatic,shi2022statistical,wang2021projected}.
	\item[(iii)] Third, the residual errors $e_t=R_t-r_t(A_t,S_t)$ can exhibit temporal correlation.  
	If the residuals are uncorrelated, the resulting data-generating process simplifies to a standard MDP. \vspace{-0.5em}
\end{enumerate}

\begin{figure}[t]
	\centering
	\includegraphics[width=0.4\textwidth,height=0.2\textwidth]{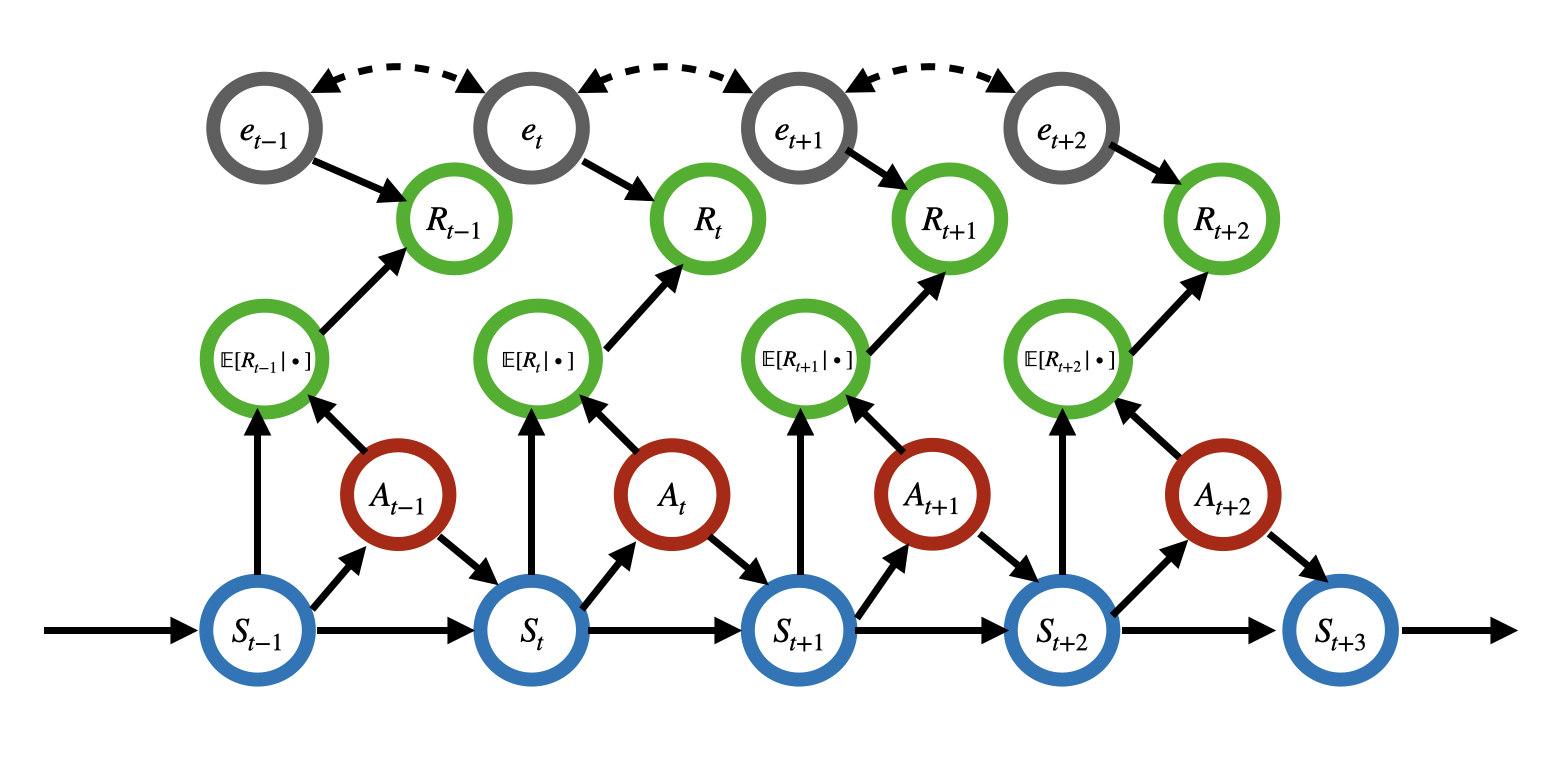}\vspace{-1em}
	\caption{{Visualization of our MDP with autocorrelated reward errors. The solid lines represent the causal relationships. The dash lines imply that the reward errors are potentially correlated.}}\label{fig1}
        \vspace{-3mm} 
\end{figure}
A graphical visualization of our model is given in Figure \ref{fig1}. Notably, while our methods rely on the MDP assumption, we also consider cases where this assumption may not hold; refer to Appendix \ref{sec:add_discussion} for details. 

Finally, we remark that both the reward and transition functions are explicitly indexed by $t$. This is essential to capture the time-dependent dynamics that  are often inherent in practical applications \citep{luo2022policy}.

\textit{\textbf{Designs. }}We introduce the \textit{switchback design} frequently utilized in practice. Under this design, the company alternates between the two policies, each for a fixed duration per day. Let 
$m\ge 1$ represent the time span for each switch. A smaller value of $m$ corresponds to more frequent switching between policies. To illustrate this design, consider the following examples: 
\vspace{-3mm}
\begin{itemize}[leftmargin=*]
		  \setlength{\itemsep}{0pt} 
		  \setlength{\parsep}{0pt}   
		  \setlength{\parskip}{0pt}  
	\item For $m=1$, the policies alternate at every time step, formally expressed as $A_t=1-A_{t+1}$ for any $t\ge 1$; 
	\item For a general $2\le m\le T/2$, the policy remains constant for $m$ time steps and then switches, which can be mathematically represented as $A_{mt-m+1}=A_{mt-m+2}=\ldots=A_{mt}=1-A_{mt+1}=\ldots=1-A_{mt+m}$; 
	\item For $m=T$, the same policy is applied throughout each day,  i.e., $A_1=A_2=\ldots=A_T$. 
\end{itemize}
\vspace{-3mm}
Furthermore, the initial policy alternates across days, which can be mathematically described as $A_{i,1}=1-A_{i+1,1}$ for any $i$, 
and the initial policy on the first day $A_{1,1}$ is uniformly generated. 
Consequently, when $m=T$, the design essentially becomes an \textit{alternating-day} scheme where the two policies are switched on a daily basis.

Moreover, to improve clarity and facilitate a more intuitive comparison between our SB design and the standard A/B testing, we present a detailed discussion in Appendix    \ref{sec:add_discussion}.

\textit{\textbf{Objective. }}We aim to estimate the global average treatment effect, defined as the difference between the average cumulative rewards when implementing the new policy throughout each day and that when using the baseline policy, 
\begin{equation*}
	\textrm{ATE}=\frac{1}{T}\sum_{t=1}^T \Mean^1 (R_t)-\frac{1}{T}\sum_{t=1}^T \Mean^0 (R_t),
\end{equation*}
where $\Mean^1$ and $\Mean^0$ represent the expectation when the new policy (coded as $1$) and the baseline policy (coded as $0$) are applied at all times, respectively. This metric is particularly relevant for A/B testing in RL \citep{tang2022reinforcement}.

In standard terms, an AD design operates under an \textit{on-policy} framework where within each day, the behavior policy generating the experimental data aligns with the target policy under evaluation (which assigns a constant action, either 0 or 1, at each time). Conversely, when $m\neq T$, the switchback (SB) design operates under an \textit{off-policy} framework where the behavior policy differs from the target policy. 

In standard MDPs, off-policy estimators are considered less efficient than on-policy ones due to the distributional shift between the behavior and target policies \citep{li2023optimal}. 
However, in the next section, we will demonstrate a trade-off between AD and SB. This distributional shift issue becomes predominant when the carryover effect is large, thus favouring AD in such settings. Conversely, when the carryover effect is small, the effectiveness of different designs depends crucially on the autocorrelations among reward errors. In particular, when the majority of reward errors are positively autocorrelated, SB can actually outperform an AD. 
These findings provide valuable insights for technology
companies to optimize their A/B testing strategies.

\section{Optimality of Switchback Designs}\label{sec:switchbackdesign}
This section investigates the effectiveness of switchback designs when applied to various ATE estimators. We begin with a toy example to offer insight into why off-policy estimators under switchback designs can be more effective than their on-policy counterparts under alternating-day designs. 

\textit{\textbf{Toy example.}} To facilitate understanding, we consider the following simple model without carryover effects over time: $R_{t}=\beta_1 A_t+\beta_0 (1-A_t)+e_t$.
By definition, the oracle ATE equals $\beta_1-\beta_0$. A natural plug-in estimator for the ATE is given by: $    \frac{\sum_{i,t} R_{i,t}A_{i,t}}{\sum_{i,t} A_{i,t}}-\frac{\sum_{i,t} R_{i,t}(1-A_{i,t})}{\sum_{i,t} (1-A_{i,t})}$.
Upon calculation, it can be shown that its MSE under AD and a particular SB design with $m=1$ is proportional to
\begin{eqnarray*}
    \textrm{MSE}(\textrm{AD})\propto \Var(e_1+e_2+e_3+\cdots+e_T),\\
    \textrm{MSE}(\textrm{SB}^{(1)})\propto \Var(e_1-e_2+\cdots-(-1)^Te_T),
\end{eqnarray*}
respectively. 
Based on these formulas, the ATE estimator's MSE depends solely on the correlation structure of the residuals $\{e_t\}_t$. In the AD design, the MSE is proportional to the variance of the sum of all residuals, which inflates when errors are positively correlated. In contrast, the SB design’s MSE is determined by the variance of a weighted sum of residuals with alternating signs, effectively canceling positively correlated errors and improving ATE estimation accuracy. 
We extend this analysis to incorporate carryover effects and examine various policy value estimators, including model-based estimators, LSTD, and DRL.

\subsection{Methods}\label{subsec:methods}
This section is organized as follows. We first study linear-model-based estimators. We next analyze the LSTD estimator, a popular value-based estimator. Finally, we consider the DRL estimator, an advanced model-free estimator known for its double robustness.

\textbf{\textit{Model-based method}}. In model-based approach, we assume a system dynamics model and utilize this model to construct the ATE estimator. In particular, we apply linear models to both the reward function and the expected value of the next state, resulting in the following set of linearity assumptions:
\begin{equation}\label{linear_mdp}
	\left\{ 
	\begin{split}
		r_t(A_t,S_t)&=\alpha_t + S_t^\top \beta_t + \gamma_t A_t, \\
		\Mean (S_{t+1}|A_t,S_t)&=\phi_t + \Phi_t S_t + \Gamma_t A_t,
	\end{split}
	\right. 
\end{equation}
where $\alpha_t$ and $\gamma_t$ are real-valued, $\beta_t,\phi_t,$ and $\Gamma_t$ are vectors in  $\mathbb{R}^d$, and $\Phi_t \in \mathbb{R}^{d\times d}$.

We make two remarks. First, the model presented in \eqref{linear_mdp} resembles the linear dynamic system model commonly found in linear-quadratic-Gaussian control problems \citep{krishnamurthy2016partially}. It is also consistent with the linear MDP assumption -- a condition frequently employed in the RL literature \citep{jin2020provably,li2021sample,xie2023semiparametrically}.

Second, under Model \eqref{linear_mdp}, as outlined in \citet{luo2022policy}, the ATE can be expressed as
\begin{eqnarray}\label{ate_est_formula}
\frac{1}{T}\sum_{t=1}^T \gamma_t+\frac{1}{T}\sum_{t=2}^T \beta_t^\top \Big[ \sum_{k=1}^{t-1} (\Phi_{t-1}\Phi_{t-2}\ldots\Phi_{k+1}) \Gamma_k \Big],
\end{eqnarray}
where the product $\Phi_{t-1}\ldots \Phi_{k+1}$ is treated as an identity matrix if $t-1<k+1$. The first term on the right-hand side (RHS) of (\ref{ate_est_formula}) represents the direct effect of actions on immediate rewards, while the latter term accounts for the delayed or carryover effects of previous actions. 

Equation \eqref{ate_est_formula} motivates us to employ the ordinary least square (OLS) regression to compute the estimators $\widehat{\alpha}_t$, $\widehat{\beta}_t$, $\widehat{\gamma}_t$, $\widehat{\phi}_t$, $\widehat{\Gamma}_t$ and $\widehat{\Phi}_t$ and plug them into \eqref{ate_est_formula} to compute the final ATE estimator. Refer to Appendix \ref{sec:implementation} for details.

\textbf{\textit{LSTD}}. LSTD is a popular model-free, value-based OPE estimator. To illustrate the LSTD estimator, we first introduce the notion of value function. For any time $t \geq 1$, action $a \in \{0,1\}$, and state $s$, the value function $V_t^a(s)$ represents the expected cumulative return from time $t$ in state $s$, assuming the agent follows a constant action $a$ and can be mathematically expressed as: $V_t^a(s) = \sum_{j=t}^T \Mean^a (R_j | S_t = s)$. The ATE can then be equivalently represented by $T^{-1} \Mean [V_1^1(S_1) - V_1^0(S_1)]$. LSTD computes an estimated value function $\widehat{V}_t^{a,m}$ and approximates the expectation using empirical averages under the $m-$switchback design, leading to the following ATE estimator: 
\begin{eqnarray}\label{eqn:LSTDest}
\frac{1}{nT}\sum_{i=1}^n [\widehat{V}_1^{1,m}(S_{i,1}) - \widehat{V}_1^{0,m}(S_{i,1})].
\end{eqnarray}
We next outline the approach for estimating the value function using LSTD, which employs linear sieves \citep{grenander1981abstract}  to approximate the value function $V_t^a(s)$ by $\varphi_t^\top(s) \theta^*_{t,a}$, with a given basis function $\varphi_t$ and the associated regression coefficients $\theta^*_{t,a}$. A crucial aspect of this methodology is that these value functions follow the Bellman equation: $\Mean [R_t + V_{t+1}^a(S_{t+1}) - V_t^a(S_t) | A_t = a, S_t = s] = 0$ for every state-action pair $(s,a)$. This leads to the formulation of the following estimating equation under the $m-$switchback design:
\begin{eqnarray}\label{eqn:LSTDeet}
	\begin{aligned}
		&\frac{1}{n} \sum_{i=1}^n \varphi_t(S_{i,t}) \mathbb{I}(A_{i,t}=a) \Big[  R_{i,t} \\
        &+ \varphi_{t+1}^\top(S_{i,t+1}) \widehat{\theta}_{t+1,a,m} - \varphi_t^\top(S_{i,t}) \widehat{\theta}_{t,a,m} \Big] = 0.
	\end{aligned}
\end{eqnarray}
The coefficients $\{\widehat{\theta}_{t,a,m}\}_{t,a}$ are computed in a backward manner, as detailed in Algorithm \ref{algo:algo_lstd}. With these estimators in hand, we construct the value function estimator and plug them into \eqref{eqn:LSTDest} to derive the final ATE estimator. 
\begin{algorithm}[t]
	\caption{Estimating ATE via LSTD.}\label{algo:algo_lstd}
\begin{algorithmic}
		\STATE {\bfseries Input:} $\left\lbrace (S_{it}, R_{it}, A_{it}): 1 \leq i \leq n, 1 \leq t \leq T \right\rbrace$.
	\STATE {Set $\widehat{\theta}_{T+1,a,m}=0$}, for $a \in \{0,1\}$. 
	\FOR{$t=T$ {\bfseries to} $1$}
	\STATE Solve \eqref{eqn:LSTDeet} to obtain  $\widehat{\theta}_{t,a,m}$.
	\ENDFOR
	\STATE{\bfseries Output:} The ATE estimator \eqref{eqn:LSTDest} with the estimator $\widehat{\theta}_{1,a,m}$. 
\end{algorithmic}
\end{algorithm}

\textit{\textbf{DRL}}. The  DRL  estimator extends the double machine learning estimator, originally developed for contextual bandit settings 
\citep{chernozhukov_doubledebiased_2018} to sequential decision making. This approach combines the value-based estimator with the marginal IS estimator \citep{liu2018breaking} for more robust and efficient policy evaluation. A key feature of the DRL estimator is its double robustness: it remains consistent as long as either the estimated value function or the marginal IS ratio is consistent. Additionally, the DRL estimator is semiparametrically efficient, achieving the lowest MSE among the class of regular and asymptotically linear estimators \citep{bickel1993efficient, tsiatis2007semiparametric}.

To present the DRL estimator, we define an estimating function $\psi(\{S_t, A_t, R_t\}_t; \{V_t^a\}_{t,a}, \{\omega_t^{a,m}\}_{t,a}) $ as follows:
\begin{eqnarray*}
	V_1^1(S_1) - V_1^0(S_1)  + \sum_{t=1}^T \sum_{a=0}^1 {(-1)^{a+1}} \omega_{t}^{a,m}(A_t,S_t) \\
	\times \big[R_t + V_{t+1}^a(S_{t+1}) - V_t^a(S_t)\big].
\end{eqnarray*}
Here, the first part corresponds to the value-based estimator. The second part serves as an augmentation term, which is mean-zero according to the Bellman equation when the value function is correctly specified. This term enhances the estimator's robustness against potential model misspecification of the value function. In particular, $\omega_{t}$ corresponds to the marginalized IS ratio, which is crucial for efficient OPE in MDPs \citep{liu2018breaking}. For any action $a\in \{0,1\}$ and time $1\le t\le T$, let $p_t^a$ denote the probability mass function of $(S_t,A_t)$ under consistent application of policy $a$. Additionally, let $p_t^m$ denote the probability mass function of $(S_t,A_t)$ under an AD or SB design. The marginalized IS ratio $\omega_t^{a,m}(s,a')$ is defined as  $p_t^a(s,a')/p_t^m(s,a')$. It can be shown that $\psi$ is unbiased to the ATE if either $\{V_t^a\}_{t,a}$ or $\{\omega_t^{a,m}\}_{t,a}$ is correctly specified. 

Notice that the value function and marginalized IS ratio can be estimated using any advanced RL algorithms. We plug their estimators into the estimating function $\psi$, and utilize sample-splitting and cross-fitting \citep{chernozhukov_doubledebiased_2018} to construct the final DRL estimator. The detailed estimating procedure is summarized in Algorithm \ref{algo:algo_drl}.

\begin{algorithm}[t]
	\caption{Estimating ATE via DRL.}\label{algo:algo_drl}
	\begin{algorithmic}
		\STATE {\bfseries Input:} $\left\lbrace (S_{it}, R_{it}, A_{it}): 1 \leq i \leq n, 1 \leq t \leq T \right\rbrace$. 
		\STATE {\bfseries Step 1:} Randomly divide the data trajectories into $K$ equally-sized folds $\{\mathcal{D}_k\}_{k=1}^K$.
		\STATE{\bfseries Step 2:} For $k=1, \ldots, K$, construct estimators $\{\widehat{\omega}_{t,-k}^{a,m}\}_{t, a}$ and  $ \{\widehat{V}_{t,-k}^{a,m}\}_{t,a}$ using all trajectories except those in $\mathcal{D}_k$.
		\STATE {\bfseries Output:} The ATE estimator
		\begin{eqnarray*}
			\frac{1}{nT}\sum_{k=1}^K \sum_{i\in \mathcal{D}_k} \psi(\{S_{i,t},A_{i,t},R_{i,t}\}_t;\{\widehat{V}_{t,-k}^{a,m}\}_{t,a},\{\widehat{\omega}_{t,-k}^{a,m}\}_{t,a}). 
		\end{eqnarray*}
	\end{algorithmic}
\end{algorithm}

\subsection{Theoretical analysis} \label{subsec:theory_sec}
We begin by introducing two key notations essential for our analysis: (i) First, let $\sigma_e(t_1,t_2)$ denote the covariance between reward residuals $e_{t_1}$ and $e_{t_2}$; (ii) Second, we define $\delta$ as the measure of the impact of the new policy on the state transition functions $\{p_t\}_t$, such that $\delta=\max_{s,t}\sum_{s'} |p_t(s'|1,s)-p_t(s'|0,s)|$. 

Notice that $\delta$ inherently quantifies the size of the carryover effect since under the MDP model, the carryover effect is modeled via state transitions; refer to Figure \ref{fig1}. In particular, when $\delta=0$, past actions have the same effects on state transitions, eliminating any carryover effect. 

We next impose a common assumption required by all the three estimators. 
\begin{asmp}[Bounded rewards]\label{asmp:reward}
	The rewards $\{R_t\}_t$ are uniformly bounded, i.e., $\max_{t} |R_t| \leq R_{\max}$ for some $R_{\max} < \infty$ almost surely.
\end{asmp}
Assumption \ref{asmp:reward} is frequently employed in the RL literature \citep[see e.g.,][]{chen2019information,fan2020theoretical}. 

In Appendix \ref{subsec:assump}, we introduce other estimator-specific assumptions. Notably, each type of estimators only requires a subset of these assumptions. These assumptions are mild and can be easily satisfied, as discussed in Appendix \ref{subsec:assump}. 

\begin{thm}\label{thm::main}
	Under the given conditions, the difference in the MSE of the ATE estimator between the alternating-day design and an $m$-switchback design (where each switch duration equals $m$) is lower bounded by
\begin{eqnarray}\label{eq:ATE_difference}
\begin{split}
\frac{16}{nT^2} \sum_{\substack{k_2-k_1=1,3,5, \ldots\\ 0\le k_1<k_2< T/m}} \sum_{l_1,l_2=1}^{m} \sigma_{e}(l_1+k_1m, l_2+k_2 m)
     \\-\frac{c\delta R_{\max}^2}{n}-o\Big(\frac{1}{n}\Big),
\end{split}
\end{eqnarray}
for some constant $c>0$ and some reminder term of the order $o(1/n)$.  
\end{thm}
Based on Theorem \ref{thm::main}, it is immediate to see that the lower bound for the difference in the MSE depends on three terms: (i) an autocorrelation term, quantifying the auto-correlations among reward errors; (ii) a carryover effect term which is proportional to $\delta$ and quantifies the magnitude of the carryover effect; (iii) a reminder term.

Here, the reminder term is a high-order, estimator-dependent term. Specifically, for the model-based estimator, it is of the order $O(n^{-3/2})$ up to some logarithmic factors, which depends on $n$ through $n^{-3/2}$ as opposed to the first two terms, which depend on 
$n^{-1}$. Consequently, the reminder term decays to zero at a much  faster rate. For LSTD and DRL, the reminder term additionally relies on 
the estimation errors of value function and/or MIS ratio. Its specific order for these estimators is detailed in Appendix \ref{thm:full_thm}.

As such, the first two terms are the primary drivers. They indicate that the effectiveness of switchback designs depends on two crucial factors: \textbf{the autocorrelation structure} and \textbf{the size of the carryover effect}. With a large carryover effect, different designs will induce substantially different state distributions, and the off-policy estimator under the switchback design will suffer from substantial distributional shifts when estimating the ATE. Mathematically, this effect is manifested by the second term in \eqref{eq:ATE_difference}, making AD the most efficient design. This observation has been empirically verified in our numerical studies.

Conversely, with a weak carryover effect where $\delta$ is sufficiently small, the first term becomes the leading term, and the effectiveness of different designs is primarily determined by the autocorrelation structure of reward errors. By the definition of $\sigma_e$, it is evident that: 
\begin{itemize}[leftmargin=*]
	\item When the majority of reward errors are positively correlated (as demonstrated in our real-data applications depicted in Figure \ref{fig0}), the first term in \eqref{eq:ATE_difference} is strictly positive. This implies that SB is more efficient than AD. Additionally, when the covariance function is stationary and satisfies $\sigma_e(t_1,t_2)=\sigma_e^*(|t_1-t_2|)$ for some $\sigma_e^*(\bullet)$ being a monotonically decreasing function, the second line becomes a monotonically decreasing function of $m$; see e.g., Corollary \ref{cor1}-\ref{cor3} below. 
	This formally verifies that increasing the frequency of policy switches (reducing the value of $m$) can enhance the efficiency of the switchback design. 
	\item With uncorrelated errors, the first term in \eqref{eq:ATE_difference} becomes zero, and all designs become asymptotically equivalent. 
	\item When the majority of errors are negatively correlated, AD becomes the most efficient. 
\end{itemize}


To the best of our knowledge, the aforementioned findings have not been systematically established in the RL literature. While existing works have studied switchback designs, they often focus on simple and specific policy value estimators. 

Next, we investigate three commonly used covariance structures -- autoregressive, moving average and exchangeable to further elaborate  Theorem \ref{thm::main}.

\begin{cor}[Autoregressive]\label{cor1}
Let $\sigma_e(t_1,t_2)=\sigma^2\rho^{|t_1-t_2|}$ for some $-1< \rho<1$ and $\sigma^2>0$. For sufficiently large $T$, the first term in \eqref{eq:ATE_difference} becomes asymptotically equivalently to
\begin{eqnarray*}
\frac{16\sigma^2\rho(1-\rho^m)}{mT(1-\rho)^2(1+\rho^m)}, 
\end{eqnarray*}
which is a strictly decreasing function of $m$ when $\rho>0$
\end{cor}

\begin{cor}[Moving average]\label{cor2}
Let $e_t=K^{-\frac{1}{2}} \sum_{k=1}^K \varepsilon_{t+k}$ for a white noise process $\{\varepsilon_t\}_t$ with $\Var(\varepsilon_t)=\sigma^2>0$. For any $m\ge K$ that divides $T$, the first term in \eqref{eq:ATE_difference} becomes
	\begin{eqnarray*}
		\frac{8\sigma^2(T/m-1)(K^2-1)}{3T^2},
	\end{eqnarray*}
which is a strictly decreasing function of $m$.
\end{cor} 

\begin{cor}[Exchangeable]\label{cor3}
Assume $\sigma_e(t_1,t_2)=\sigma^2 [\rho\mathbb{I}(t_1 \neq t_2) +\mathbb{I}(t_1 = t_2)]$ for some $-1< \rho<1$ and $\sigma^2>0$. Then the first term in \eqref{eq:ATE_difference} equals
\begin{eqnarray*}
    \left\{
    \begin{array}{ll}
        4\sigma^2 \rho, & \text{if } T/m~\text{is even}, \\ 
        4\sigma^2 \rho(1-m^2/T^2), & \text{if } T/m~\text{is odd},
    \end{array}
    \right.
\end{eqnarray*}
which is a constant function of $m$ when $T/m$ is even, and varies strictly monotonically (increasing or decreasing) as a function of $m$ when $T/m$ is odd, depending on whether $\rho>0$ or $\rho<0$.
\end{cor}
We remark that while these structures may appear simple, they are widely adopted in practice \citep{williams1952experimental, berenblut1974experimental, zeger1988regression}. 

\section{Numerical Experiments}\label{numerical_secs}
In this section, we conduct numerical experiments to verify our theory. Our code is available at \url{https://github.com/QianglinSIMON/SwitchMDP}. 
\subsection{Synthetic Environments}\label{subsec:linearDGP}
\textit{\textbf{DGP}}. We design two data generating processes (DGPs) with a common time horizon $T=48$ and state dimension $d=3$: one with a linear DGP and the other with a nonlinear DGP (refer to Appendix \ref{sec:addresults} for the detailed setup), to evaluate the performance of various switchback designs and different ATE estimators. The reward errors follow an autoregressive covariance function so that $\Cov(e_{t_1},e_{t_2})=1.5\rho^{|t_1-t_2|}$ whenever $t_1\neq t_2$, with the parameter $\rho$ varied among the set $\{0.3, 0.5, 0.7, 0.9\}$. We also vary the size of carryover effects, characterized by a parameter $\delta$.  The number of days $n$ used in our simulations is selected from a range of 16 to 52 in increments of four. 

\begin{figure}[!t]
	\centering
	\vspace{-1mm} 
	\begin{minipage}{0.49\linewidth}
		\centering
		\includegraphics[width=1\linewidth,height=0.6\linewidth]{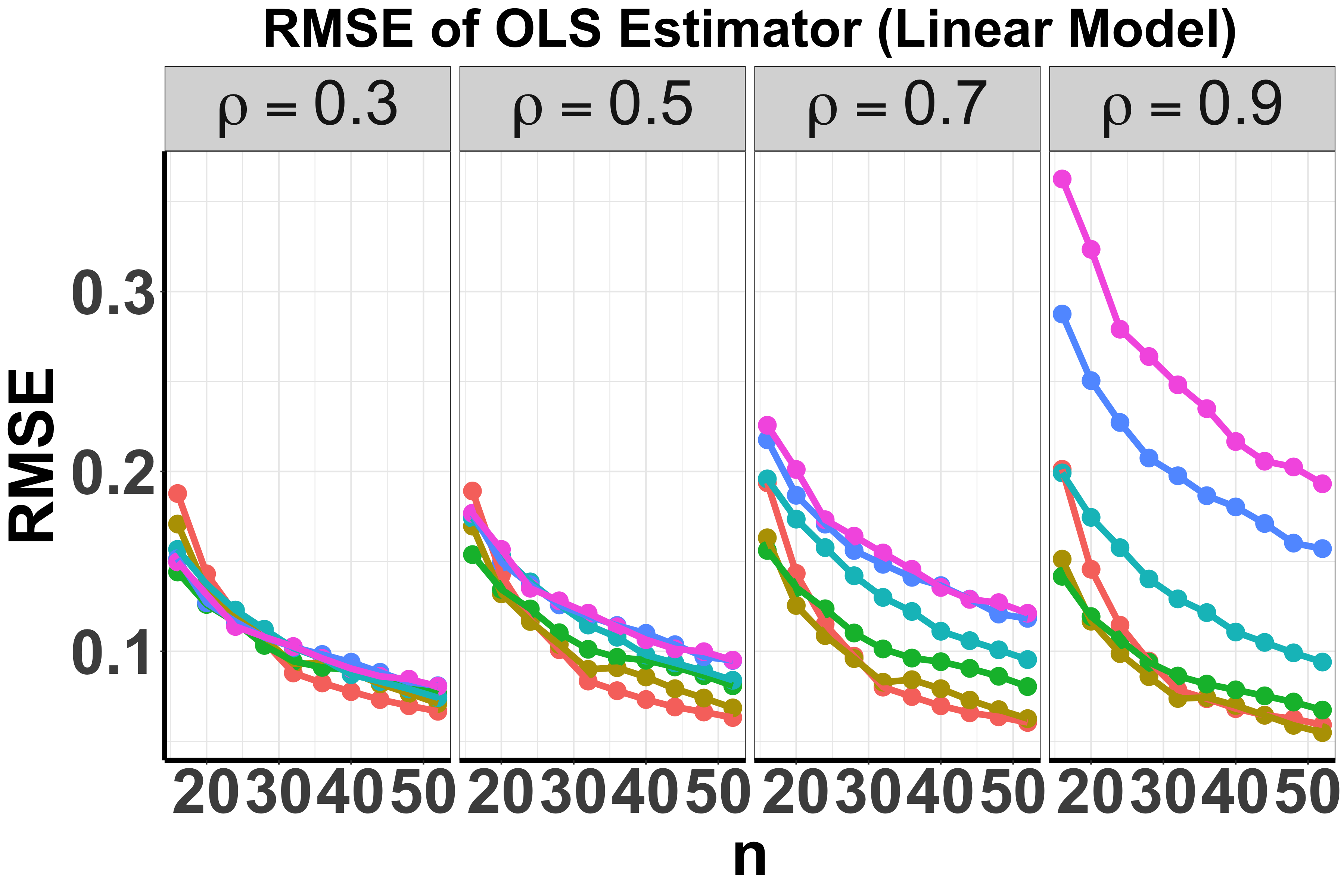}
	\end{minipage}
	\begin{minipage}{0.49\linewidth}
		\centering
		\includegraphics[width=1\linewidth,height=0.6\linewidth]{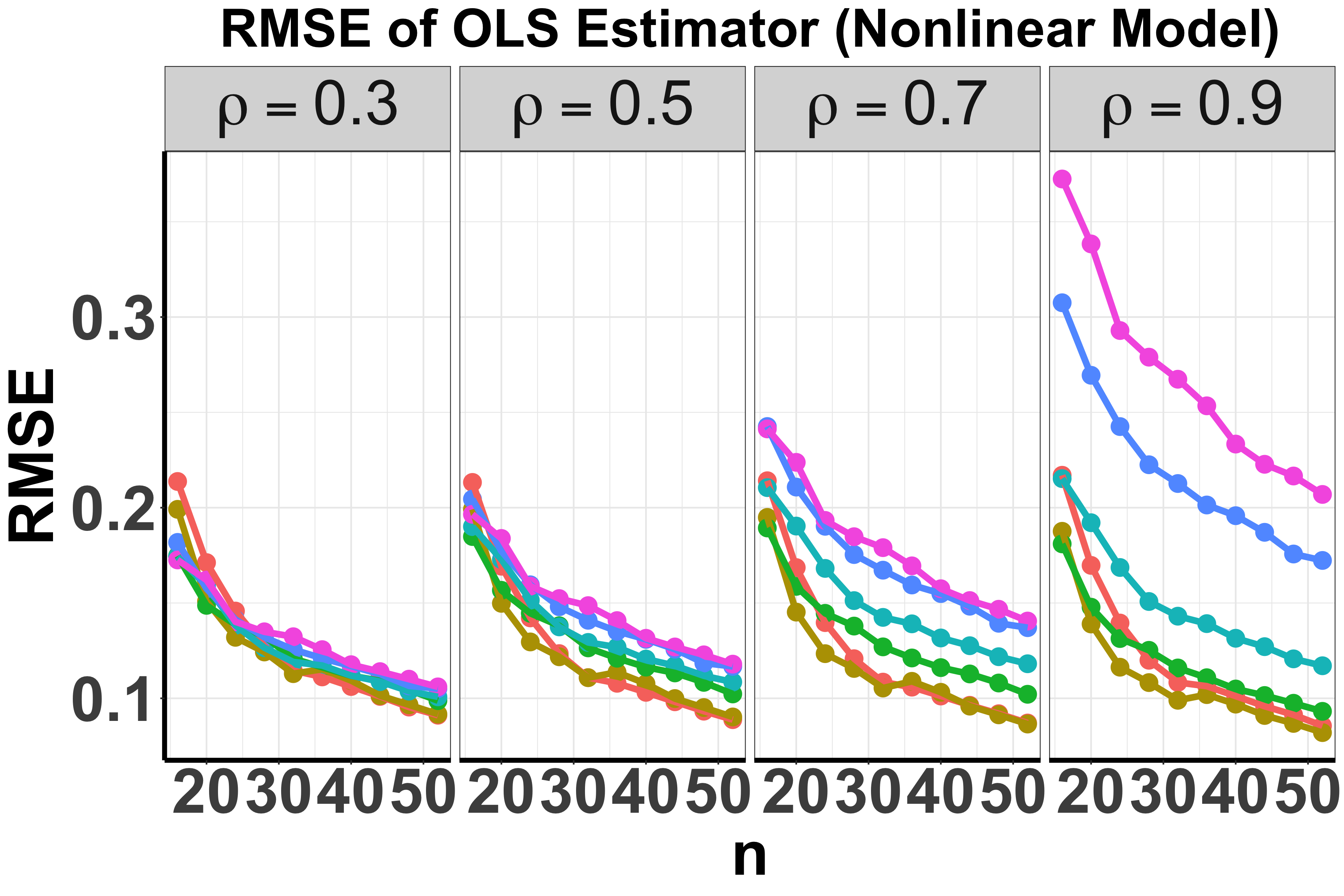}
	\end{minipage}
	\begin{minipage}{0.49\linewidth}
		\centering
		\includegraphics[width=1\linewidth,height=0.6\linewidth]{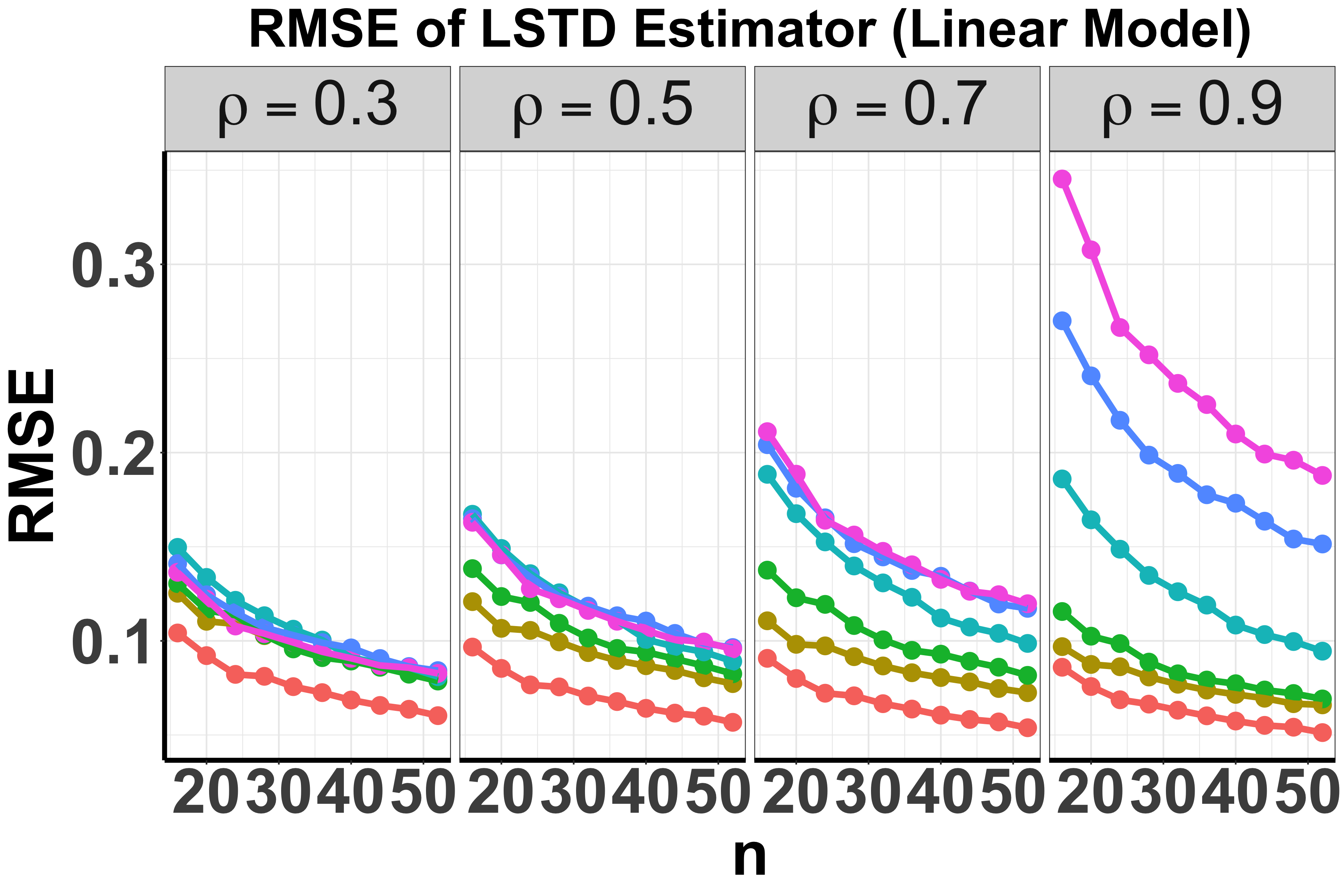}
	\end{minipage}
	\begin{minipage}{0.49\linewidth}
		\centering
		\includegraphics[width=1\linewidth,height=0.6\linewidth]{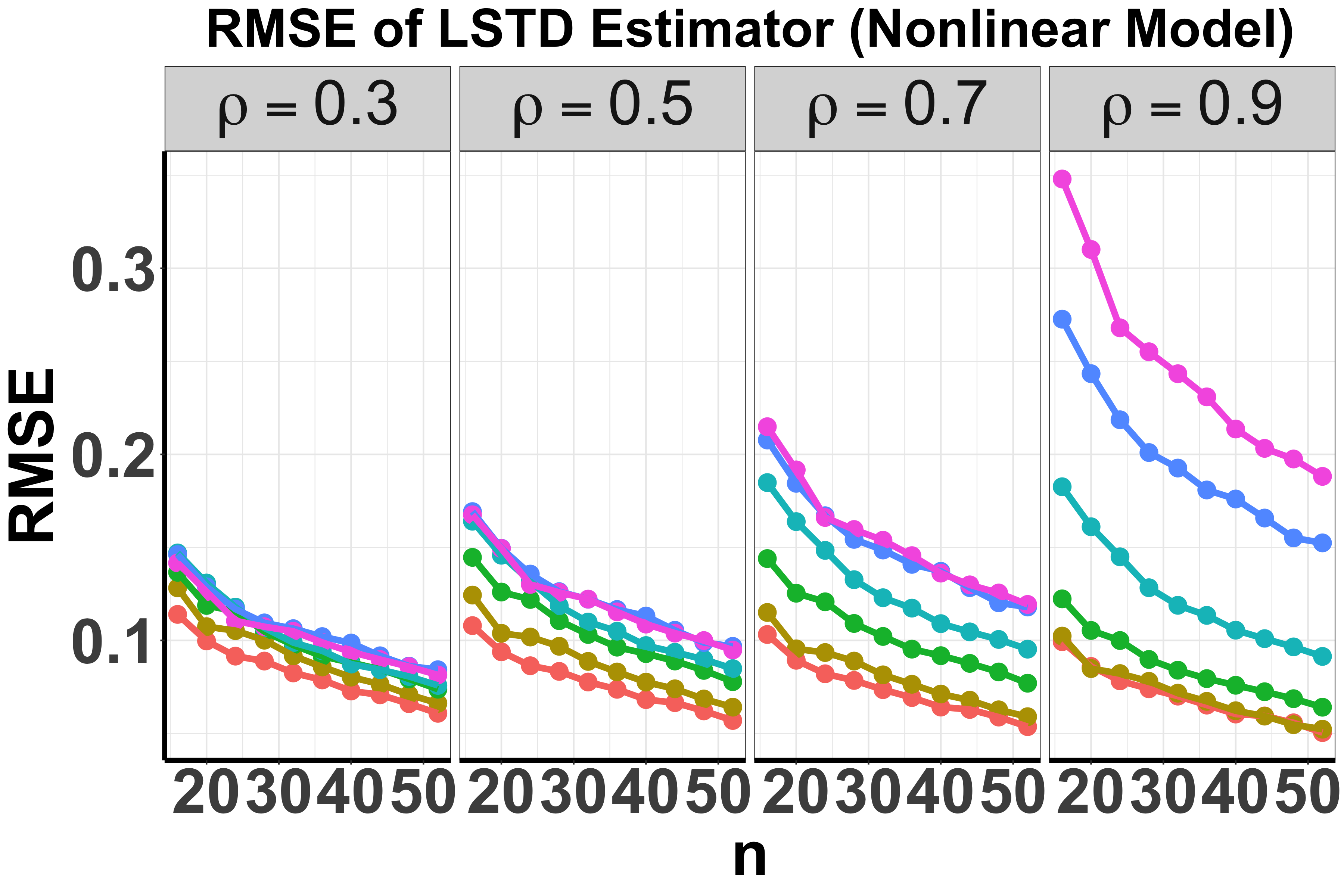}
	\end{minipage}
	\begin{minipage}{0.49\linewidth}
		\centering
		\includegraphics[width=1\linewidth,height=0.6\linewidth]{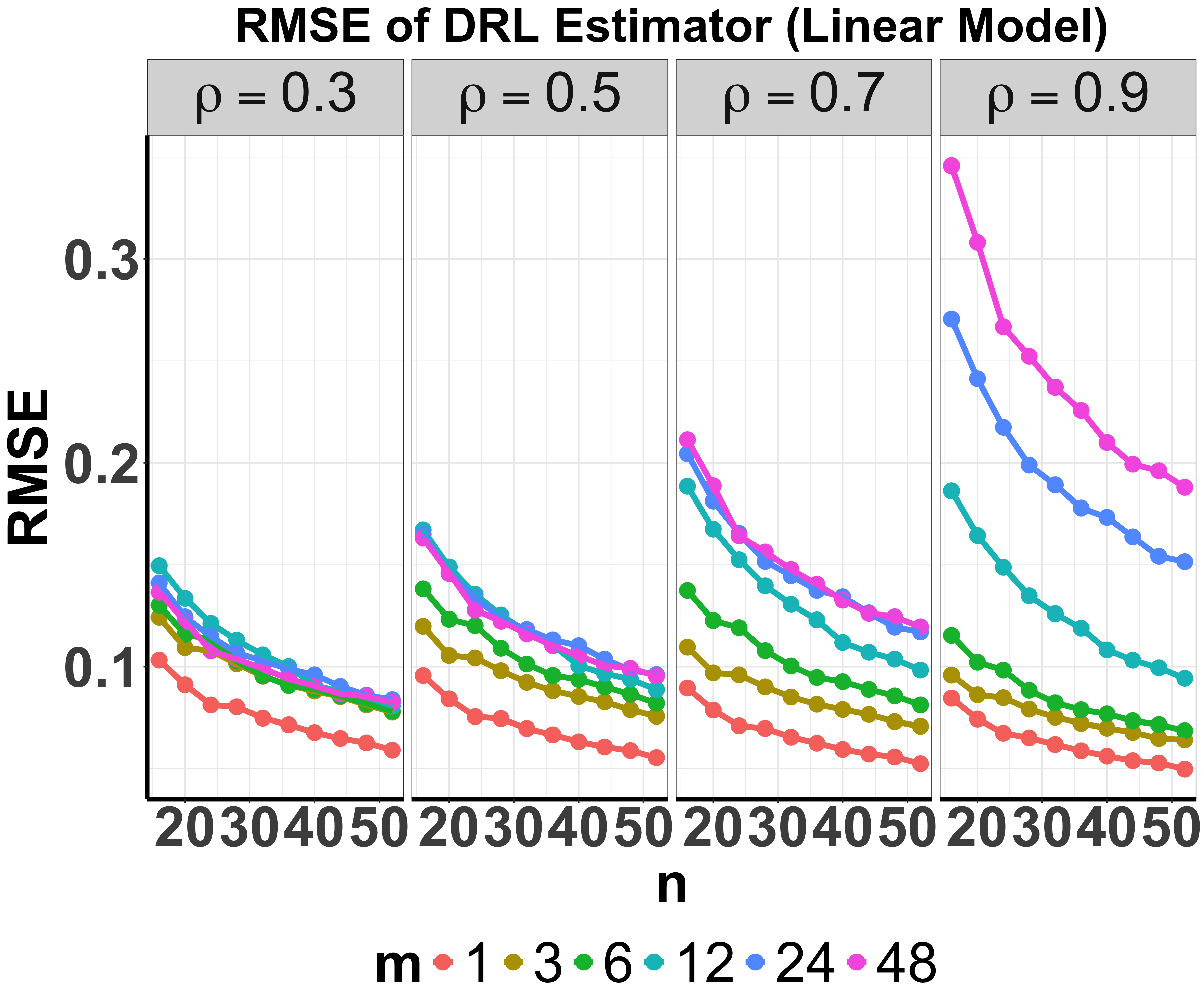}
	\end{minipage}
	\begin{minipage}{0.49\linewidth}
		\centering
		\includegraphics[width=1\linewidth,height=0.55\linewidth]{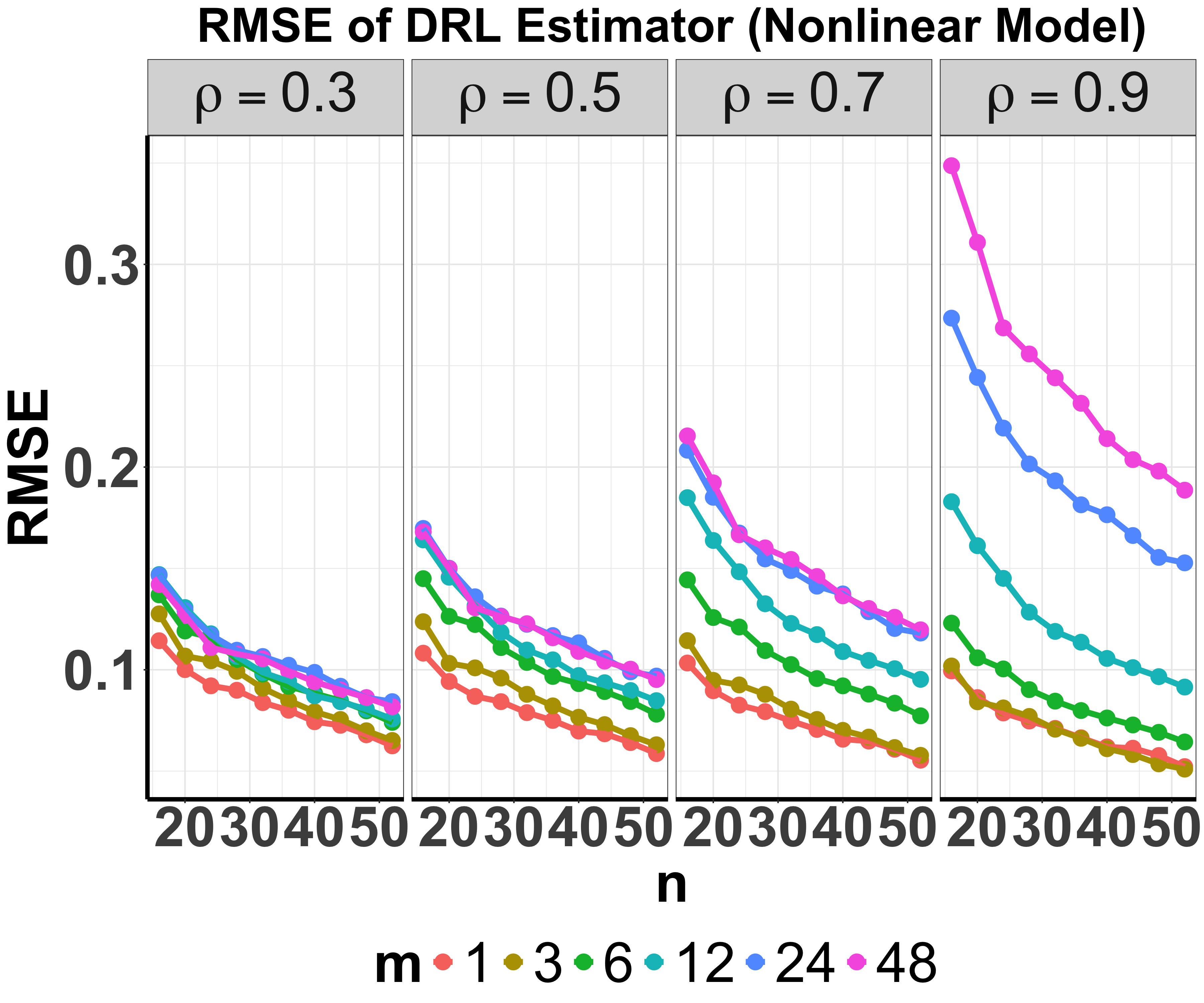}
	\end{minipage}

	\caption{ \small Simulation results with different combinations of $(n,m,\rho)$ alongside different estimation procedures. $\delta$ is fixed to zero, resulting in a weak yet nonzero carryover effect.}\label{fig2}
\end{figure}

\begin{figure}[t]
    \centering
    \begin{minipage}{0.48\linewidth}
        \centering
        \includegraphics[width=\linewidth]{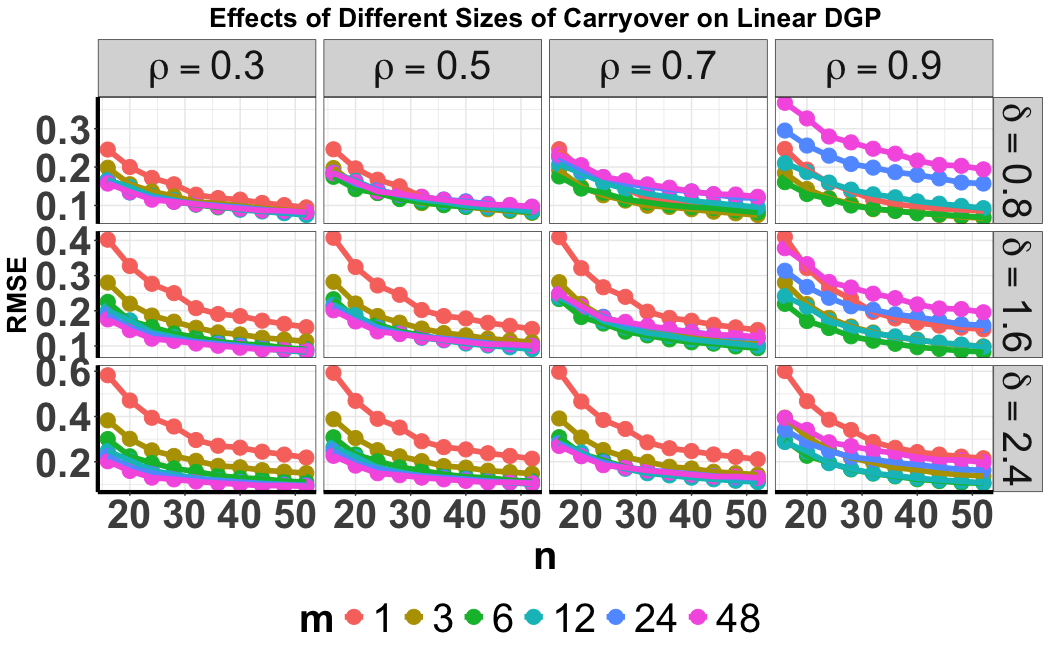}
    \end{minipage}
    \hspace{0.1cm} 
    \begin{minipage}{0.48\linewidth}
        \centering
        \includegraphics[width=\linewidth]{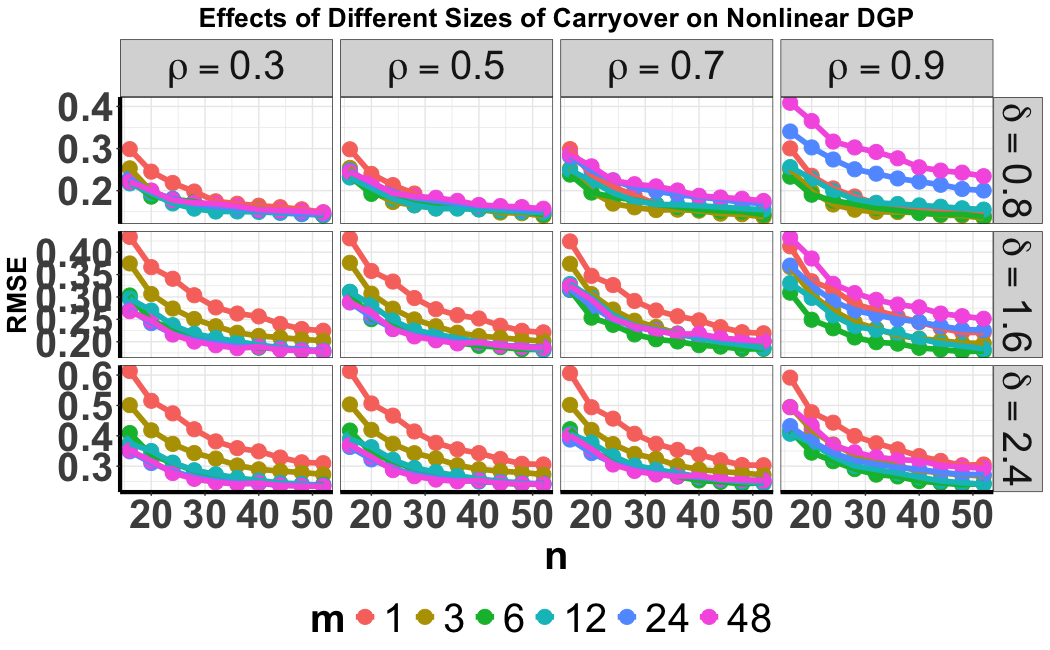}
    \end{minipage}
    \caption{{\small RMSEs of OLS estimator with different combinations of $(n,m,\rho,\delta)$. $\delta$ denotes the constant shift in generating $\Gamma_t$, i.e. $\{\Gamma_t^{(j)}\}_{t,j} \stackrel{i.i.d.}{\sim} N(\delta,0.3^2)$. A larger $\delta$ in absolute value leads to a more pronounced carryover effect.}}
\label{fig:different_carryover_sizes}
	\vspace{-3mm} 
\end{figure}

\textit{\textbf{Results}}. 
We implement various $m$-switchback designs with $m\in\{1,3,6,12,24,48\}$ in these environments and report the root MSEs of the resulting OLS, LSTD, and DRL estimators (refer to Appendix \ref{sec:implementation} for their implementation details) aggregated over 200 simulations in Figures \ref{fig2}  and \ref{fig:different_carryover_sizes}, considering different combinations of $n$, $\rho$, $\delta$ and the estimating procedure. Notably, when $m=48$, the resulting design coincides with AD. Figure \ref{fig2} report the results with weak carryover effects. It can be seen from Figure \ref{fig2} that the root MSE (RMSE) decays with $m$ in most cases. Additionally, the difference in RMSE between the SB design and the AD design grows with $\rho$, which corresponds to the autocorrelation coefficient of $\{e_t\}_t$. This aligns with our analysis, suggesting that a higher degree of positive correlation in the residuals favors SB over AD. Meanwhile, it can be seen from Figure \ref{fig:different_carryover_sizes} that as the carryover effects increase (i.e., as $\delta$ becomes larger), the AD becomes progressively more efficient. These results empirically validate our theories. 

\textbf{\textit{Sensitivity of covariance structure}}. We further examine four additional covariance structures: (i) moving average, (ii) exchangeable (with a positive correlation), (iii) uncorrelated and (iv) autoregressive (with a negative autocorrelation). 
We focus on OLS estimators and report their MSEs under different designs in Figures \ref{fig::linear_sensitivity_covariance} and \ref{fig::nonlinear_sensitivity_covariance}
of Appendix \ref{sec:addresults}. The efficiency of switchback designs varies with \(m\), depending on the presence of negative or positive correlation. In cases of uncorrelated errors, most designs attain similar performance and AD works the best in small samples. 

\textbf{\textit{Comparison}}. 
We further conduct simulations to compare the RL-based estimator with three other baseline estimators: (i) the sequential IS estimator \citep{bojinov2023design} which addresses carryover effects via multi-step importance sampling; (ii) the difference-in-mean estimator of \citet{hu2022switchback} which uses burn-in to mitigate carryover effects during policy transitions; (iii) the simple IS estimator of \citet{Xiong2023}  which does not account for carryover effects. Results are reported in Figure \ref{fig:compare_logmse_linear_nonlinear}, where RL-based estimators consistently outperform (i) and (ii), and both DRL and LSTD significantly outperform (iii) in most settings.


\begin{figure}[t]
    \centering
    \begin{minipage}[t]{0.48\linewidth}
        \centering
        \includegraphics[width=\linewidth, height=0.2\textheight]{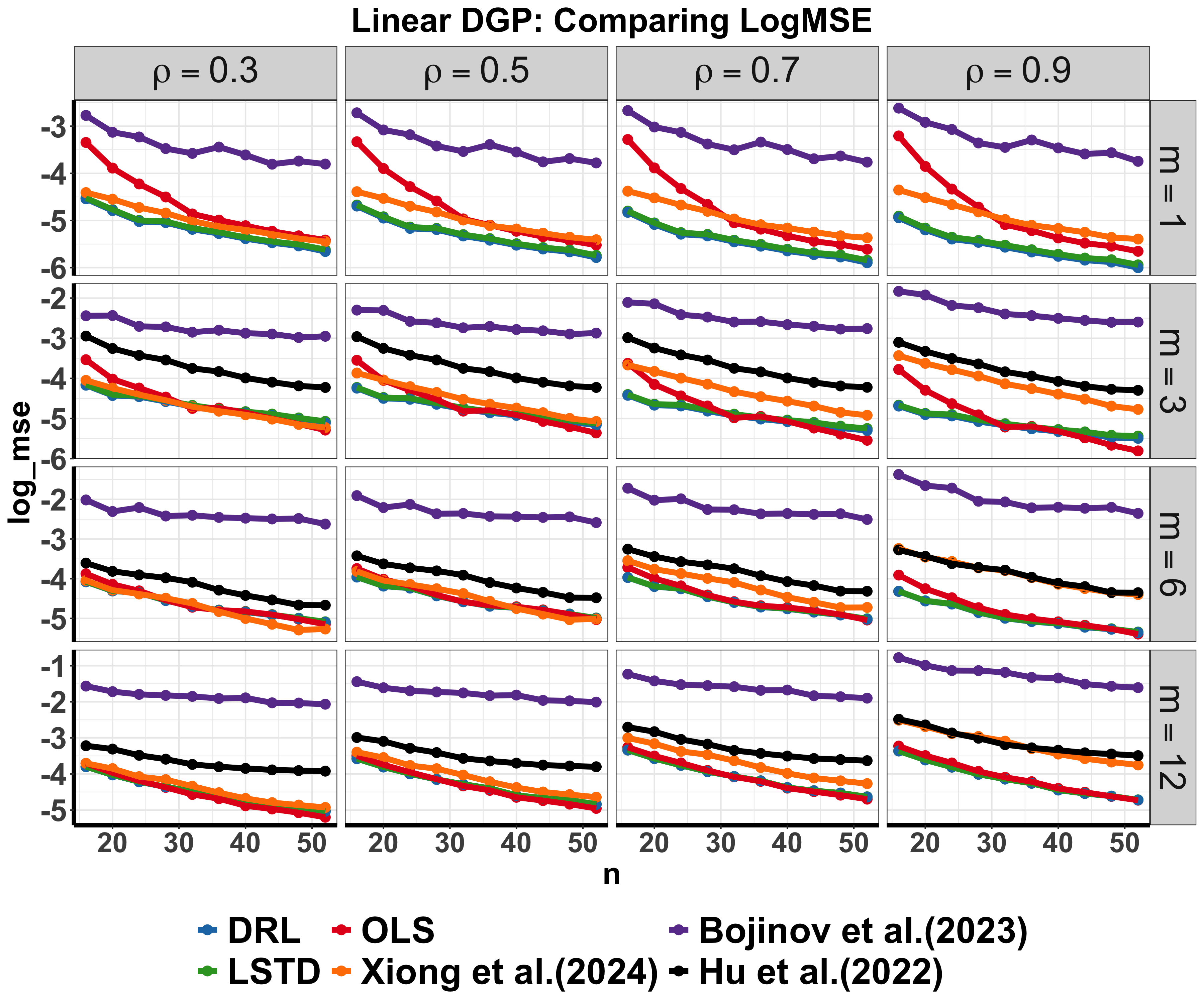}
    \end{minipage}
    \hfill
    \begin{minipage}[t]{0.48\linewidth}
        \centering
        \includegraphics[width=\linewidth, height=0.2\textheight]{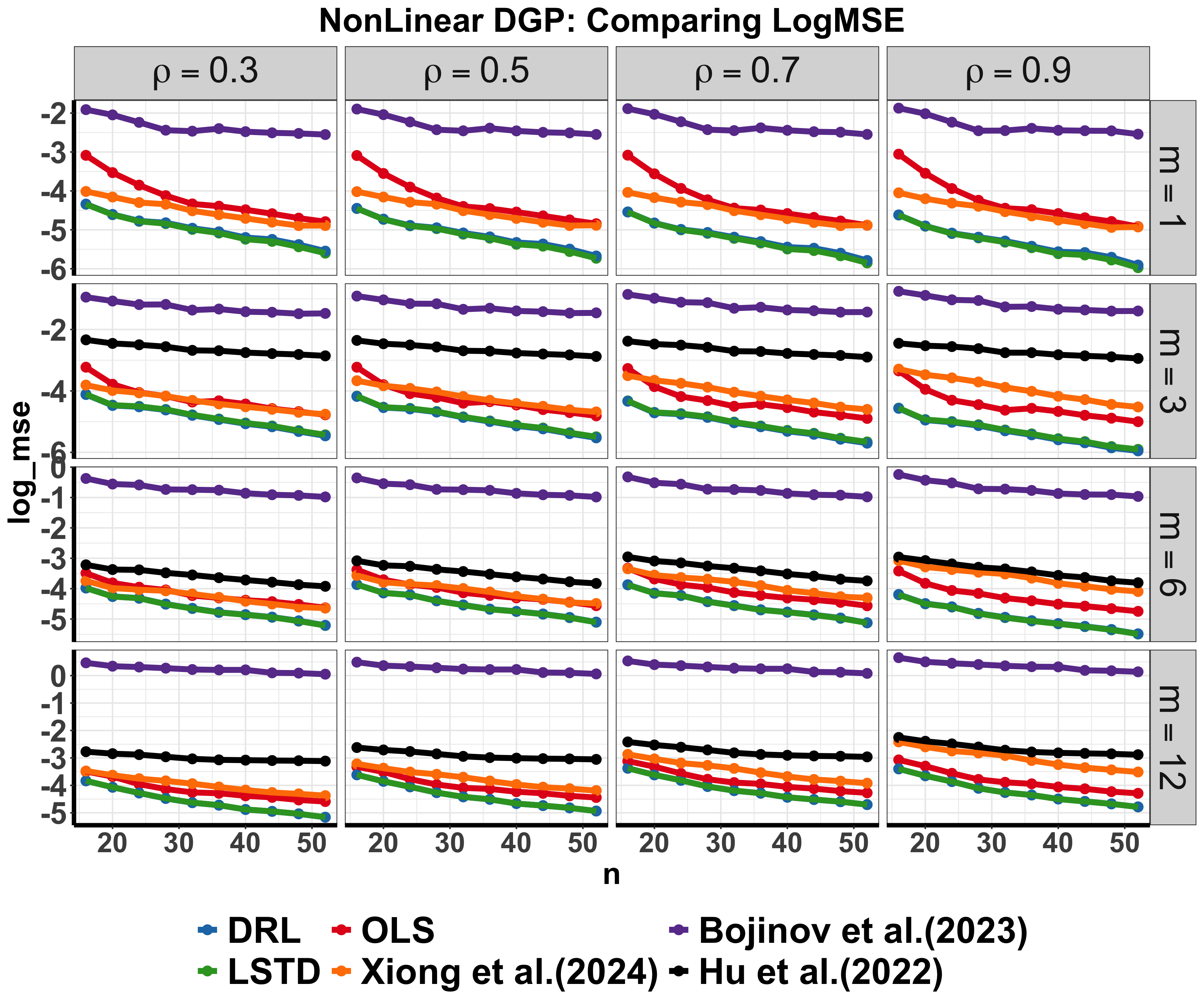}
    \end{minipage}
    \caption{\small
      Comparing log MSEs of all ATE estimators with different combinations of $(n,m, \rho)$ under the linear DGP (left) and nonlinear DGP (right).
    }\label{fig:compare_logmse_linear_nonlinear}
    \vspace{-0.8cm}
\end{figure}

\subsection{Real-data-based Simulation}\label{sec:realdata}
We use two real datasets from a leading ridesharing company, each with 40 days of data ($N=40$). The state variable includes the number of order requests and the driver's total online time in each interval, measuring the demand and supply dynamics that impact the platform's outcomes \citep{shi2023dynamic}. The reward is the total income earned by the drivers within each time interval, whose residual shows a noticeable positive correlation, as depicted in Figure \ref{fig0}. Both datasets are collected from A/A experiments, in which a single order dispatch policy was consistently applied over time ($A_t=0$ for all $t$). To utilize these datasets for evaluating the performance of different designs, we create two simulation environments using the wild bootstrap \citep{wu1986jackknife}, as detailed in Appendix \ref{sec:addresults}. To generate data under different policies, we introduce an effect size parameter $\lambda$ and consider four choices, corresponding to 0 (i.e., no treatment effect at all), 2\%, 5\%, 7.5\%, 10\%, 12\% and 15\%. By adjusting $\lambda$, the data are generated so that both the direct effect and carryover effect of the new policy (see Equation \eqref{ate_est_formula}) are increase by $\frac{\lambda}{2}$, leading to an overall ATE increase of $\lambda$.

Figures \ref{fig_real_data_rmse} summarize the results, which strongly support our theoretical findings. Specifically, with positive correlated reward residuals, the benefits of employing switchback designs with more frequent switches are evident when there is no or only a minor effect enhancement (i.e., 0 or 2\% increase). However, as the size of the carryover increases to 15\%, the AD becomes more efficient. 

To better understand our estimators, we conducted additional simulations based on the City II dataset, employed a non-parametric bootstrap method to construct confidence intervals (CIs), and reported both the coverage probability (CP) and the average CI width in the Figure \ref{fig:CI_mean_widths} of Appendix. Most CPs exceed 92\%, which is close to the nominal level. For small values of $\lambda$, more frequent policy switching reduces the average CI width. As $\lambda$ increases to 10\%, AD yields the narrowest CI on average. These results support our claim that a lower MSE corresponds to shorter CIs.

\begin{figure}[t]
	\centering
	\vspace{-1mm} 
	\begin{minipage}{0.49\linewidth}
		\centering
		\includegraphics[width=1\linewidth,height=0.6\linewidth]{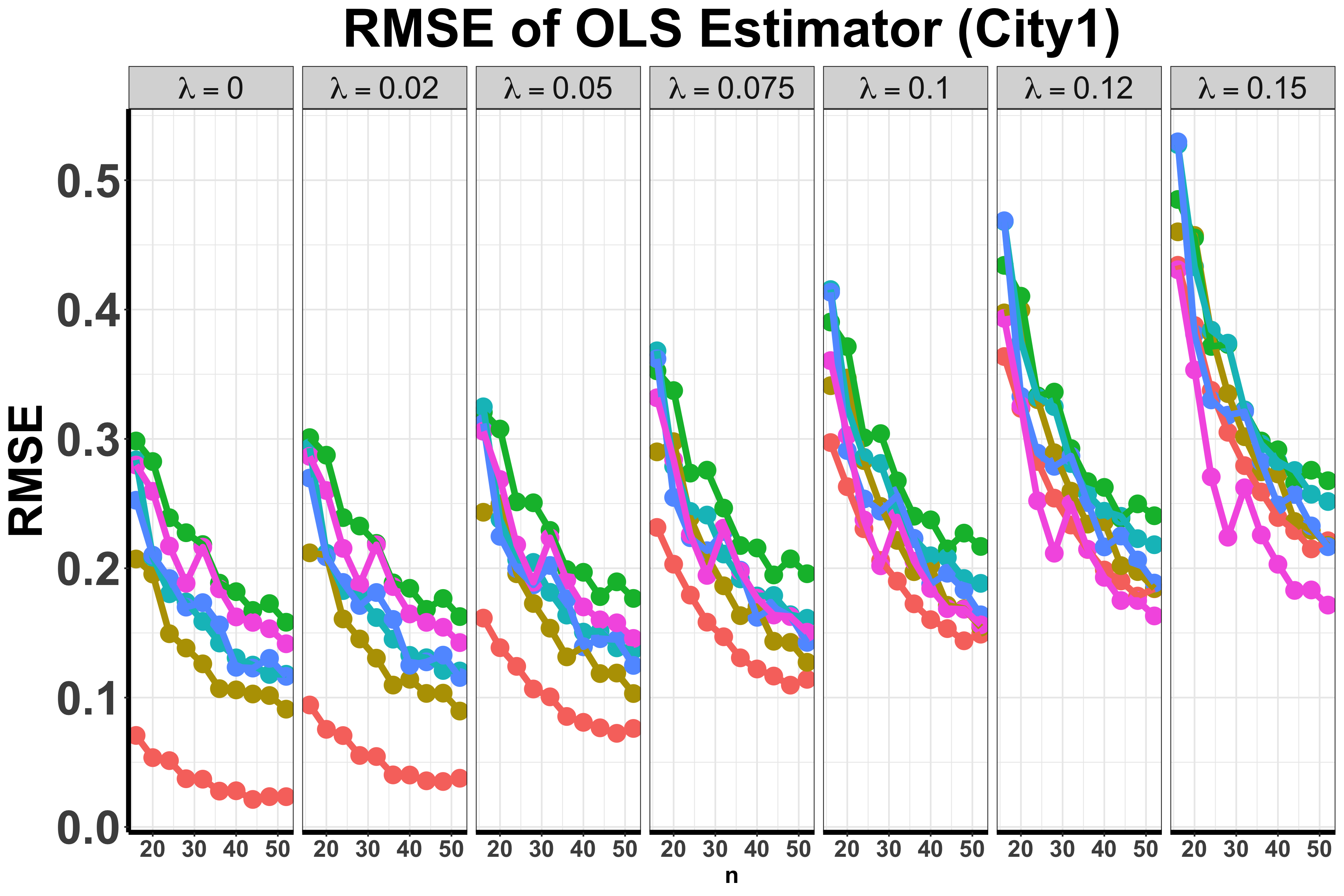}
	\end{minipage}
	\begin{minipage}{0.49\linewidth}
		\centering
		\includegraphics[width=1\linewidth,height=0.6\linewidth]{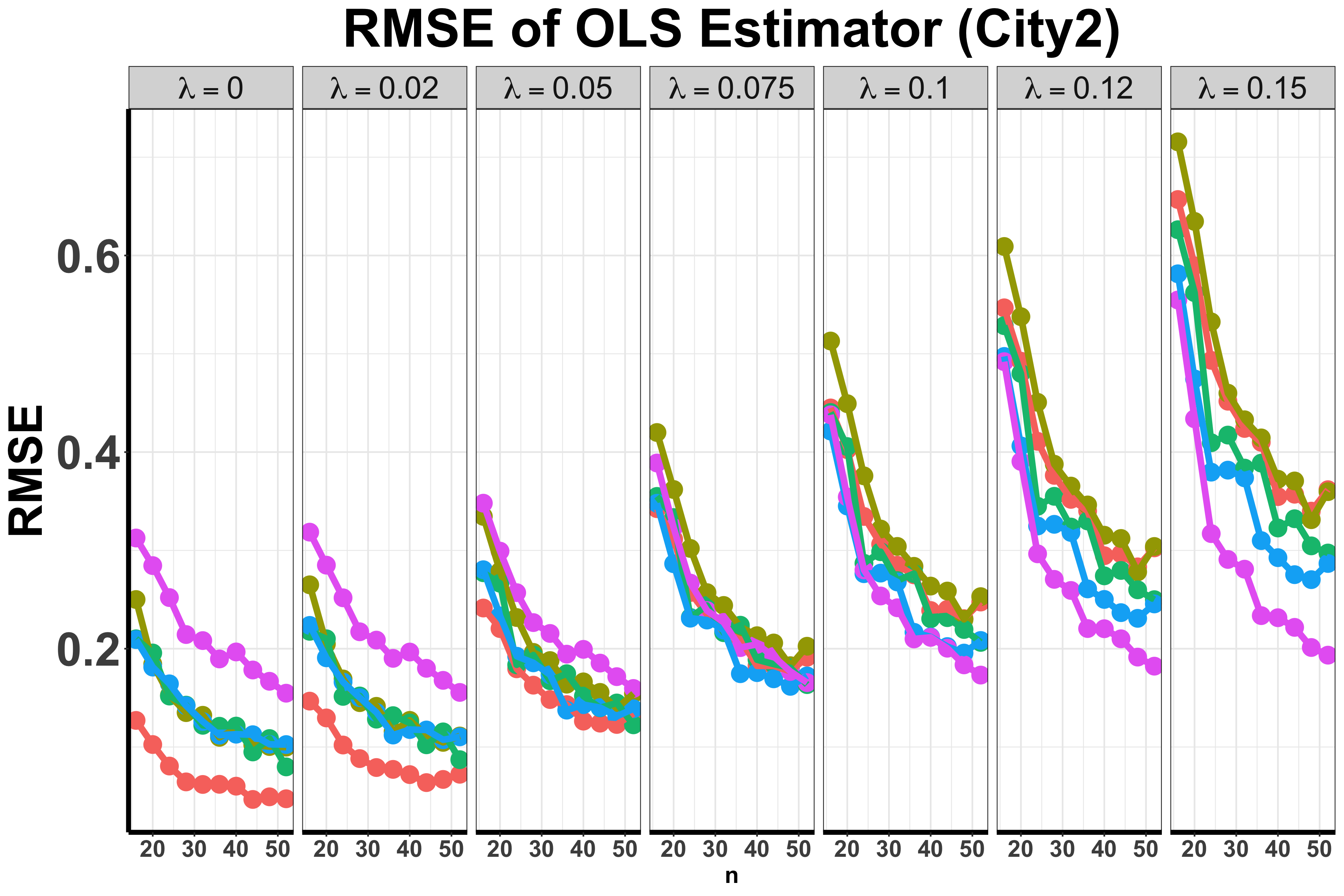}
	\end{minipage}
	\begin{minipage}{0.49\linewidth}
		\centering
		\includegraphics[width=1\linewidth,height=0.6\linewidth]{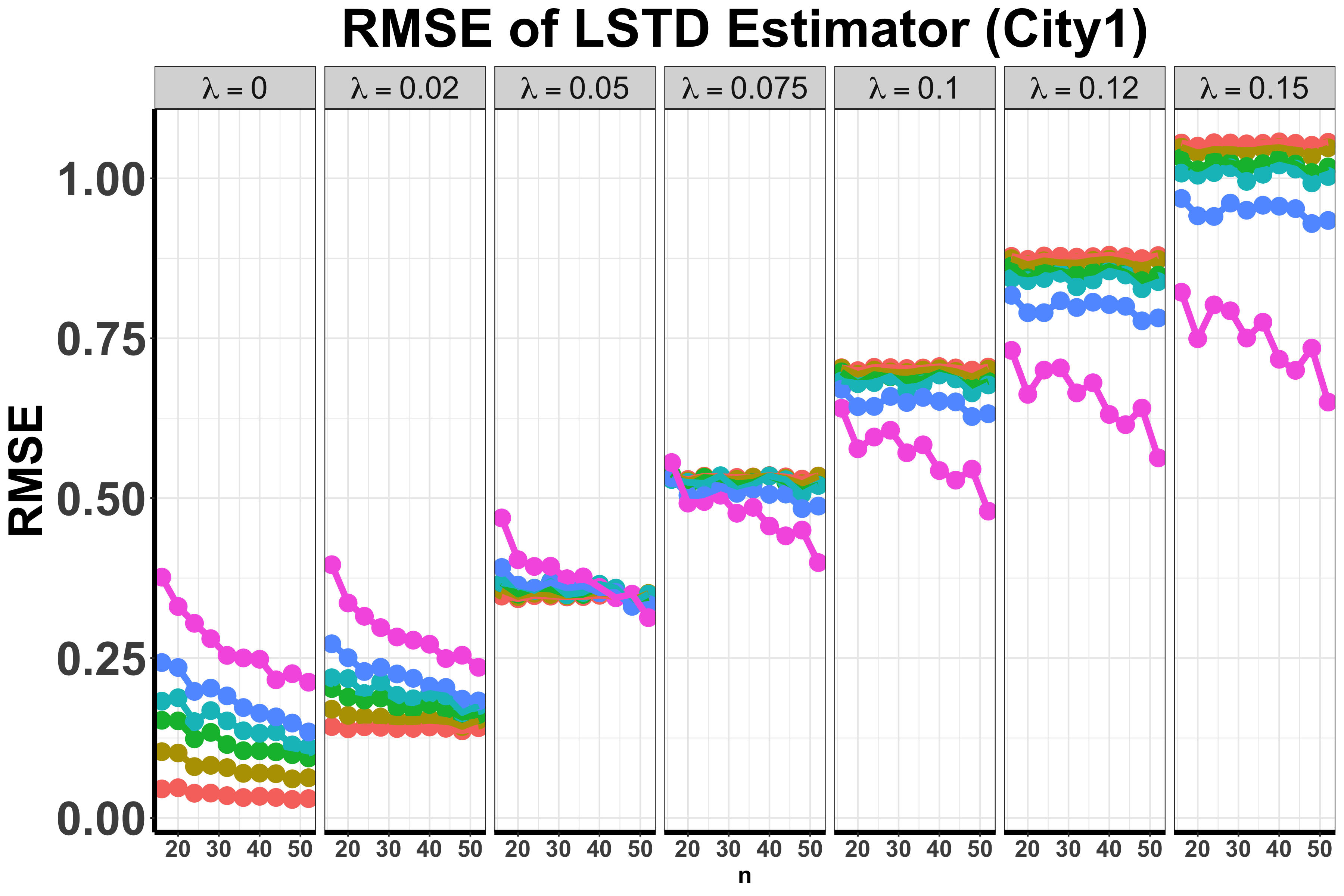}
	\end{minipage}
	\begin{minipage}{0.49\linewidth}
		\centering
		\includegraphics[width=1\linewidth,height=0.6\linewidth]{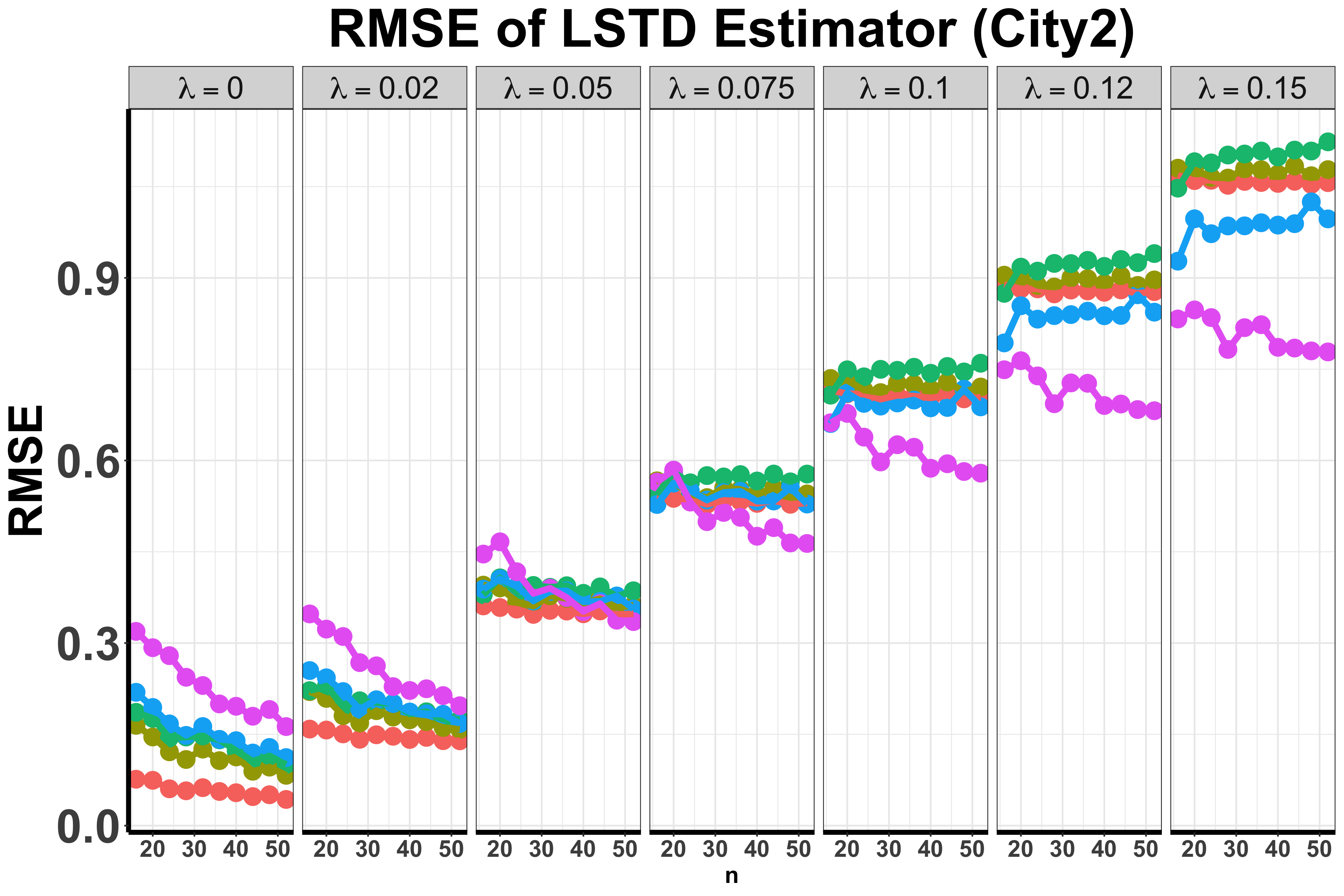}
	\end{minipage}
	\begin{minipage}{0.49\linewidth}
		\centering
		\includegraphics[width=1\linewidth,height=0.6\linewidth]{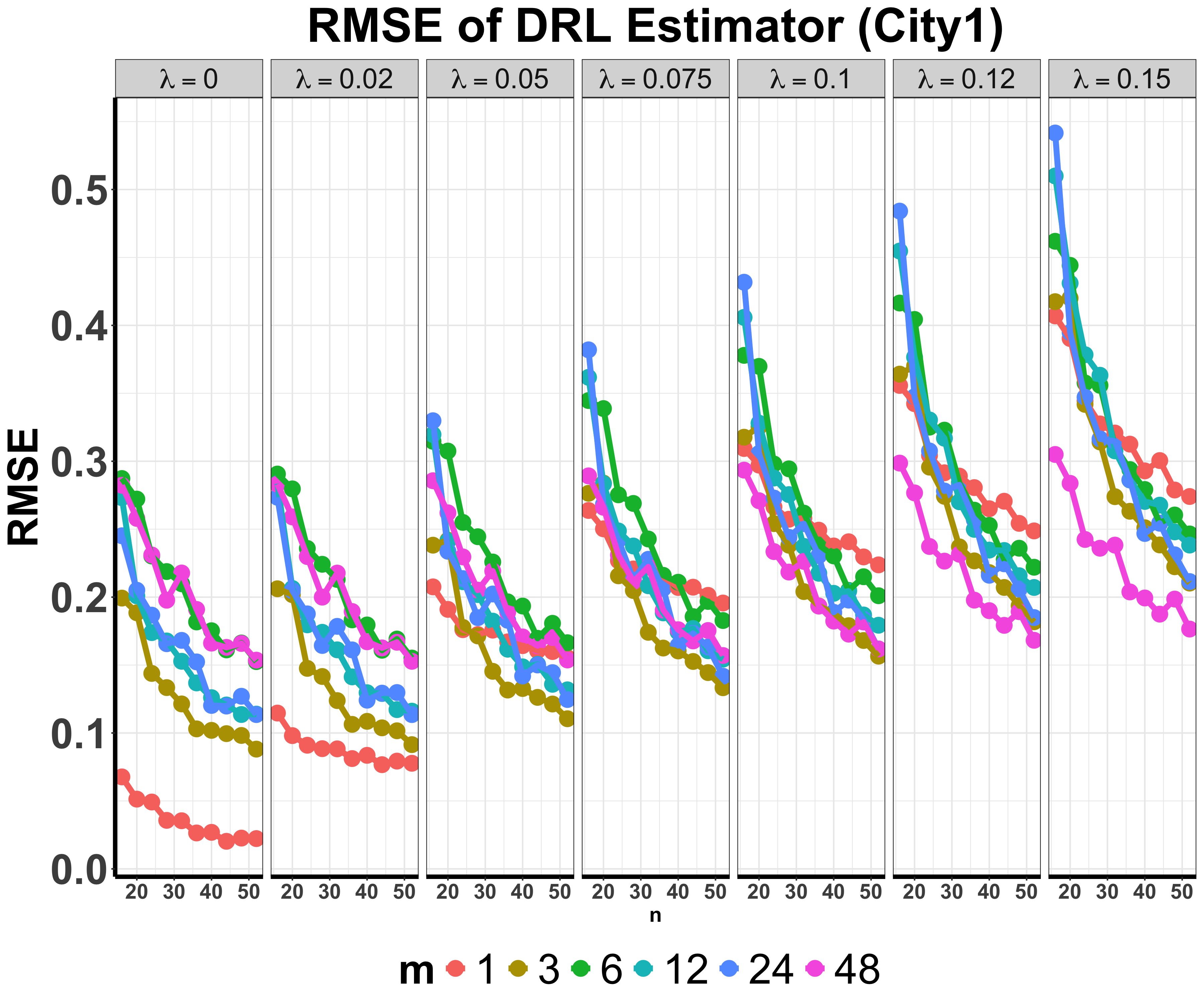}
	\end{minipage}
	\begin{minipage}{0.49\linewidth}
		\centering
		\includegraphics[width=1\linewidth,height=0.6\linewidth]{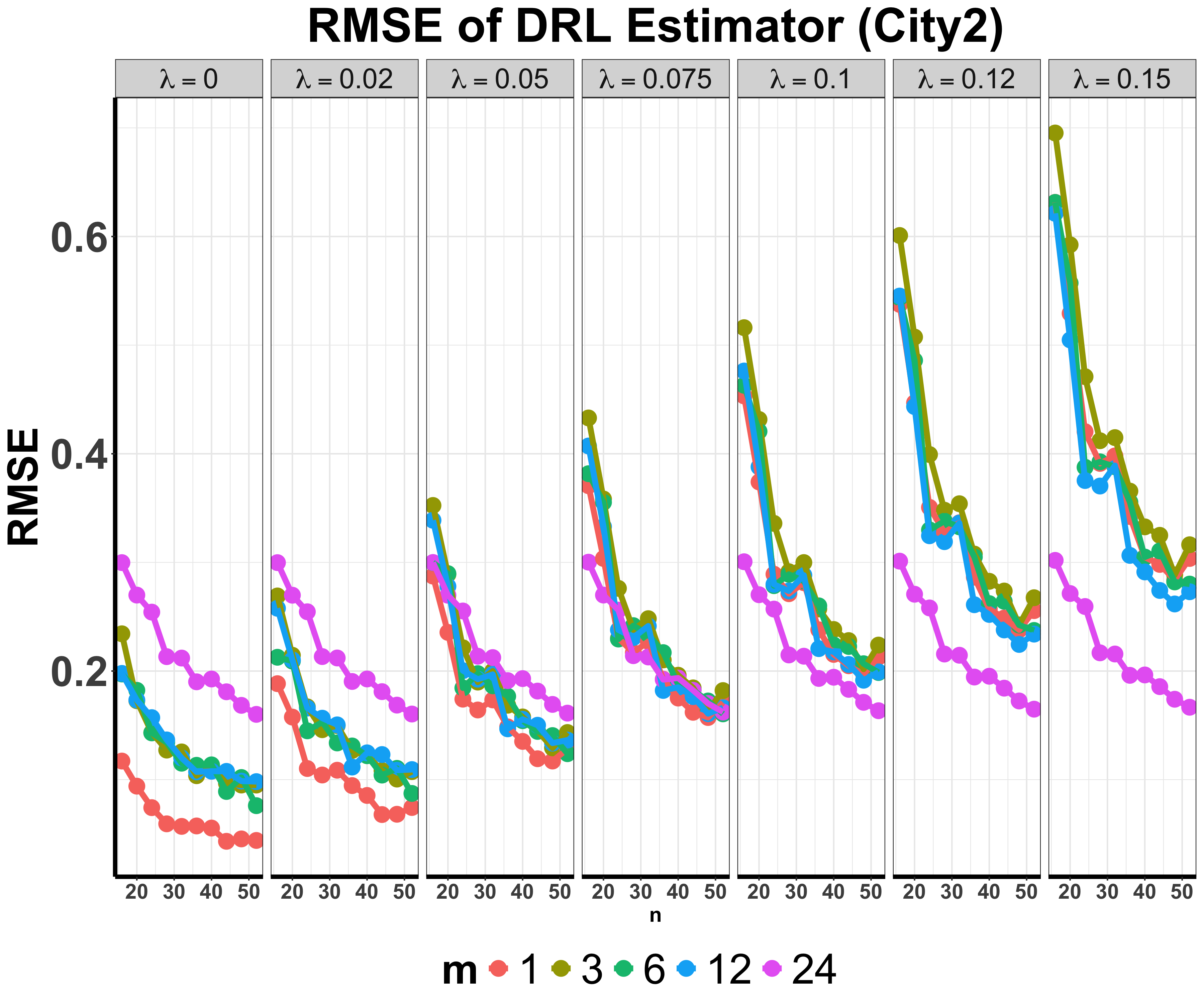}
	\end{minipage}

	\caption{Results from real-data-based simulation.
    }\label{fig_real_data_rmse}
	\vspace{-4mm} 
\end{figure}

\section{Discussion}\label{sec::prac_guide}
This paper studies switchback designs under an RL framework. To offer practical guidance, we outline a workflow in this section (see also Figure \ref{fig:prac_guide}): 
\begin{enumerate}[leftmargin=*]
    \item[(i)] The first step is to discretize the data to define appropriate time intervals, ensuring that both the state and reward follow the Markov assumption (see Section \ref{sec:model}). This assumption can be tested via existing state-of-the-art methods \citep[see e.g.,][]{chen2012testing,shi2020does,zhou2023testing}. When the Markov assumption is violated, it is necessary to increase the length of time intervals accordingly to satisfy this condition. 
    \item[(ii)] The second step is to assess the magnitude of the carryover effect. Should the effect be strong, AD is recommended. In our numerical studies, we observe that with a large carryover effect, the ATE estimator under SB suffers from much a larger bias than that under AD. 
    \item[(iii)] Finally, if  the carryover effect is weak, we proceed to analyze error correlations. When errors exhibit positive correlations, we recommend to employ the switchback design with $m=1$. In cases where errors are uncorrelated or negatively correlated, AD would be preferred. 
\end{enumerate}

\begin{figure}[!t]
	\centering
\includegraphics[width=0.4\textwidth,height=0.16\textwidth]{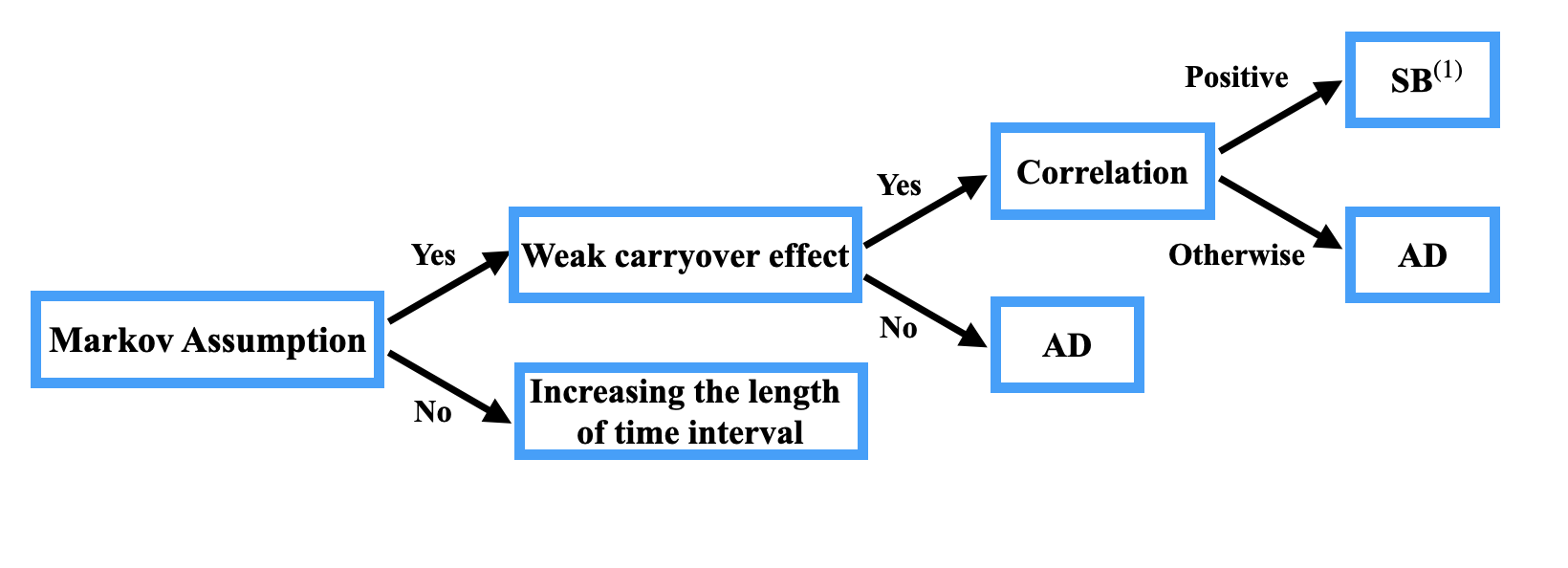}\vspace{-2em}
	\caption{{The  proposed workflow guideline. }}
	\label{fig:prac_guide}	\vspace{-1em}
\end{figure}

We also remark that in our motivating ridesharing example, there are three different types of experiments: (i) temporal randomization (over time), (ii) spatial randomization (across geographic areas), and (iii) user-level randomization (across drivers/passengers). Our primary focus is (i), which applies to the evaluation of order dispatch policies that must be implemented city-wide, making (ii) and (iii) unsuitable. Spatial randomization (ii) is typically used for testing localized subsidy policies in different regions, while user-level randomization (iii) applies when assigning personalized subsidies to individual users. 


When restricting to temporal randomization, the population corresponds to the entire time horizon, with each time interval representing an individual unit. When no carryover effects exist and residuals are uncorrelated (satisfying i.i.d. assumptions), similar to Theorem \ref{thm::main}, we can show that standard uniform randomization over time is equivalent to both AD and SB. This is because temporal ordering becomes irrelevant under the uncorrelatedness assumption. Similarly, when randomizing is conducted at the daily level rather than per time unit, standard A/B testing procedures that ignore carryover effects are equivalent to AD designs.

Finally, confidence intervals (CIs) and $p$-values are equally important metrics, as A/B testing is fundamentally a statistical inference problem. Our experimental designs are tailored to minimize the MSE of the resulting ATE estimators. A closer examination of the proof of Theorem~\ref{thm::main} reveals that the three RL-based estimators are asymptotically normal, with their MSEs primarily driven by their asymptotic variances. Consequently, optimal designs that minimize the variance of the ATE estimator also reduce CI length and enhance the power of the corresponding hypothesis test. We empirically validate this claim using real-data-based simulation in Figure \ref{fig:CI_mean_widths}. 


\section*{Acknowledgement}
We thank the anonymous referees and the meta reviewer for their constructive comments, which have led to a significant improvement of the earlier version of this article. Shi's research is partly supported by an EPSRC grant EP/W014971/1. Yang's research is partly supported by the National Natural Science Foundation of China (No. 72301276). Wen and Tang's research is partly supported by the National Key R\&D Program of China (No. 2022YFA1003701). 

\section*{Impact Statement}
This paper investigates switchback experiments in modern digital platforms, focusing on the key factors that determine their efficiency. The study finds that the effectiveness of switchback designs is primarily influenced by the autocorrelation structure and the magnitude of carryover effects, with an inherent trade-off that applies to most state-of-the-art RL estimators. Additionally, a practical workflow is proposed to improve the methodological approaches for both researchers and practitioners. However, it is important to note that switchback designs are not universally applicable and may be unsuitable in the following cases: (i) lack of temporal dependence, (ii) significant carryover effects, (iii) complex interactions across users or regions, (iv) high switching costs, and (v) rapidly changing systems. Researchers should carefully evaluate these limitations and apply switchback designs with caution.

\nocite{langley00}

\bibliography{references_paper}

\begin{thebibliography}{138}
\providecommand{\natexlab}[1]{#1}
\providecommand{\url}[1]{\texttt{#1}}
\expandafter\ifx\csname urlstyle\endcsname\relax
  \providecommand{\doi}[1]{doi: #1}\else
  \providecommand{\doi}{doi: \begingroup \urlstyle{rm}\Url}\fi

\bibitem[Athey \& Imbens(2017)Athey and Imbens]{athey2017state}
Athey, S. and Imbens, G.~W.
\newblock The state of applied econometrics: Causality and policy evaluation.
\newblock \emph{Journal of Economic perspectives}, 31\penalty0 (2):\penalty0
  3--32, 2017.

\bibitem[Atkinson \& Pedrosa(2017)Atkinson and Pedrosa]{atkinson2017optimum}
Atkinson, A. and Pedrosa, D.
\newblock Optimum design and sequential treatment allocation in an experiment
  in deep brain stimulation with sets of treatment combinations.
\newblock \emph{Statistics in Medicine}, 36\penalty0 (30):\penalty0 4804--4815,
  2017.

\bibitem[Atkinson \& Biswas(2013)Atkinson and Biswas]{atkinson2013randomised}
Atkinson, A.~C. and Biswas, A.
\newblock Randomised response-adaptive designs in clinical trials.
\newblock \emph{Monographs Stat. Appl. Probability}, 130:\penalty0 130, 2013.

\bibitem[Bajari et~al.(2021)Bajari, Burdick, Imbens, Masoero, McQueen,
  Richardson, and Rosen]{bajari2021multiple}
Bajari, P., Burdick, B., Imbens, G.~W., Masoero, L., McQueen, J., Richardson,
  T., and Rosen, I.~M.
\newblock Multiple randomization designs.
\newblock \emph{arXiv preprint arXiv:2112.13495}, 2021.

\bibitem[Bajari et~al.(2023)Bajari, Burdick, Imbens, Masoero, McQueen,
  Richardson, and Rosen]{bajari2023experimental}
Bajari, P., Burdick, B., Imbens, G.~W., Masoero, L., McQueen, J., Richardson,
  T.~S., and Rosen, I.~M.
\newblock Experimental design in marketplaces.
\newblock \emph{Statistical Science}, 38\penalty0 (3):\penalty0 458--476, 2023.

\bibitem[Baldi~Antognini \& Zagoraiou(2011)Baldi~Antognini and
  Zagoraiou]{baldi2011covariate}
Baldi~Antognini, A. and Zagoraiou, M.
\newblock The covariate-adaptive biased coin design for balancing clinical
  trials in the presence of prognostic factors.
\newblock \emph{Biometrika}, 98\penalty0 (3):\penalty0 519--535, 2011.

\bibitem[Basse et~al.(2016)Basse, Soufiani, and
  Lambert]{basse2016randomization}
Basse, G.~W., Soufiani, H.~A., and Lambert, D.
\newblock Randomization and the pernicious effects of limited budgets on
  auction experiments.
\newblock In \emph{Artificial Intelligence and Statistics}, pp.\  1412--1420.
  PMLR, 2016.

\bibitem[Basse et~al.(2023)Basse, Ding, and Toulis]{basse2023minimax}
Basse, G.~W., Ding, Y., and Toulis, P.
\newblock Minimax designs for causal effects in temporal experiments with
  treatment habituation.
\newblock \emph{Biometrika}, 110\penalty0 (1):\penalty0 155--168, 2023.

\bibitem[Begg \& Iglewicz(1980)Begg and Iglewicz]{begg1980treatment}
Begg, C.~B. and Iglewicz, B.
\newblock A treatment allocation procedure for sequential clinical trials.
\newblock \emph{Biometrics}, pp.\  81--90, 1980.

\bibitem[Berenblut \& Webb(1974)Berenblut and Webb]{berenblut1974experimental}
Berenblut, I. and Webb, G.
\newblock Experimental design in the presence of autocorrelated errors.
\newblock \emph{Biometrika}, 61\penalty0 (3):\penalty0 427--437, 1974.

\bibitem[Bhattacharya \& Sen(2024)Bhattacharya and Sen]{bhattacharya2024causal}
Bhattacharya, S. and Sen, S.
\newblock Causal effect estimation under network interference with mean-field
  methods.
\newblock \emph{arXiv preprint arXiv:2407.19613}, 2024.

\bibitem[Bian et~al.(2025)Bian, Shi, Qi, and Wang]{bian2025off}
Bian, Z., Shi, C., Qi, Z., and Wang, L.
\newblock Off-policy evaluation in doubly inhomogeneous environments.
\newblock \emph{Journal of the American Statistical Association}, to appear,
  2025.

\bibitem[Bibaut et~al.(2019)Bibaut, Malenica, Vlassis, and Van
  Der~Laan]{bibaut2019more}
Bibaut, A., Malenica, I., Vlassis, N., and Van Der~Laan, M.
\newblock More efficient off-policy evaluation through regularized targeted
  learning.
\newblock In \emph{International Conference on Machine Learning}, pp.\
  654--663. PMLR, 2019.

\bibitem[Bickel et~al.(1993)Bickel, Klaassen, Bickel, Ritov, Klaassen, Wellner,
  and Ritov]{bickel1993efficient}
Bickel, P.~J., Klaassen, C.~A., Bickel, P.~J., Ritov, Y., Klaassen, J.,
  Wellner, J.~A., and Ritov, Y.
\newblock \emph{Efficient and adaptive estimation for semiparametric models},
  volume~4.
\newblock Springer, 1993.

\bibitem[Blau et~al.(2022)Blau, Bonilla, Chades, and
  Dezfouli]{blau2022optimizing}
Blau, T., Bonilla, E.~V., Chades, I., and Dezfouli, A.
\newblock Optimizing sequential experimental design with deep reinforcement
  learning.
\newblock In \emph{International Conference on Machine Learning}, pp.\
  2107--2128. PMLR, 2022.

\bibitem[Bojinov \& Shephard(2019)Bojinov and Shephard]{bojinov2019time}
Bojinov, I. and Shephard, N.
\newblock Time series experiments and causal estimands: exact randomization
  tests and trading.
\newblock \emph{Journal of the American Statistical Association}, 114\penalty0
  (528):\penalty0 1665--1682, 2019.

\bibitem[Bojinov et~al.(2023)Bojinov, Simchi-Levi, and Zhao]{bojinov2023design}
Bojinov, I., Simchi-Levi, D., and Zhao, J.
\newblock Design and analysis of switchback experiments.
\newblock \emph{Management Science}, 69\penalty0 (7):\penalty0 3759--3777,
  2023.

\bibitem[Boruvka et~al.(2018)Boruvka, Almirall, Witkiewitz, and
  Murphy]{boruvka2018assessing}
Boruvka, A., Almirall, D., Witkiewitz, K., and Murphy, S.~A.
\newblock Assessing time-varying causal effect moderation in mobile health.
\newblock \emph{Journal of the American Statistical Association}, 113\penalty0
  (523):\penalty0 1112--1121, 2018.

\bibitem[Bradtke \& Barto(1996)Bradtke and Barto]{bradtke1996linear}
Bradtke, S.~J. and Barto, A.~G.
\newblock Linear least-squares algorithms for temporal difference learning.
\newblock \emph{Machine learning}, 22:\penalty0 33--57, 1996.

\bibitem[Cao \& Zhou(2024)Cao and Zhou]{cao2024orthogonalized}
Cao, D. and Zhou, A.
\newblock Orthogonalized estimation of difference of $ q $-functions.
\newblock \emph{arXiv preprint arXiv:2406.08697}, 2024.

\bibitem[Chen \& Hong(2012)Chen and Hong]{chen2012testing}
Chen, B. and Hong, Y.
\newblock Testing for the markov property in time series.
\newblock \emph{Econometric Theory}, 28\penalty0 (1):\penalty0 130--178, 2012.

\bibitem[Chen \& Jiang(2019)Chen and Jiang]{chen2019information}
Chen, J. and Jiang, N.
\newblock Information-theoretic considerations in batch reinforcement learning.
\newblock In \emph{International Conference on Machine Learning}, pp.\
  1042--1051. PMLR, 2019.

\bibitem[Chen \& Christensen(2015)Chen and Christensen]{chen2015optimal}
Chen, X. and Christensen, T.~M.
\newblock Optimal uniform convergence rates and asymptotic normality for series
  estimators under weak dependence and weak conditions.
\newblock \emph{Journal of Econometrics}, 188\penalty0 (2):\penalty0 447--465,
  2015.

\bibitem[Chen \& Qi(2022)Chen and Qi]{chen2022well}
Chen, X. and Qi, Z.
\newblock On well-posedness and minimax optimal rates of nonparametric
  q-function estimation in off-policy evaluation.
\newblock In \emph{International Conference on Machine Learning}, pp.\
  3558--3582. PMLR, 2022.

\bibitem[Chernozhukov et~al.(2018)Chernozhukov, Chetverikov, Demirer, Duflo,
  Hansen, Newey, and Robins]{chernozhukov_doubledebiased_2018}
Chernozhukov, V., Chetverikov, D., Demirer, M., Duflo, E., Hansen, C., Newey,
  W., and Robins, J.
\newblock Double/debiased machine learning for treatment and structural
  parameters.
\newblock \emph{The Econometrics Journal}, 21\penalty0 (1):\penalty0 C1--C68,
  2018.
\newblock ISSN 1368-4221.

\bibitem[Chernozhukov et~al.(2022)Chernozhukov, Newey, Singh, and
  Syrgkanis]{chernozhukov2022automatic}
Chernozhukov, V., Newey, W., Singh, R., and Syrgkanis, V.
\newblock Automatic debiased machine learning for dynamic treatment effects and
  general nested functionals.
\newblock \emph{arXiv preprint arXiv:2203.13887}, 2022.

\bibitem[Dai et~al.(2020)Dai, Nachum, Chow, Li, Szepesv{\'a}ri, and
  Schuurmans]{dai2020coindice}
Dai, B., Nachum, O., Chow, Y., Li, L., Szepesv{\'a}ri, C., and Schuurmans, D.
\newblock Coindice: Off-policy confidence interval estimation.
\newblock \emph{Advances in neural information processing systems},
  33:\penalty0 9398--9411, 2020.

\bibitem[Dai et~al.(2024)Dai, Wang, Zhou, Luo, Qin, Shi, and
  Zhu]{dai2024causal}
Dai, R., Wang, J., Zhou, F., Luo, S., Qin, Z., Shi, C., and Zhu, H.
\newblock Causal deepsets for off-policy evaluation under spatial or
  spatio-temporal interferences.
\newblock \emph{arXiv preprint arXiv:2407.17910}, 2024.

\bibitem[Fan et~al.(2020)Fan, Wang, Xie, and Yang]{fan2020theoretical}
Fan, J., Wang, Z., Xie, Y., and Yang, Z.
\newblock A theoretical analysis of deep q-learning.
\newblock In \emph{Learning for dynamics and control}, pp.\  486--489. PMLR,
  2020.

\bibitem[Farias et~al.(2022)Farias, Li, Peng, and Zheng]{farias2022markovian}
Farias, V., Li, A., Peng, T., and Zheng, A.
\newblock Markovian interference in experiments.
\newblock \emph{Advances in Neural Information Processing Systems},
  35:\penalty0 535--549, 2022.

\bibitem[Farias et~al.(2023)Farias, Li, Peng, Ren, Zhang, and
  Zheng]{farias2023correcting}
Farias, V., Li, H., Peng, T., Ren, X., Zhang, H., and Zheng, A.
\newblock Correcting for interference in experiments: A case study at douyin.
\newblock In \emph{Proceedings of the 17th ACM Conference on Recommender
  Systems}, pp.\  455--466, 2023.

\bibitem[Foster et~al.(2021)Foster, Ivanova, Malik, and
  Rainforth]{foster2021deep}
Foster, A., Ivanova, D.~R., Malik, I., and Rainforth, T.
\newblock Deep adaptive design: Amortizing sequential bayesian experimental
  design.
\newblock In \emph{International Conference on Machine Learning}, pp.\
  3384--3395. PMLR, 2021.

\bibitem[Glynn et~al.(2020)Glynn, Johari, and Rasouli]{glynn2020adaptive}
Glynn, P.~W., Johari, R., and Rasouli, M.
\newblock Adaptive experimental design with temporal interference: A maximum
  likelihood approach.
\newblock \emph{Advances in Neural Information Processing Systems},
  33:\penalty0 15054--15064, 2020.

\bibitem[Gottesman et~al.(2019)Gottesman, Liu, Sussex, Brunskill, and
  Doshi-Velez]{gottesman2019combining}
Gottesman, O., Liu, Y., Sussex, S., Brunskill, E., and Doshi-Velez, F.
\newblock Combining parametric and nonparametric models for off-policy
  evaluation.
\newblock In \emph{International Conference on Machine Learning}, pp.\
  2366--2375. PMLR, 2019.

\bibitem[Grenander(1981)]{grenander1981abstract}
Grenander, U.
\newblock \emph{Abstract inference}.
\newblock Wiley Series, New York, 1981.

\bibitem[Hanna et~al.(2017)Hanna, Thomas, Stone, and Niekum]{hanna2017data}
Hanna, J.~P., Thomas, P.~S., Stone, P., and Niekum, S.
\newblock Data-efficient policy evaluation through behavior policy search.
\newblock In \emph{International Conference on Machine Learning}, pp.\
  1394--1403. PMLR, 2017.

\bibitem[Hao et~al.(2021)Hao, Ji, Duan, Lu, Szepesvari, and
  Wang]{hao2021bootstrapping}
Hao, B., Ji, X., Duan, Y., Lu, H., Szepesvari, C., and Wang, M.
\newblock Bootstrapping fitted q-evaluation for off-policy inference.
\newblock In \emph{International Conference on Machine Learning}, pp.\
  4074--4084. PMLR, 2021.

\bibitem[Hu et~al.(2009)Hu, Zhang, and He]{hu2009efficient}
Hu, F., Zhang, L.-X., and He, X.
\newblock Efficient randomized-adaptive designs.
\newblock \emph{The Annals of Statistics}, pp.\  2543--2560, 2009.

\bibitem[Hu et~al.(2015)Hu, Zhu, and Hu]{hu2015unified}
Hu, J., Zhu, H., and Hu, F.
\newblock A unified family of covariate-adjusted response-adaptive designs
  based on efficiency and ethics.
\newblock \emph{Journal of the American Statistical Association}, 110\penalty0
  (509):\penalty0 357--367, 2015.

\bibitem[Hu \& Wager(2022)Hu and Wager]{hu2022switchback}
Hu, Y. and Wager, S.
\newblock Switchback experiments under geometric mixing.
\newblock \emph{arXiv preprint arXiv:2209.00197}, 2022.

\bibitem[Hu \& Wager(2023)Hu and Wager]{Hu2023off}
Hu, Y. and Wager, S.
\newblock Off-policy evaluation in partially observed markov decision processes
  under sequential ignorability.
\newblock \emph{The Annals of Statistics}, 51\penalty0 (4):\penalty0
  1561--1585, 2023.

\bibitem[Hudgens \& Halloran(2008)Hudgens and Halloran]{hudgens2008toward}
Hudgens, M.~G. and Halloran, M.~E.
\newblock Toward causal inference with interference.
\newblock \emph{Journal of the American Statistical Association}, 103\penalty0
  (482):\penalty0 832--842, 2008.

\bibitem[Imbens \& Rubin(2015)Imbens and Rubin]{imbens2015causal}
Imbens, G.~W. and Rubin, D.~B.
\newblock \emph{Causal inference in statistics, social, and biomedical
  sciences}.
\newblock Cambridge university press, 2015.

\bibitem[Jia et~al.(2023)Jia, Kallus, and Yu]{jia2023clustered}
Jia, S., Kallus, N., and Yu, C.~L.
\newblock Clustered switchback experiments: Near-optimal rates under
  spatiotemporal interference.
\newblock \emph{arXiv preprint arXiv:2312.15574}, 2023.

\bibitem[Jia et~al.(2024)Jia, Frazier, and Kallus]{jia2024multi}
Jia, S., Frazier, P., and Kallus, N.
\newblock Multi-armed bandits with interference.
\newblock \emph{arXiv preprint arXiv:2402.01845}, 2024.

\bibitem[Jiang \& Li(2016)Jiang and Li]{jiang2016doubly}
Jiang, N. and Li, L.
\newblock Doubly robust off-policy value evaluation for reinforcement learning.
\newblock In \emph{International Conference on Machine Learning}, pp.\
  652--661. PMLR, 2016.

\bibitem[Jin et~al.(2020)Jin, Yang, Wang, and Jordan]{jin2020provably}
Jin, C., Yang, Z., Wang, Z., and Jordan, M.~I.
\newblock Provably efficient reinforcement learning with linear function
  approximation.
\newblock In \emph{Conference on Learning Theory}, pp.\  2137--2143. PMLR,
  2020.

\bibitem[Johari et~al.(2017)Johari, Koomen, Pekelis, and
  Walsh]{johari2017peeking}
Johari, R., Koomen, P., Pekelis, L., and Walsh, D.
\newblock Peeking at a/b tests: Why it matters, and what to do about it.
\newblock In \emph{Proceedings of the 23rd ACM SIGKDD International Conference
  on Knowledge Discovery and Data Mining}, pp.\  1517--1525, 2017.

\bibitem[Johari et~al.(2022)Johari, Li, Liskovich, and
  Weintraub]{johari2022experimental}
Johari, R., Li, H., Liskovich, I., and Weintraub, G.~Y.
\newblock Experimental design in two-sided platforms: An analysis of bias.
\newblock \emph{Management Science}, 68\penalty0 (10):\penalty0 7065--7791,
  2022.

\bibitem[Jones \& Goos(2009)Jones and Goos]{jones2009d}
Jones, B. and Goos, P.
\newblock D-optimal design of split-split-plot experiments.
\newblock \emph{Biometrika}, 96\penalty0 (1):\penalty0 67--82, 2009.

\bibitem[Kallus \& Uehara(2020)Kallus and Uehara]{kallus2020double}
Kallus, N. and Uehara, M.
\newblock Double reinforcement learning for efficient off-policy evaluation in
  markov decision processes.
\newblock \emph{Journal of Machine Learning Research}, 21\penalty0
  (167):\penalty0 1--63, 2020.

\bibitem[Kallus \& Uehara(2022)Kallus and Uehara]{kallus2022efficiently}
Kallus, N. and Uehara, M.
\newblock Efficiently breaking the curse of horizon in off-policy evaluation
  with double reinforcement learning.
\newblock \emph{Operations Research}, 70\penalty0 (6):\penalty0 3282--3302,
  2022.

\bibitem[Kato et~al.(2024)Kato, Oga, Komatsubara, and Inokuchi]{kato2024active}
Kato, M., Oga, A., Komatsubara, W., and Inokuchi, R.
\newblock Active adaptive experimental design for treatment effect estimation
  with covariate choices.
\newblock \emph{arXiv preprint arXiv:2403.03589}, 2024.

\bibitem[Kong et~al.(2021)Kong, Yuan, and Zheng]{kong2021approximate}
Kong, X., Yuan, M., and Zheng, W.
\newblock Approximate and exact designs for total effects.
\newblock \emph{The Annals of Statistics}, 49\penalty0 (3):\penalty0
  1594--1625, 2021.

\bibitem[Krishnamurthy(2016)]{krishnamurthy2016partially}
Krishnamurthy, V.
\newblock \emph{Partially observed Markov decision processes}.
\newblock Cambridge university press, 2016.

\bibitem[Langley(2000)]{langley00}
Langley, P.
\newblock Crafting papers on machine learning.
\newblock In Langley, P. (ed.), \emph{Proceedings of the 17th International
  Conference on Machine Learning (ICML 2000)}, pp.\  1207--1216, Stanford, CA,
  2000. Morgan Kaufmann.

\bibitem[Larsen et~al.(2024)Larsen, Stallrich, Sengupta, Deng, Kohavi, and
  Stevens]{larsen2023statistical}
Larsen, N., Stallrich, J., Sengupta, S., Deng, A., Kohavi, R., and Stevens,
  N.~T.
\newblock Statistical challenges in online controlled experiments: A review of
  a/b testing methodology.
\newblock \emph{The American Statistician}, 78\penalty0 (2):\penalty0 135--149,
  2024.

\bibitem[Leung(2022)]{leung2022rate}
Leung, M.~P.
\newblock Rate-optimal cluster-randomized designs for spatial interference.
\newblock \emph{The Annals of Statistics}, 50\penalty0 (5):\penalty0
  3064--3087, 2022.

\bibitem[Leung \& Loupos(2022)Leung and Loupos]{leung2022graph}
Leung, M.~P. and Loupos, P.
\newblock Graph neural networks for causal inference under network confounding.
\newblock \emph{arXiv preprint arXiv:2211.07823}, 2022.

\bibitem[Li et~al.(2021)Li, Chen, Chi, Gu, and Wei]{li2021sample}
Li, G., Chen, Y., Chi, Y., Gu, Y., and Wei, Y.
\newblock Sample-efficient reinforcement learning is feasible for linearly
  realizable mdps with limited revisiting.
\newblock \emph{Advances in Neural Information Processing Systems},
  34:\penalty0 16671--16685, 2021.

\bibitem[Li et~al.(2023{\natexlab{a}})Li, Wu, Chi, Ma, Rinaldo, and
  Wei]{li2023sharp}
Li, G., Wu, W., Chi, Y., Ma, C., Rinaldo, A., and Wei, Y.
\newblock Sharp high-probability sample complexities for policy evaluation with
  linear function approximation.
\newblock \emph{arXiv preprint arXiv:2305.19001}, 2023{\natexlab{a}}.

\bibitem[Li et~al.(2022)Li, Zhao, Johari, and Weintraub]{li2022interference}
Li, H., Zhao, G., Johari, R., and Weintraub, G.~Y.
\newblock Interference, bias, and variance in two-sided marketplace
  experimentation: Guidance for platforms.
\newblock In \emph{Proceedings of the ACM Web Conference 2022}, pp.\  182--192,
  2022.

\bibitem[Li et~al.(2023{\natexlab{b}})Li, Shi, Wang, Zhou,
  et~al.]{li2023optimal}
Li, T., Shi, C., Wang, J., Zhou, F., et~al.
\newblock Optimal treatment allocation for efficient policy evaluation in
  sequential decision making.
\newblock \emph{Advances in Neural Information Processing Systems},
  36:\penalty0 48890--48905, 2023{\natexlab{b}}.

\bibitem[Li et~al.(2024)Li, Shi, Lu, Li, and Zhu]{li2023evaluating}
Li, T., Shi, C., Lu, Z., Li, Y., and Zhu, H.
\newblock Evaluating dynamic conditional quantile treatment effects with
  applications in ridesharing.
\newblock \emph{Journal of the American Statistical Association}, 119\penalty0
  (547):\penalty0 1736--1750, 2024.

\bibitem[Li et~al.(2019)Li, Ding, Lin, Yang, and Liu]{li2019randomization}
Li, X., Ding, P., Lin, Q., Yang, D., and Liu, J.~S.
\newblock Randomization inference for peer effects.
\newblock \emph{Journal of the American Statistical Association}, 114\penalty0
  (528):\penalty0 1651--1664, 2019.

\bibitem[Liang \& Recht(2025)Liang and Recht]{liang2025randomization}
Liang, T. and Recht, B.
\newblock Randomization inference when n equals one.
\newblock \emph{Biometrika}, 2025:\penalty0 asaf013, 2025.
\newblock \doi{10.1093/biomet/asaf013}.

\bibitem[Liao et~al.(2021)Liao, Klasnja, and Murphy]{liao2021ope_jasa}
Liao, P., Klasnja, P., and Murphy, S.
\newblock Off-policy estimation of long-term average outcomes with applications
  to mobile health.
\newblock \emph{Journal of the American Statistical Association}, 116\penalty0
  (533):\penalty0 382--391, 2021.

\bibitem[Liao et~al.(2022)Liao, Qi, Wan, Klasnja, and Murphy]{liao2022batch}
Liao, P., Qi, Z., Wan, R., Klasnja, P., and Murphy, S.~A.
\newblock Batch policy learning in average reward markov decision processes.
\newblock \emph{Annals of statistics}, 50\penalty0 (6):\penalty0 3364, 2022.

\bibitem[Lim et~al.(2022)Lim, Novoseller, Ichnowski, Huang, and
  Goldberg]{lim2022policy}
Lim, V., Novoseller, E., Ichnowski, J., Huang, H., and Goldberg, K.
\newblock Policy-based bayesian experimental design for non-differentiable
  implicit models.
\newblock \emph{arXiv preprint arXiv:2203.04272}, 2022.

\bibitem[Liu et~al.(2020)Liu, Mao, and Kang]{liu2020trustworthy}
Liu, M., Mao, J., and Kang, K.
\newblock Trustworthy online marketplace experimentation with budget-split
  design.
\newblock \emph{arXiv preprint arXiv:2012.08724}, 2020.

\bibitem[Liu et~al.(2018)Liu, Li, Tang, and Zhou]{liu2018breaking}
Liu, Q., Li, L., Tang, Z., and Zhou, D.
\newblock Breaking the curse of horizon: Infinite-horizon off-policy
  estimation.
\newblock In \emph{Advances in Neural Information Processing Systems
  (NeurIPS)}, volume~31, 2018.

\bibitem[Liu et~al.(2023)Liu, Tu, Zhang, and Chen]{liu2023online}
Liu, W., Tu, J., Zhang, Y., and Chen, X.
\newblock Online estimation and inference for robust policy evaluation in
  reinforcement learning.
\newblock \emph{arXiv preprint arXiv:2310.02581}, 2023.

\bibitem[Luckett et~al.(2020)Luckett, Laber, Kahkoska, Maahs, Mayer-Davis, and
  Kosorok]{luckett2019}
Luckett, D.~J., Laber, E.~B., Kahkoska, A.~R., Maahs, D.~M., Mayer-Davis, E.,
  and Kosorok, M.~R.
\newblock Estimating dynamic treatment regimes in mobile health using
  v-learning.
\newblock \emph{Journal of the american statistical association}, 2020.

\bibitem[Luedtke \& Van Der~Laan(2016)Luedtke and Van
  Der~Laan]{luedtke2016statistical}
Luedtke, A.~R. and Van Der~Laan, M.~J.
\newblock Statistical inference for the mean outcome under a possibly
  non-unique optimal treatment strategy.
\newblock \emph{Annals of statistics}, 44\penalty0 (2):\penalty0 713, 2016.

\bibitem[Luo et~al.(2024)Luo, Yang, Shi, Yao, Ye, and Zhu]{luo2022policy}
Luo, S., Yang, Y., Shi, C., Yao, F., Ye, J., and Zhu, H.
\newblock Policy evaluation for temporal and/or spatial dependent experiments.
\newblock \emph{Journal of the Royal Statistical Society Series B: Statistical
  Methodology}, pp.\  1--27, 2024.

\bibitem[Mukherjee et~al.(2022)Mukherjee, Hanna, and Nowak]{mukherjee2022revar}
Mukherjee, S., Hanna, J.~P., and Nowak, R.~D.
\newblock Revar: Strengthening policy evaluation via reduced variance sampling.
\newblock In \emph{Uncertainty in Artificial Intelligence}, pp.\  1413--1422.
  PMLR, 2022.

\bibitem[Munro et~al.(2021)Munro, Wager, and Xu]{munro2021treatment}
Munro, E., Wager, S., and Xu, K.
\newblock Treatment effects in market equilibrium.
\newblock \emph{arXiv preprint arXiv:2109.11647}, 2021.

\bibitem[Nachum et~al.(2019)Nachum, Chow, Dai, and Li]{nachum2019dualdice}
Nachum, O., Chow, Y., Dai, B., and Li, L.
\newblock Dualdice: Behavior-agnostic estimation of discounted stationary
  distribution corrections.
\newblock \emph{Advances in neural information processing systems}, 32, 2019.

\bibitem[Ni et~al.(2023)Ni, Bojinov, and Zhao]{ni_design_2023}
Ni, T., Bojinov, I., and Zhao, J.
\newblock Design of panel experiments with spatial and temporal interference.
\newblock \emph{Available at SSRN 4466598}, 2023.

\bibitem[Pollmann(2020)]{pollmann2020causal}
Pollmann, M.
\newblock Causal inference for spatial treatments.
\newblock \emph{arXiv preprint arXiv:2011.00373}, 2020.

\bibitem[Puterman(2014)]{puterman2014markov}
Puterman, M.~L.
\newblock \emph{Markov decision processes: discrete stochastic dynamic
  programming}.
\newblock John Wiley \& Sons, 2014.

\bibitem[Qin et~al.(2025)Qin, Tang, Li, Zhu, and Ye]{qin2025reinforcement}
Qin, Z., Tang, X., Li, Q., Zhu, H., and Ye, J.
\newblock \emph{Reinforcement Learning in the Ridesharing Marketplace}.
\newblock Springer, 2025.

\bibitem[Quin et~al.(2024)Quin, Weyns, Galster, and Silva]{quin2023b}
Quin, F., Weyns, D., Galster, M., and Silva, C.~C.
\newblock A/b testing: a systematic literature review.
\newblock \emph{Journal of Systems and Software}, pp.\  112011, 2024.

\bibitem[Reich et~al.(2021)Reich, Yang, Guan, Giffin, Miller, and
  Rappold]{reich2021review}
Reich, B.~J., Yang, S., Guan, Y., Giffin, A.~B., Miller, M.~J., and Rappold, A.
\newblock A review of spatial causal inference methods for environmental and
  epidemiological applications.
\newblock \emph{International Statistical Review}, 89\penalty0 (3):\penalty0
  605--634, 2021.

\bibitem[Robins(1986)]{robins1986new}
Robins, J.
\newblock A new approach to causal inference in mortality studies with a
  sustained exposure period—application to control of the healthy worker
  survivor effect.
\newblock \emph{Mathematical modelling}, 7\penalty0 (9-12):\penalty0
  1393--1512, 1986.

\bibitem[Rosenblum et~al.(2020)Rosenblum, Fang, and Liu]{rosenblum2020optimal}
Rosenblum, M., Fang, E.~X., and Liu, H.
\newblock Optimal, two-stage, adaptive enrichment designs for randomized
  trials, using sparse linear programming.
\newblock \emph{Journal of the Royal Statistical Society Series B: Statistical
  Methodology}, 82\penalty0 (3):\penalty0 749--772, 2020.

\bibitem[Shi(2025)]{shi2025statistical}
Shi, C.
\newblock Statistical inference in reinforcement learning: A selective survey.
\newblock \emph{arXiv preprint arXiv:2502.16195}, 2025.

\bibitem[Shi et~al.(2020)Shi, Wan, Song, Lu, and Leng]{shi2020does}
Shi, C., Wan, R., Song, R., Lu, W., and Leng, L.
\newblock Does the markov decision process fit the data: Testing for the markov
  property in sequential decision making.
\newblock In \emph{International Conference on Machine Learning}, pp.\
  8807--8817. PMLR, 2020.

\bibitem[Shi et~al.(2022)Shi, Zhang, Lu, and Song]{shi2022statistical}
Shi, C., Zhang, S., Lu, W., and Song, R.
\newblock Statistical inference of the value function for reinforcement
  learning in infinite-horizon settings.
\newblock \emph{Journal of the Royal Statistical Society Series B: Statistical
  Methodology}, 84\penalty0 (3):\penalty0 765--793, 2022.

\bibitem[Shi et~al.(2023{\natexlab{a}})Shi, Wan, Song, Luo, Zhu, and
  Song]{shi2023multiagent}
Shi, C., Wan, R., Song, G., Luo, S., Zhu, H., and Song, R.
\newblock A multiagent reinforcement learning framework for off-policy
  evaluation in two-sided markets.
\newblock \emph{The Annals of Applied Statistics}, 17\penalty0 (4):\penalty0
  2701--2722, 2023{\natexlab{a}}.

\bibitem[Shi et~al.(2023{\natexlab{b}})Shi, Wang, Luo, Zhu, Ye, and
  Song]{shi2023dynamic}
Shi, C., Wang, X., Luo, S., Zhu, H., Ye, J., and Song, R.
\newblock Dynamic causal effects evaluation in a/b testing with a reinforcement
  learning framework.
\newblock \emph{Journal of the American Statistical Association}, 118\penalty0
  (543):\penalty0 2059--2071, 2023{\natexlab{b}}.

\bibitem[Shi et~al.(2024)Shi, Zhu, Shen, Luo, Zhu, and Song]{shi2024off}
Shi, C., Zhu, J., Shen, Y., Luo, S., Zhu, H., and Song, R.
\newblock Off-policy confidence interval estimation with confounded markov
  decision process.
\newblock \emph{Journal of the American Statistical Association}, 119\penalty0
  (545):\penalty0 273--284, 2024.

\bibitem[Shirani \& Bayati(2024)Shirani and Bayati]{causal_mess_network}
Shirani, S. and Bayati, M.
\newblock Causal message-passing for experiments with unknown and general
  network interference.
\newblock \emph{Proceedings of the National Academy of Sciences}, 121\penalty0
  (40):\penalty0 e2322232121, 2024.
\newblock \doi{10.1073/pnas.2322232121}.

\bibitem[Sobel \& Lindquist(2014)Sobel and Lindquist]{sobel_causal_2014}
Sobel, M.~E. and Lindquist, M.~A.
\newblock Causal {Inference} for {fMRI} {Time} {Series} {Data} with
  {Systematic} {Errors} of {Measurement} in a {Balanced} {On}/{Off} {Study} of
  {Social} {Evaluative} {Threat}.
\newblock \emph{Journal of the American Statistical Association}, 109\penalty0
  (507):\penalty0 967--976, July 2014.
\newblock \doi{10.1080/01621459.2014.922886}.

\bibitem[Sun et~al.(2024)Sun, Kong, Zhu, and Shi]{sun2024optimal}
Sun, K., Kong, L., Zhu, H., and Shi, C.
\newblock Arma-design: Optimal treatment allocation strategies for a/b testing
  in partially observable time series experiments.
\newblock \emph{arXiv preprint arXiv:2408.05342}, 2024.

\bibitem[Sutton \& Barto(2018)Sutton and Barto]{sutton2018reinforcement}
Sutton, R.~S. and Barto, A.~G.
\newblock \emph{Reinforcement learning: An introduction}.
\newblock MIT press, 2018.

\bibitem[Sutton et~al.(2008)Sutton, Szepesv{\'a}ri, and
  Maei]{sutton2008convergent}
Sutton, R.~S., Szepesv{\'a}ri, C., and Maei, H.~R.
\newblock A convergent o(n) algorithm for off-policy temporal-difference
  learning with linear function approximation.
\newblock \emph{Advances in neural information processing systems}, 21\penalty0
  (21):\penalty0 1609--1616, 2008.

\bibitem[Tang et~al.(2019)Tang, Qin, Zhang, Wang, Xu, Ma, Zhu, and
  Ye]{Tang2019}
Tang, X., Qin, Z., Zhang, F., Wang, Z., Xu, Z., Ma, Y., Zhu, H., and Ye, J.
\newblock A deep value-network based approach for multi-driver order
  dispatching.
\newblock \emph{In The 25th ACM SIGKDD Conference on Knowledge Discovery and
  Data Mining (KDD’19)}, 25:\penalty0 1780--1790, 2019.

\bibitem[Tang et~al.(2022)Tang, Duan, Zhang, and Li]{tang2022reinforcement}
Tang, Z., Duan, Y., Zhang, S., and Li, L.
\newblock A reinforcement learning approach to estimating long-term treatment
  effects.
\newblock \emph{arXiv preprint arXiv:2210.07536}, 2022.

\bibitem[Tchetgen~Tchetgen et~al.(2021)Tchetgen~Tchetgen, Fulcher, and
  Shpitser]{tchetgen2021auto}
Tchetgen~Tchetgen, E.~J., Fulcher, I.~R., and Shpitser, I.
\newblock Auto-g-computation of causal effects on a network.
\newblock \emph{Journal of the American Statistical Association}, 116\penalty0
  (534):\penalty0 833--844, 2021.

\bibitem[Thams et~al.(2023)Thams, Saengkyongam, Pfister, and
  Peters]{thams2023statistical}
Thams, N., Saengkyongam, S., Pfister, N., and Peters, J.
\newblock Statistical testing under distributional shifts.
\newblock \emph{Journal of the Royal Statistical Society Series B: Statistical
  Methodology}, 85\penalty0 (3):\penalty0 597--663, 2023.

\bibitem[Thomas \& Brunskill(2016)Thomas and Brunskill]{thomas2016data}
Thomas, P. and Brunskill, E.
\newblock Data-efficient off-policy policy evaluation for reinforcement
  learning.
\newblock In \emph{International Conference on Machine Learning}, pp.\
  2139--2148, 2016.

\bibitem[Thomas et~al.(2015)Thomas, Theocharous, and Ghavamzadeh]{thomas2015}
Thomas, P.~S., Theocharous, G., and Ghavamzadeh, M.
\newblock High-confidence off-policy evaluation.
\newblock In \emph{Twenty-Ninth AAAI Conference on Artificial Intelligence},
  2015.

\bibitem[Tropp(2012)]{tropp2012user}
Tropp, J.~A.
\newblock User-friendly tail bounds for sums of random matrices.
\newblock \emph{Foundations of computational mathematics}, 12:\penalty0
  389--434, 2012.

\bibitem[Tsiatis(2007)]{tsiatis2007semiparametric}
Tsiatis, A.
\newblock \emph{Semiparametric Theory and Missing Data}.
\newblock Springer Science \& Business Media, 2007.

\bibitem[Uehara et~al.(2020)Uehara, Huang, and Jiang]{uehara2020minimax}
Uehara, M., Huang, J., and Jiang, N.
\newblock Minimax weight and q-function learning for off-policy evaluation.
\newblock In \emph{International Conference on Machine Learning}, pp.\
  9659--9668. PMLR, 2020.

\bibitem[Uehara et~al.(2022)Uehara, Shi, and Kallus]{uehara2022review}
Uehara, M., Shi, C., and Kallus, N.
\newblock A review of off-policy evaluation in reinforcement learning.
\newblock \emph{arXiv preprint arXiv:2212.06355}, 2022.

\bibitem[Uehara et~al.(2023)Uehara, Imaizumi, Jiang, and
  Kallus]{uehara2023efficiently}
Uehara, M., Imaizumi, M., Jiang, N., and Kallus, N.
\newblock Efficiently estimating markovian average treatment effects in
  infinite-horizon settings.
\newblock In \emph{Advances in Neural Information Processing Systems
  (NeurIPS)}, volume~36, 2023.

\bibitem[Ugander et~al.(2013)Ugander, Karrer, Backstrom, and
  Kleinberg]{ugander2013graph}
Ugander, J., Karrer, B., Backstrom, L., and Kleinberg, J.
\newblock Graph cluster randomization: Network exposure to multiple universes.
\newblock In \emph{Proceedings of the 19th ACM SIGKDD International Conference
  on Knowledge Discovery and Data Mining}, pp.\  329--337, 2013.

\bibitem[Van \& Wellner(1996)Van and Wellner]{van1996weak}
Van, D. and Wellner, J.~A.
\newblock \emph{Weak convergence and empirical processes}.
\newblock Springer,, 1996.

\bibitem[Viviano \& Bradic(2023)Viviano and Bradic]{viviano2023synthetic}
Viviano, D. and Bradic, J.
\newblock Synthetic learner: model-free inference on treatments over time.
\newblock \emph{Journal of Econometrics}, 234\penalty0 (2):\penalty0 691--713,
  2023.

\bibitem[Viviano et~al.(2023)Viviano, Lei, Imbens, Karrer, Schrijvers, and
  Shi]{viviano2023causal}
Viviano, D., Lei, L., Imbens, G., Karrer, B., Schrijvers, O., and Shi, L.
\newblock Causal clustering: design of cluster experiments under network
  interference.
\newblock \emph{arXiv preprint arXiv:2310.14983}, 2023.

\bibitem[Wahba(1975)]{wahba1975smoothing}
Wahba, G.
\newblock Smoothing noisy data with spline functions.
\newblock \emph{Numerische mathematik}, 24\penalty0 (5):\penalty0 383--393,
  1975.

\bibitem[Wan et~al.(2022)Wan, Kveton, and Song]{wan2022safe}
Wan, R., Kveton, B., and Song, R.
\newblock Safe exploration for efficient policy evaluation and comparison.
\newblock In \emph{International Conference on Machine Learning}, pp.\
  22491--22511. PMLR, 2022.

\bibitem[Wang et~al.(2023)Wang, Qi, and Wong]{wang2021projected}
Wang, J., Qi, Z., and Wong, R.~K.
\newblock Projected state-action balancing weights for offline reinforcement
  learning.
\newblock \emph{The Annals of Statistics}, 51\penalty0 (4):\penalty0
  1639--1665, 2023.

\bibitem[Wang et~al.(2024)Wang, Li, and Wu]{wang2024off}
Wang, W., Li, Y., and Wu, X.
\newblock Off-policy evaluation for tabular reinforcement learning with
  synthetic trajectories.
\newblock \emph{Statistics and Computing}, 34\penalty0 (1):\penalty0 41, 2024.

\bibitem[Waudby-Smith et~al.(2022)Waudby-Smith, Wu, Ramdas, Karampatziakis, and
  Mineiro]{waudby2022anytime}
Waudby-Smith, I., Wu, L., Ramdas, A., Karampatziakis, N., and Mineiro, P.
\newblock Anytime-valid off-policy inference for contextual bandits.
\newblock \emph{ACM/JMS Journal of Data Science}, 2022.

\bibitem[Williams(1952)]{williams1952experimental}
Williams, R.
\newblock Experimental designs for serially correlated observations.
\newblock \emph{Biometrika}, 39\penalty0 (1/2):\penalty0 151--167, 1952.

\bibitem[Wong \& Zhu(2008)Wong and Zhu]{wong2008optimum}
Wong, W.~K. and Zhu, W.
\newblock Optimum treatment allocation rules under a variance heterogeneity
  model.
\newblock \emph{Statistics in Medicine}, 27\penalty0 (22):\penalty0 4581--4595,
  2008.

\bibitem[Wu et~al.(1986)]{wu1986jackknife}
Wu, C.-F.~J. et~al.
\newblock Jackknife, bootstrap and other resampling methods in regression
  analysis.
\newblock \emph{the Annals of Statistics}, 14\penalty0 (4):\penalty0
  1261--1295, 1986.

\bibitem[Xie et~al.(2023)Xie, Yang, and Zhang]{xie2023semiparametrically}
Xie, C., Yang, W., and Zhang, Z.
\newblock Semiparametrically efficient off-policy evaluation in linear markov
  decision processes.
\newblock In \emph{International Conference on Machine Learning}. PMLR, 2023.

\bibitem[Xie et~al.(2019)Xie, Ma, and Wang]{xie2019towards}
Xie, T., Ma, Y., and Wang, Y.-X.
\newblock Towards optimal off-policy evaluation for reinforcement learning with
  marginalized importance sampling.
\newblock \emph{Advances in neural information processing systems}, 32, 2019.

\bibitem[Xiong et~al.(2024)Xiong, Chin, and Taylor]{Xiong2023}
Xiong, R., Chin, A., and Taylor, S.~J.
\newblock Data-driven switchback experiments: Theoretical tradeoffs and
  empirical bayes designs.
\newblock \emph{arXiv preprint arXiv:2406.06768}, 2024.

\bibitem[Xu et~al.(2023)Xu, Zhu, Shi, Luo, and Song]{xu2023instrumental}
Xu, Y., Zhu, J., Shi, C., Luo, S., and Song, R.
\newblock An instrumental variable approach to confounded off-policy
  evaluation.
\newblock In \emph{International Conference on Machine Learning}, pp.\
  38848--38880. PMLR, 2023.

\bibitem[Xu et~al.(2018)Xu, Li, Guan, Zhang, Li, Nan, Liu, Bian, and
  Ye]{Xu2018}
Xu, Z., Li, Z., Guan, Q., Zhang, D., Li, Q., Nan, J., Liu, C., Bian, W., and
  Ye, J.
\newblock Large-scale order dispatch in on-demand ride-hailing platforms: A
  learning and planning approach.
\newblock In \emph{Proceedings of the 24th ACM SIGKDD international conference
  on knowledge discovery \& data mining}, pp.\  905--913, 2018.

\bibitem[Yang et~al.(2024)Yang, Shi, Yao, Wang, and Zhu]{yang2024spatially}
Yang, Y., Shi, C., Yao, F., Wang, S., and Zhu, H.
\newblock Spatially randomized designs can enhance policy evaluation.
\newblock \emph{arXiv preprint arXiv:2403.11400}, 2024.

\bibitem[Yin \& Wang(2020)Yin and Wang]{yin2020asymptotically}
Yin, M. and Wang, Y.-X.
\newblock Asymptotically efficient off-policy evaluation for tabular
  reinforcement learning.
\newblock In \emph{International Conference on Artificial Intelligence and
  Statistics}, pp.\  3948--3958. PMLR, 2020.

\bibitem[Yu et~al.(2024)Yu, Fang, Peng, Qi, Zhou, and Shi]{yu2024two}
Yu, S., Fang, S., Peng, R., Qi, Z., Zhou, F., and Shi, C.
\newblock Two-way deconfounder for off-policy evaluation in causal
  reinforcement learning.
\newblock In \emph{Advances in Neural Information Processing Systems},
  volume~37, pp.\  78169--78200, 2024.

\bibitem[Zeger(1988)]{zeger1988regression}
Zeger, S.~L.
\newblock A regression model for time series of counts.
\newblock \emph{Biometrika}, 75\penalty0 (4):\penalty0 621--629, 1988.

\bibitem[Zhan et~al.(2024)Zhan, Han, Hu, and Jiang]{zhan2024estimating}
Zhan, R., Han, S., Hu, Y., and Jiang, Z.
\newblock Estimating treatment effects under recommender interference: A
  structured neural networks approach.
\newblock \emph{arXiv preprint arXiv:2406.14380}, 2024.

\bibitem[Zhang et~al.(2013)Zhang, Tsiatis, Laber, and
  Davidian]{zhang2013robust}
Zhang, B., Tsiatis, A.~A., Laber, E.~B., and Davidian, M.
\newblock Robust estimation of optimal dynamic treatment regimes for sequential
  treatment decisions.
\newblock \emph{Biometrika}, 100\penalty0 (3):\penalty0 681--694, 2013.

\bibitem[Zhang et~al.(2024)Zhang, Yang, and Yao]{zhang2024spatial}
Zhang, W., Yang, Y., and Yao, F.
\newblock Spatial interference detection in treatment effect model.
\newblock \emph{arXiv preprint arXiv:2409.04836}, 2024.

\bibitem[Zhang \& Wang(2024)Zhang and Wang]{zhang2024online}
Zhang, Z. and Wang, Z.
\newblock Online experimental design with estimation-regret trade-off under
  network interference.
\newblock \emph{arXiv preprint arXiv:2412.03727}, 2024.

\bibitem[Zhong et~al.(2022)Zhong, Zhang, Sch{\"a}fer, Albrecht, and
  Hanna]{zhong2022robust}
Zhong, R., Zhang, D., Sch{\"a}fer, L., Albrecht, S., and Hanna, J.
\newblock Robust on-policy sampling for data-efficient policy evaluation in
  reinforcement learning.
\newblock \emph{Advances in Neural Information Processing Systems},
  35:\penalty0 37376--37388, 2022.

\bibitem[Zhou et~al.(2021)Zhou, Luo, Qie, Ye, and Zhu]{zhou2021graph}
Zhou, F., Luo, S., Qie, X., Ye, J., and Zhu, H.
\newblock Graph-based equilibrium metrics for dynamic supply--demand systems
  with applications to ride-sourcing platforms.
\newblock \emph{Journal of the American Statistical Association}, 116\penalty0
  (536):\penalty0 1688--1699, 2021.

\bibitem[Zhou et~al.(2025)Zhou, Hanna, Zhu, Yang, and Shi]{zhou2025IS}
Zhou, H., Hanna, J.~P., Zhu, J., Yang, Y., and Shi, C.
\newblock Demystifying the paradox of importance sampling with an estimated
  history-dependent behavior policy in off-policy evaluation.
\newblock In \emph{International conference on machine learning}. PMLR, 2025.

\bibitem[Zhou et~al.(2023)Zhou, Shi, Li, and Yao]{zhou2023testing}
Zhou, Y., Shi, C., Li, L., and Yao, Q.
\newblock Testing for the markov property in time series via deep conditional
  generative learning.
\newblock \emph{Journal of the Royal Statistical Society Series B: Statistical
  Methodology}, 85\penalty0 (4):\penalty0 1204--1222, 2023.

\bibitem[Zhu et~al.(2025)Zhu, Li, Zhou, Lin, Lin, and
  Shi]{zhu2025causalgraphcut}
Zhu, J., Li, J., Zhou, H., Lin, Y., Lin, Z., and Shi, C.
\newblock Balancing interference and correlation in spatial experimental
  designs: A causal graph cut approach.
\newblock In \emph{International conference on machine learning}. PMLR, 2025.

\end{thebibliography}
\bibliographystyle{icml2025}

\newpage
\appendix
\onecolumn
\renewcommand{\thefigure}{S\arabic{figure}} 
\renewcommand{\thetable}{S\arabic{table}}   
\setcounter{figure}{0} 
\setcounter{table}{0}  

\section{Implementation Details}
\label{sec:implementation}
In this section, we first detail the parametric estimation of the model-based method using the OLS approach, as described in Section \ref{subsec:methods}. Next, we introduce a modified LSTD estimator which is designed to improve the efficiency of the original LSTD estimator and is implemented in our numerical study. Finally, we detail our implementation of the DRL estimator, focusing on the estimation of the IS ratio and the value function.

\textit{\textbf{OLS Estimation.}} 
Given the observational data $\{(S_{i,t}, A_{i,t}, R_{i,t}):1\le t\le T\}_{i=1}^n$, for each $1\le t\le T$, we deploy the OLS regression to the dataset $\{(S_{i,t},A_{i,t},R_{i,t}):1\le i\le n\}$ with $R_{i,t}$ as the response and $(1, S_{i,t},A_{i,t})$ as the predictor to compute the estimators $\widehat{\alpha}_t$, $\widehat{\beta}_t$ and $\widehat{\gamma}_t$. Similarly, we apply OLS to $\{(S_{i,t},A_{i,t},S_{i,t+1}^{(j)}):1\le i\le n\}$ with the $j-$th component of $S_{i,t+1}$, $S_{i,t+1}^{(j)}$ as the responses and $(1, S_{i,t},A_{i,t})$ as the predictor to estimate the $j$th element of $\phi_t$, the $j$th element of   $\Gamma_t$, as well as the $j$th row of $\Phi_t$. Concatenating all the estimators across $j$ produces $\widehat{\phi}_t$, $\widehat{\Gamma}_t$ and $\widehat{\Phi}_t$. With these estimators in hand, we plug them into \eqref{ate_est_formula} to compute the final estimators $\textrm{ATE}_{\textrm{AD}}$ and $\textrm{ATE}_{\textrm{SB}}^{(m)}$.

\textit{\textbf{Modified LSTD}}. In the original LSTD algorithm described in Section \ref{subsec:methods}, the value function at a specific time $t$ is estimated using only the data subset corresponding to that particular time. This approach might be inefficient when the system dynamics remain relatively consistent over time. To enhance estimation efficiency, we incorporate the time index into the state, resulting in an augmented state, denoted as $\widetilde{S}_t$ for each time 
$t$. We then approximate the value function using a linear combination of sieves. It is important to note that the basis functions $\varphi$ contains not only the bases for the original state but also those for the time component, addressing potential nonstationarity.

To lay down the foundation, for any $t$, $a$ and $\tilde{s}$, we first define a value function $V_{t:t}^a(\tilde{s})=\Mean^a (R_t|\widetilde{S}_t=\tilde{s},A_t=a)$. Next, we recursively define $V_{t-j:t}^a(\tilde{s})=\Mean^a [R_{t-j}+V_{t-j+1:t}^a(\widetilde{S}_{t-j+1})|\widetilde{S}_t=\tilde{s},A_t=a]$ for $j=1,2,\ldots,t-1$. Essentially, for any $t_1\le t_2$, $V_{t_1:t_2}^a$ represents the expected cumulative reward from time $t_1$ to $t_2$ starting from a given state at time $t_1$. Additionally, the final ATE estimator can be represented by $\Mean [V_{1:T}^1(\widetilde{S}_1)-V_{1:T}^0(\widetilde{S}_1)]$. 

It remains to estimate these doubly indexed value functions. A key observation is that, with the time index included in the state to account for nonstationarity, $V_{t_1:t_2}^a$ shall equal $V_{t_3:t_4}^a$ provided the time gaps $t_2-t_1$ and $t_4-t_3$ are equal. This allows us to aggregate all data over time to simultaneously estimate all value functions. 

Under the $m-$switchback design, our first step is to estimate $\{V_{t:t}^a\}_t$ by solving the following estimating equation:
\begin{eqnarray}\label{eqn:estVtt}
	\sum_{i,t} [R_{i,t}-\varphi^\top(\widetilde{S}_{i,t})\widehat{\theta}_{0,a,m}]\varphi(\widetilde{S}_{i,t})\mathbb{I}(A_{i,t}=a)=0.
\end{eqnarray}
From this, we compute $\widehat{V}_{t:t}^{a,m}(\tilde{s})$ as $\varphi^\top(\tilde{s}) \widehat{\theta}_{0,a,m}$. Next, we sequentially compute $\{V_{t-1:t}^a\}_t$,$\{V_{t-2:t}^a\}_t$, 
and so forth. Specifically, for each $j=1,\ldots,T-1$, we recursively solve the following estimating equation,
\begin{eqnarray}\label{eqn:estVt-jt}
	\sum_{i,t} [R_{i,t-j}+\varphi^\top(\widetilde{S}_{i,t-j+1})\widehat{\theta}_{j-1,a,m}-\varphi^\top(\widetilde{S}_{i,t-j})\widehat{\theta}_{j,a,m}]\varphi(\widetilde{S}_{i,t})\mathbb{I}(A_{i,t}=a)=0.
\end{eqnarray}
This leads to the derivation of $\widehat{\theta}_{j,a,m}$, based on which we set $\widehat{V}_{t-j:t}^{a,m}(\tilde{s})$ to $\varphi^\top(\tilde{s}) \widehat{\theta}_{j,a,m}$. Finally, we set the ATE estimator to 
\begin{eqnarray}\label{eqn:estmlstd}
	\frac{1}{n}\sum_{i=1}^n\varphi^\top(\widetilde{S}_{i,1})(\widehat{\theta}_{T-1,1,m}-\widehat{\theta}_{T-1,0,m}).
\end{eqnarray}
We provide a pseudocode in Algorithm \ref{algo:algo_mlstd} to summarize the aforementioned procedure. 

\begin{algorithm}[H]
	\caption{Estimating ATE via the modified LSTD.}\label{algo:algo_mlstd}
\begin{algorithmic}
		\STATE {\bfseries Input:} $\left\lbrace (S_{it}, R_{it}, A_{it}): 1 \leq i \leq n, 1 \leq t \leq T \right\rbrace$.
	\STATE {\bfseries Output:} The ATE estimator. 
	Solve \eqref{eqn:estVtt} to obtain $\widehat{\theta}_{T,a,m}$.\\
	\FOR{$t=1$ {\bfseries to} $T-1$}
	\STATE Solve \eqref{eqn:estVt-jt} to obtain  $\widehat{\theta}_{T-t,a,m}$.
	\ENDFOR
	\STATE{\bfseries Return:} The ATE estimator constructed according to \eqref{eqn:estmlstd}. 
\end{algorithmic}
\end{algorithm}

\textit{\textbf{Estimation of the IS ratios}}. We have devised a model-based approach to estimate the marginal IS ratio based on the linear model assumption presented in Equation \eqref{linear_mdp}. It is worth noting that both the numerator and the denominator of the ratio correspond to the marginal probability density/mass function of the state at a given time, given a sequence of past treatments it has received. As a result, we can focus on estimating these marginal density/mass functions to construct the ratio estimator.

Within the framework of the linear model assumption, we can express the state at time $t$, denoted as $S_t$, as follows:
\begin{equation}\label{eqn::IS_model}
	S_t=\sum_{k=1}^{t-1} \left( \Pi_{l=k+1}^{t-1} \Phi_l \right) \phi_k + \left( \Pi_{l=1}^{t-1} \Phi_l \right) S_{1}+{\sum_{k=1}^{t-1} \left( \Pi_{l=k+1}^{t -1} \Phi_l \right) \Gamma_k }A_{k}+\sum_{k=1}^{t -1} \left( \Pi_{l=k+1}^{t-1} \Phi_l \right) E_{k}.
\end{equation}
Where $\Pi_{l=k+1}^{t-1} \Phi_l:= \Phi_{k+1} \cdots \Phi_{t-1}$. Consequently, given a sequence of treatments, we can replace $\{A_k\}_k$ with this treatment sequence to derive the distribution function of $S_t$. To estimate this distribution function, we follow these steps:
\begin{enumerate}
	\item Estimate the model parameters in Equation \eqref{linear_mdp}.
	\item Impose models for the initial state and the residuals $\{E_k\}_k$. In our implementation, we utilize normal distributions, and estimate the mean and covariance matrix parameters within these distributions using the available data. According to \eqref{eqn::IS_model}, this ensures that each $S_t$ follows a normal distribution as well.
	\item Plug the estimated parameters obtained in the first two steps into \eqref{eqn::IS_model} to construct the mean and covariance matrix estimators for $S_t$. 
	\item Return the normal distribution function with the estimated mean and covariance matrix estimators obtained in Step 3. 
\end{enumerate}
A pseudocode summarizing our procedure is presented in Algorithm \ref{algo:algo_ratio}.

\begin{algorithm}[t]
	\caption{Model-based estimation of the IS ratio.}
	\label{algo:algo_ratio}
\begin{algorithmic}
		\STATE {\bfseries Input:} $\left\lbrace (S_{it}, R_{it}, A_{it}): 1 \leq i \leq n, 1 \leq t \leq T \right\rbrace$.
	\STATE {\bfseries Output:} The marginal IS ratio estimator.
	\STATE {\bfseries Initialization:} Calculate the least square estimators $\left\lbrace
	\widehat{\phi}_t \right\rbrace_{t=1}^{T-1}$, $\left\lbrace \widehat{\Phi}_t \right\rbrace_{t=1}^{T-1}$ and $\left\lbrace \widehat{\Gamma}_t \right\rbrace_{t=1}^{T-1}$ in \eqref{linear_mdp}. Impose models for $S_1$ and $\{E_k\}_k$, and estimate the associated model parameters.
	\FOR{$t=1$ {\bfseries to} $T-1$}
			\STATE \textbf{1.} Derive the probability density/mass function of $S_t$ using the aforementioned estimators and \eqref{eqn::IS_model}.
			\STATE \textbf{2.} Estimate the numerator $p_t^{a}$ by replacing $\{A_k\}_k$ in Equation \eqref{eqn::IS_model} with the target policy $a$.
			\STATE \textbf{3.} Estimating the denominator $p_{a,t}^m$ replacing $\{A_k\}_k$ in Equation \eqref{eqn::IS_model} with the treatment sequence under the $m$-switchback design given that $A_t=a$.
		\STATE	\textbf{4.} Calculate the ratio.
		\ENDFOR
\end{algorithmic}
\end{algorithm}

\textit{\textbf{Estimation of the value function}}. We devise a model-based approach to estimate the value function. Given the linear models outlined in \eqref{linear_mdp}, we can readily express the value function as follows:
\begin{equation}\label{eqn::mb-value}
	\begin{split}
		V_t^a(s)&= \sum_{j=t}^T \Mean^a(R_j|S_t=s) =\sum_{j=t}^T [
		\alpha_j+\gamma_j a
		+\beta_j^\top \Mean^a (S_j|S_t=s)]\\
		&=\sum_{j=t}^T \Big\{
		\alpha_j+\gamma_j a
		+\sum_{j=t}^T\beta_j^\top \Big[ 
		\sum_{k=t}^{j-1} (\Pi_{l=k+1}^{j-1} \Phi_l) \phi_k +(\Pi_{l=t}^{j-1} \Phi_l)s +
		\sum_{k=t}^{j-1} (\Pi_{l=k+1}^{j-1} \Phi_l) \Gamma_k a\Big]
		\Big\}.
	\end{split}
\end{equation}
This leads us to the approach of initially applying OLS to estimate the model parameters and subsequently incorporating these estimates into \eqref{eqn::mb-value} to formulate the value function estimators.

\section{Additional Discussions}
\label{sec:add_discussion}


Our analysis is built upon the Markov assumption (MA), which is fundamental to most RL-based estimators. In collaboration with our ride-sharing industry partner, we have observed that intervals of 30 minutes or 1 hour typically satisfy MA, showing strong lag-1 correlations with rapidly decaying higher-order correlations. This justifies the use of RL in our application.

When applied to more general applications, we recommend to properly select the interval length to meet MA as an initial step in the design of experiments (see Fig. \ref{fig:prac_guide}). If that's challenging, we further propose three approaches below, tailored to different degrees of violation of MA. Our current results directly extend to the first two cases, while the third case requires further investigation:
\begin{itemize}
    \setlength{\itemsep}{0pt}  
    \item \textbf{Mild violation:} Future observations depend on the current observation-action pair and a few past observations. This mild violation can be easily addressed by redefining the state to include recent past observations. With this modified state, MA is satisfied. Our RL-based estimators and theoretical results remain valid.
    \item \textbf{Moderate violation:} Future observations depend on a few past observation-action pairs. Here, the RL-based estimators remain applicable if the state includes these historical state-action pairs. However, our theoretical results on optimal designs must be adjusted. Preliminary analyses show that, under weak carryover effects and positively correlated residuals, the optimal switching interval extends to 1+k (where k is the number of included past actions) rather than switching at every time step. This is because each observed reward is affected by a k+1 consecutive actions, not just the most recent action. More frequent switching under these conditions causes considerable distributional shift, inflating the variance of the ATE estimator.
    \item \textbf{Severe violation:}  Data follows a POMDP. Although the existing literature provides doubly robust estimators and AD-like optimal designs \citep{li2023optimal} to handle such non-Markov MDPs, these estimators suffer from the "curse of horizon" \citep{liu2018breaking}. Recent advances propose more efficient POMDP-based estimators \citep{liang2025randomization} and designs \citep{sun2024optimal}; however, these proposals are limited to linear models. Extending these methodologies to accommodate more general estimation procedures (e.g., \citet{uehara2023efficiently}) represents an important direction for future research.
\end{itemize}
We also remark that in the first two cases, existing tests are available for testing the Markov assumption and for order selection \citep{chen2012testing,shi2020does,zhou2023testing}.

\section{Additional Experimental Results}\label{sec:addresults}
In this section, we systematically present the details of our numerical experiments and provide all relevant figures and results discussed in the preceding sections.

\textit{\textbf{DGP (Continued).}} The initial state for each day is drawn from a 3-dimensional multivariate normal distribution with zero mean and an identity covariance matrix. The coefficients in Linear DGP are specified as : $\{\gamma_t\}_t \stackrel{i.i.d.}{\sim } U[0.5,0.8], \quad \{\Gamma_t^{(j)}\}_{t,j} \stackrel{i.i.d.}{\sim} N(0, 0.3^2),\quad \{\Phi_t^{(j_1,j_2)}\}_{t,j_1,j_2}\stackrel{i.i.d.}{\sim }  U[-0.3, 0.3]$ and
\begin{equation*}
	\begin{split}
		&\{\alpha_t\}_t \stackrel{i.i.d.}{\sim }  \begin{cases}
			U[-1,-0.5]&\text{ with probability 0.5 } \\
			U[0.5, 1]&\text{ with probability 0.5 }
		\end{cases} ,	\{\beta_t^{(j)}\}_{t,j}  \stackrel{i.i.d.}{\sim }  \begin{cases}
			U[-0.3, -0.1] & \text{ with probability 0.5 } \\
			U[0.1, 0.3] & \text{ with probability 0.5 } 
		\end{cases}, \\
		&\{\phi_t^{(j)}\}_{t,j} \stackrel{i.i.d.}{\sim }  \begin{cases}
			U[-1,-0.5] & \text{ with probability 0.5 } \\
			U[0.5, 1]  & \text{ with probability 0.5 }
		\end{cases}.  
	\end{split}
\end{equation*}
Here, the superscript $j$ denotes the $j$th component of each vector, while $(j_1,j_2)$ indicates the element in the $j_1$th row and $j_2$th column of each matrix. Both the reward error $e_t$ and the residual in the state regression model $E_t = S_{t+1} - \Mean (S_{t+1}|A_t,S_t)$
are set to mean zero Gaussian noises.  The sequence $\{E_t\}_t$ is set to an i.i.d. multivariate Gaussian error process, with a covariance matrix 1.5 times the identity matrix, and it is independent of $\{e_t\}_t$.

In Nonlinear DGP, we consider the nonlinear reward function: $ r_t(a,s)=\alpha_t +2\beta_t^\top [\sin (sa)+\cos(s)]^2 +3(\beta_t^\top s) \gamma_t a +[ a \gamma_t+\cos(a \gamma_t) ]^2$, where the sine, cosine, and square functions are applied element-wise to each component of the vector. The state regression function remains linear and identical to the one presented in  \eqref{linear_mdp}. All model parameters, including $\{\alpha_t\}_t$, $\{\beta_t\}_t$, $\{\gamma_t\}_t$, $\{\Gamma_t\}_t$, $\{\phi_t\}_t$, $n$ and $T$, are the same as those in the above Linear DGP, with the exception of $\Phi_t^{(j_1, j_2)}\stackrel{i.i.d.}{\sim }  U[-0.6, 0.6]$ for $j_1, j_2=1, 2, 3$.

\graphicspath{{figures_unit/corollary_figs}}
\begin{figure}[ht]
    \centering
    \begin{minipage}{0.3\linewidth}
        \centering
        \includegraphics[width=\linewidth]{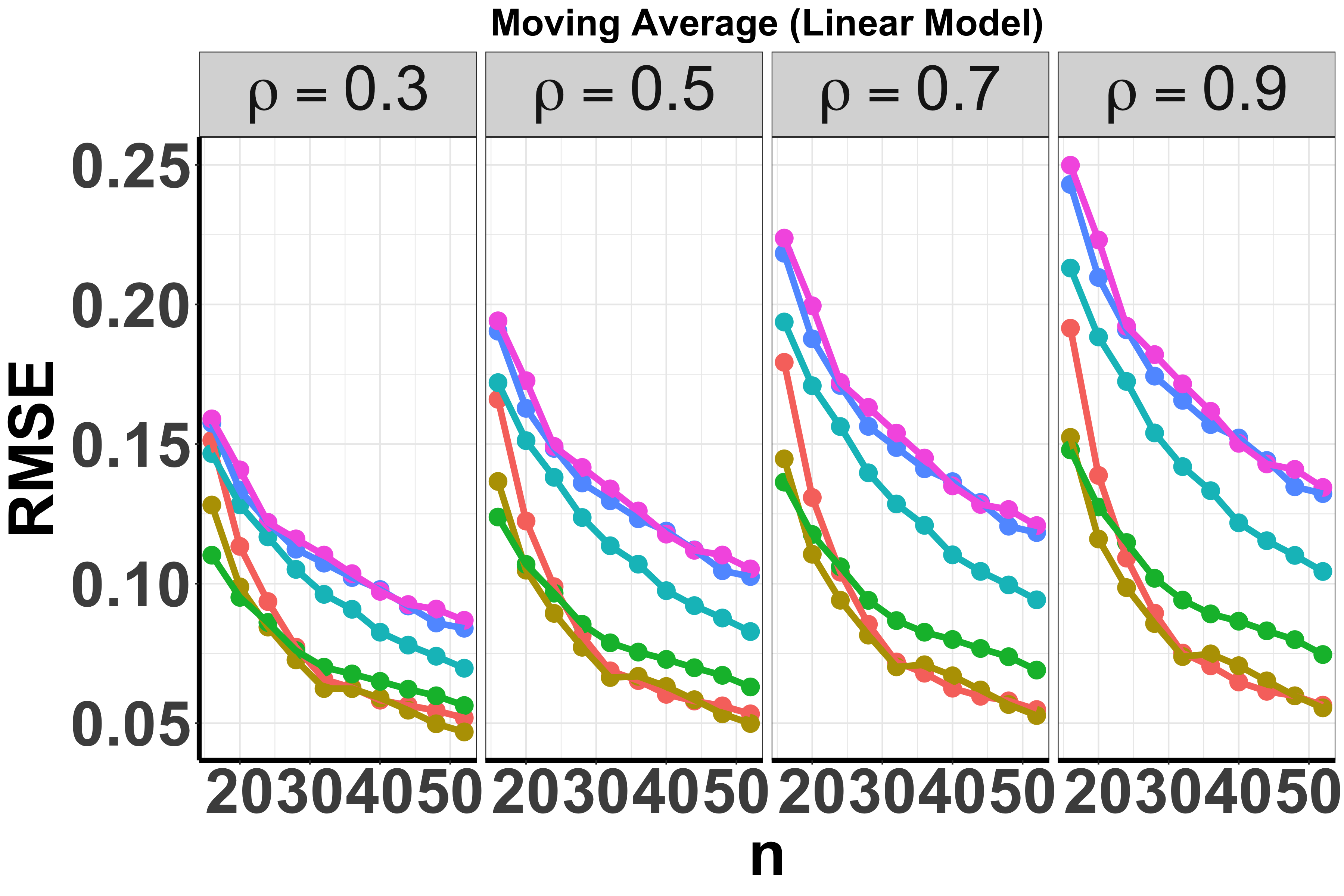}
    \end{minipage}
    \hspace{0.1cm}  
    \begin{minipage}{0.3\linewidth}
        \centering
        \includegraphics[width=\linewidth]{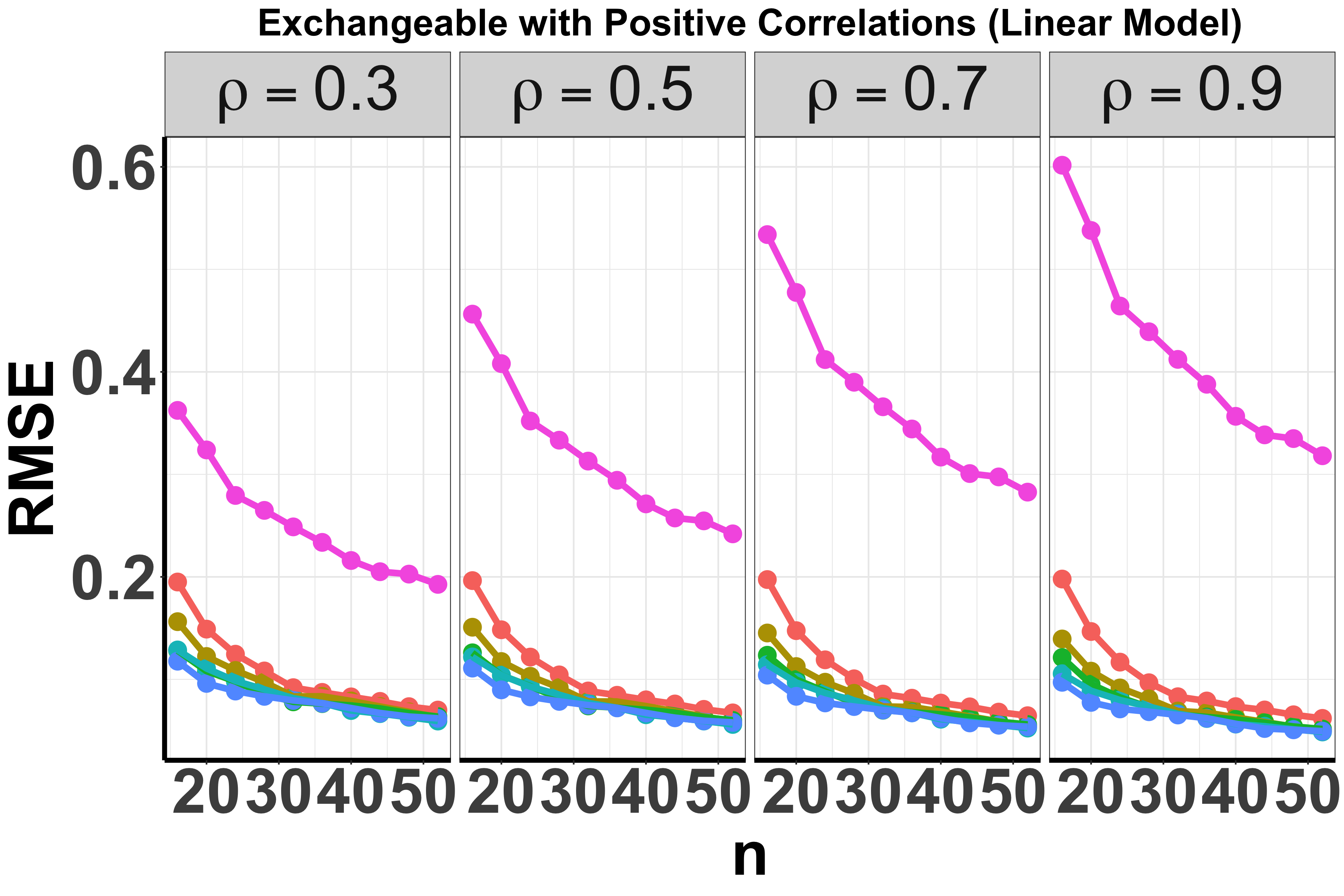}
    \end{minipage}
    
    \vskip0.3\baselineskip  

    \begin{minipage}{0.3\linewidth}
        \centering
        \includegraphics[width=\linewidth]{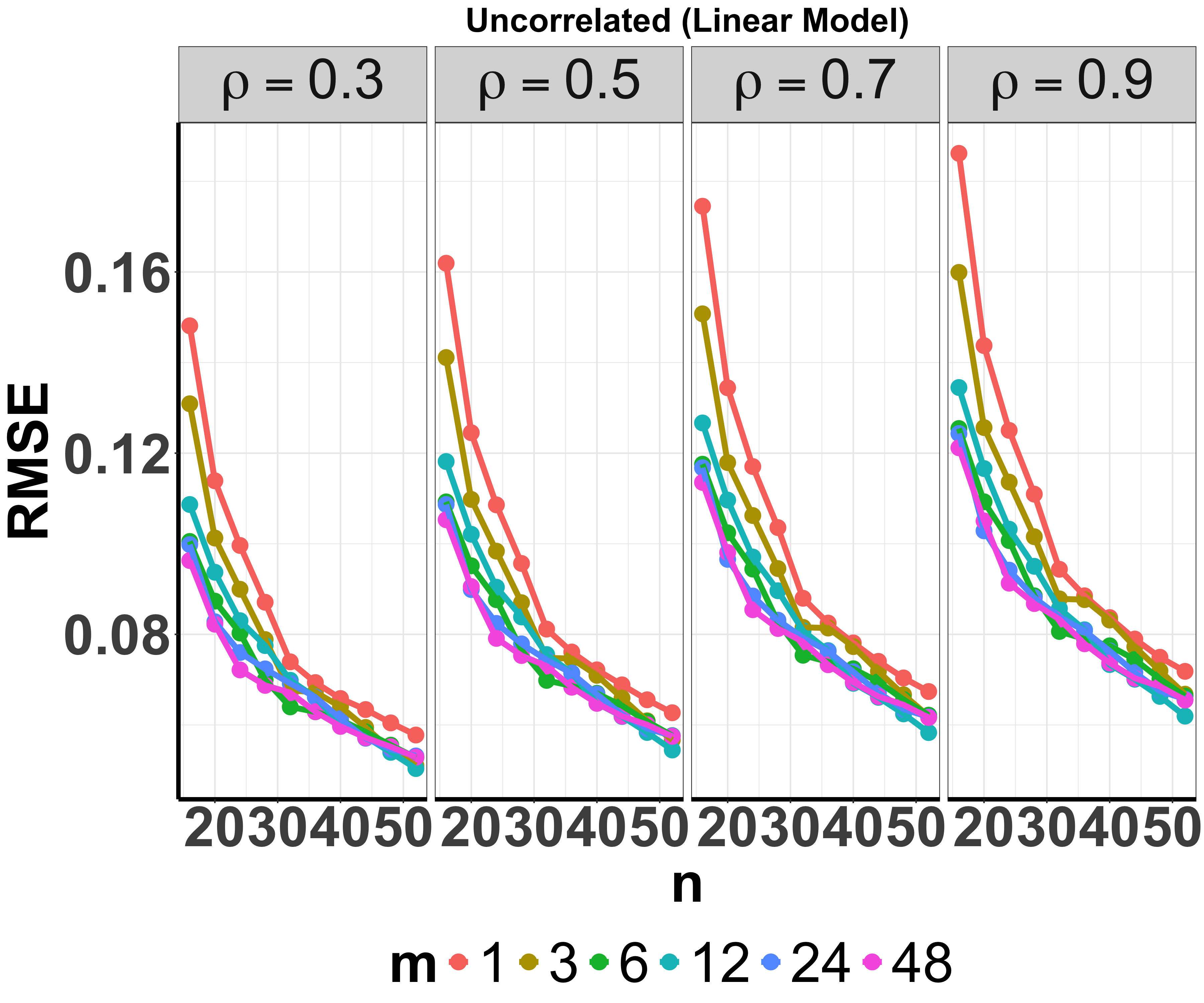}
    \end{minipage}
    \hspace{0.1cm}  
    \begin{minipage}{0.3\linewidth}
        \centering
        \includegraphics[width=\linewidth]{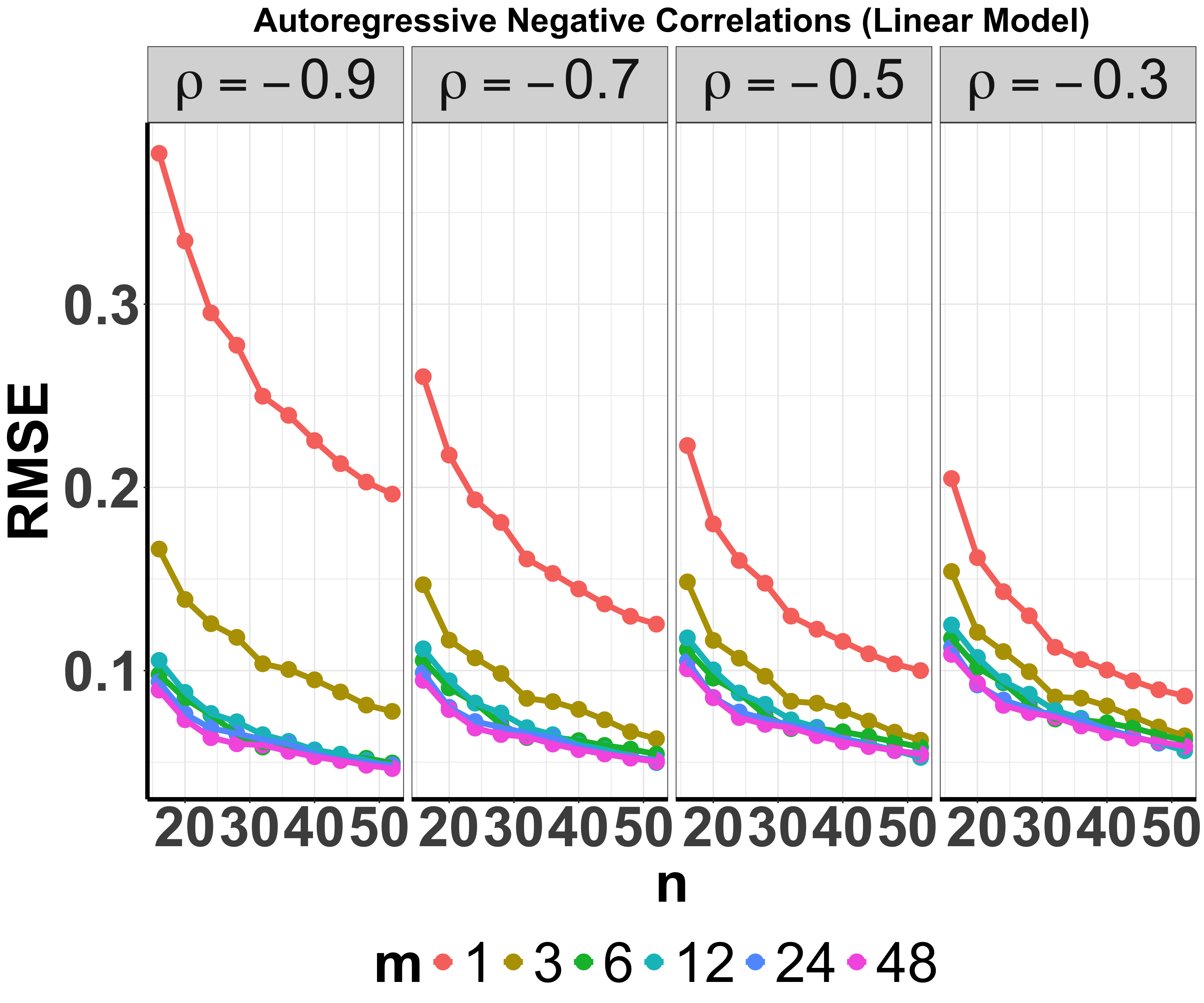}
    \end{minipage}
    
    \caption{{\small Numerical results for the linear DGP: RMSEs of OLS estimator with different combinations of $(n,m,\rho)$ and four covariance structures: moving average (top left), exchangeable with positive correlation (top right), uncorrelated (bottom left), and autoregressive with negative autocorrelation (bottom right).}}
    \label{fig::linear_sensitivity_covariance}
\end{figure}

\graphicspath{{figures_unit/corollary_figs/}}

\begin{figure}[ht]
    \centering
    \begin{minipage}{0.3\linewidth}
        \centering
        \includegraphics[width=\linewidth]{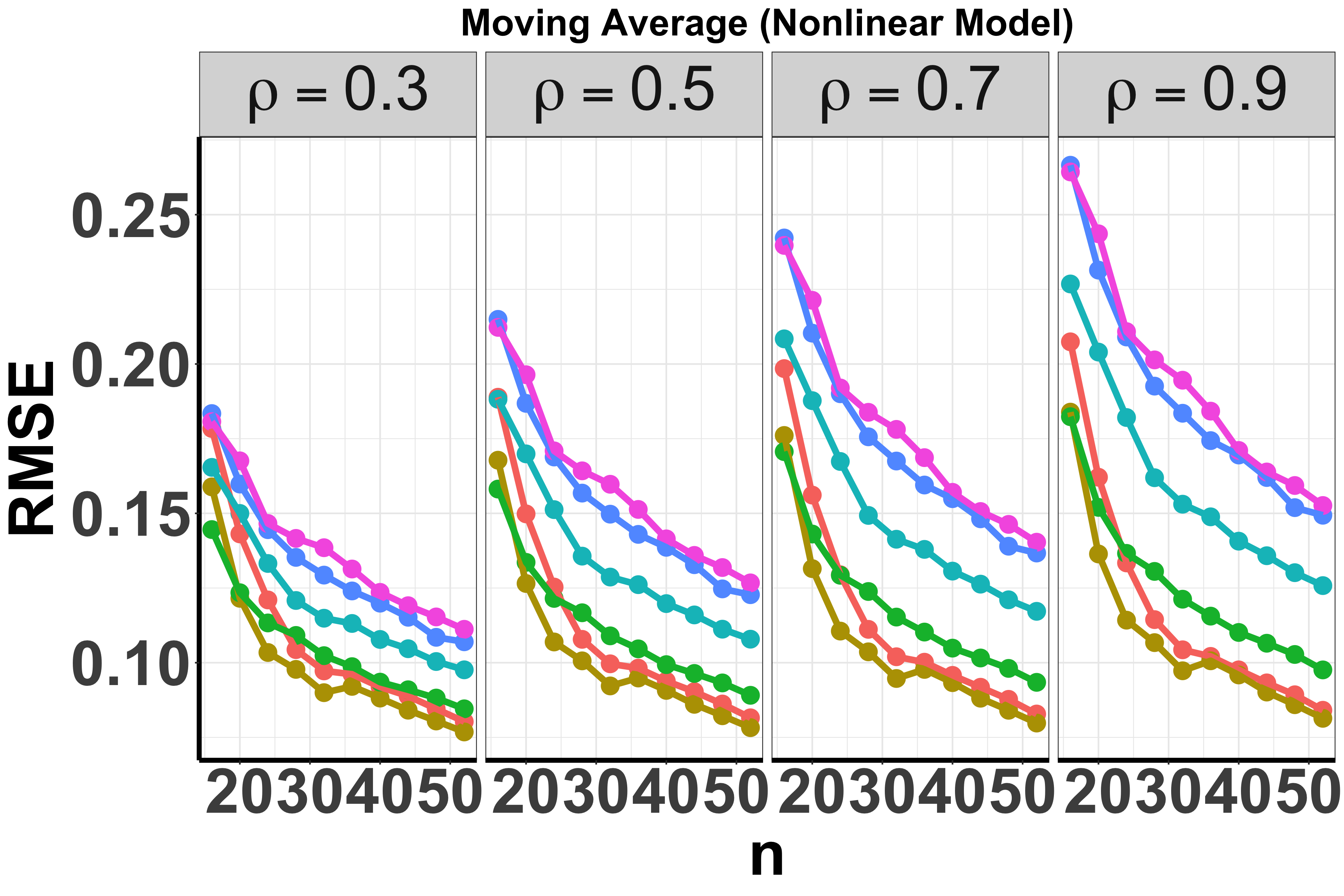}
    \end{minipage}
    \hspace{0.1cm}  
    \begin{minipage}{0.3\linewidth}
        \centering
        \includegraphics[width=\linewidth]{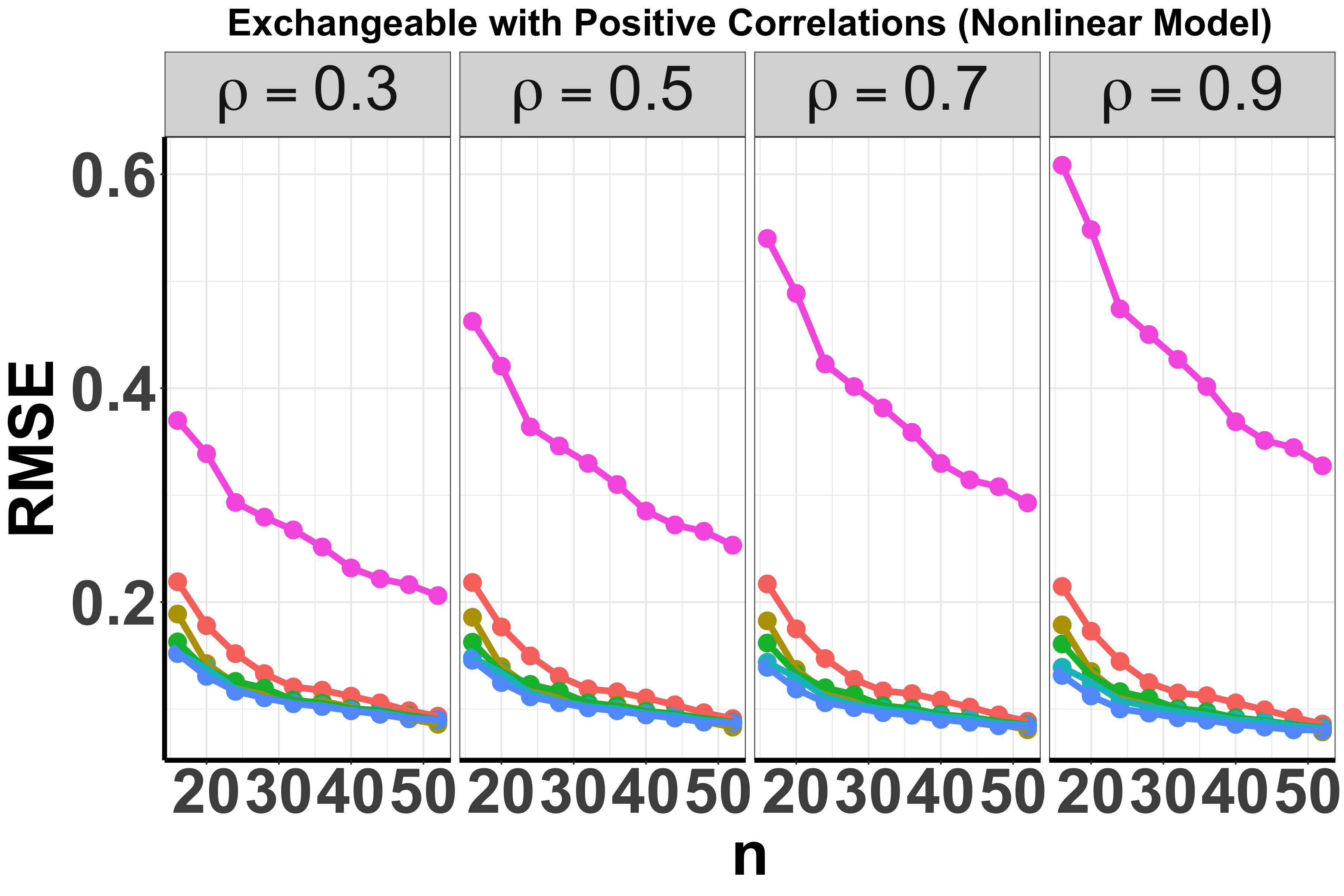}
    \end{minipage}
    
    \vskip0.3\baselineskip  

    \begin{minipage}{0.3\linewidth}
        \centering
        \includegraphics[width=\linewidth]{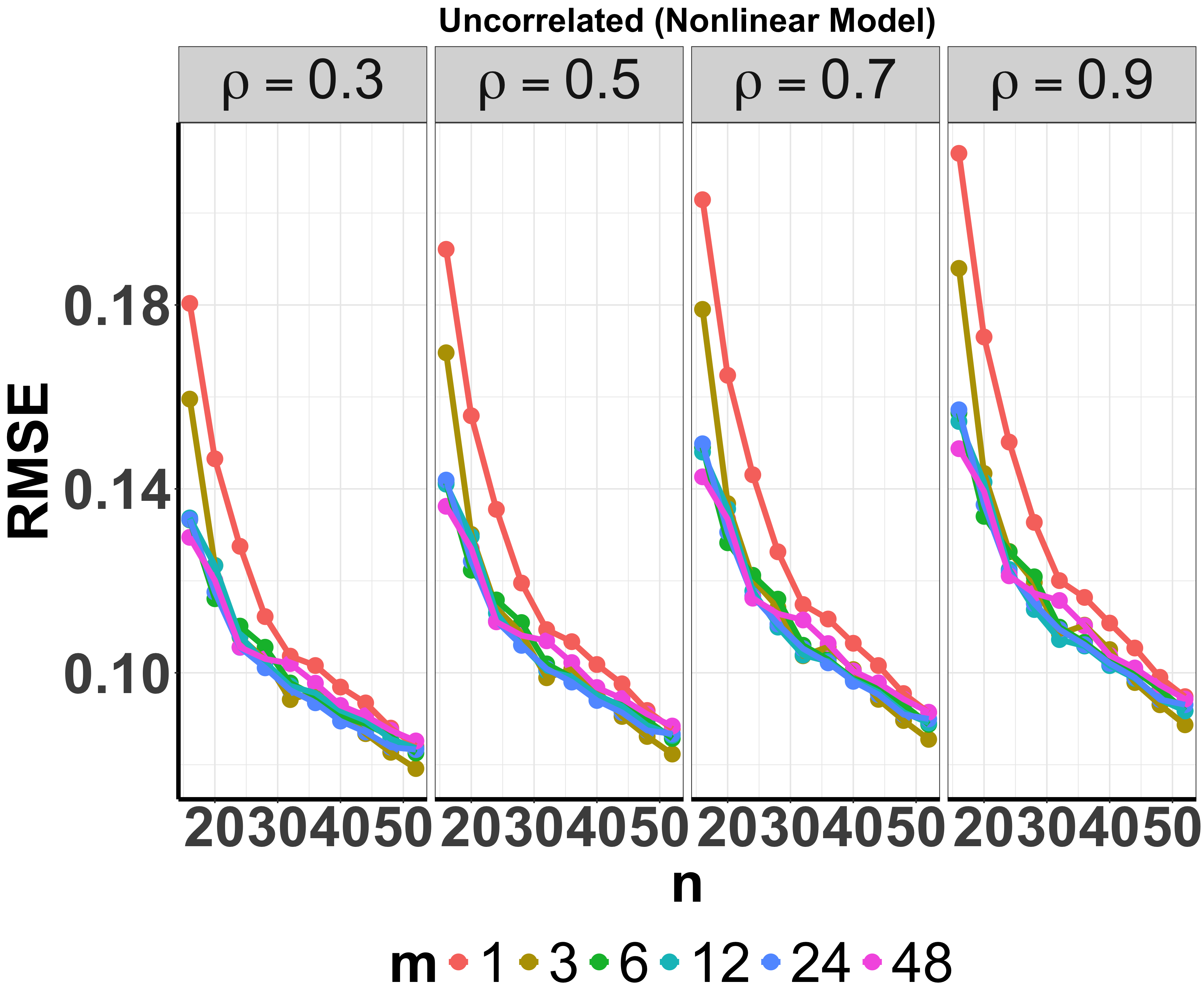}
    \end{minipage}
    \hspace{0.1cm}  
    \begin{minipage}{0.3\linewidth}
        \centering
        \includegraphics[width=\linewidth]{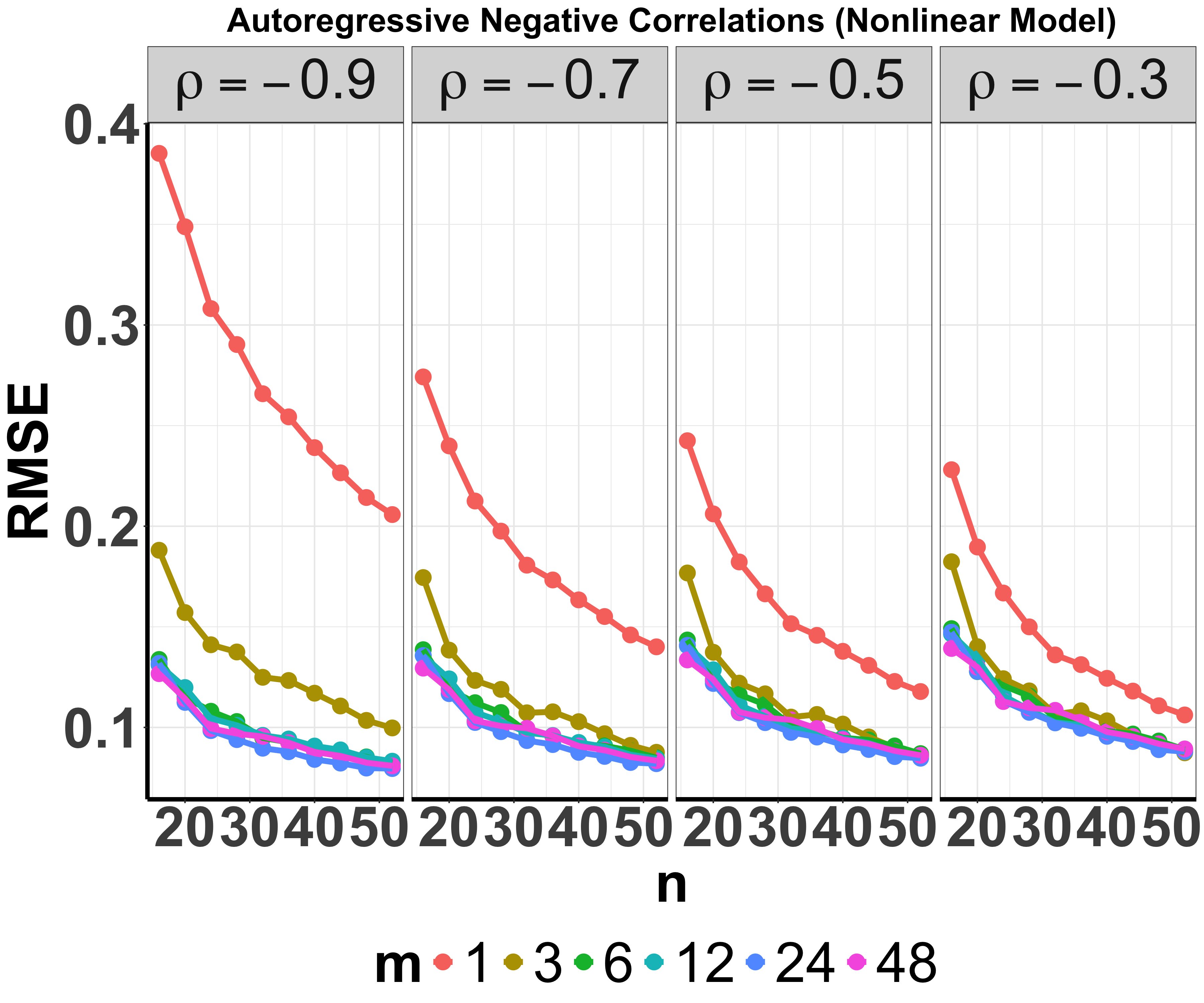}
    \end{minipage}
    
    \caption{{\small Numerical results for the nonlinear DGP: RMSEs of OLS estimator with different combinations of $(n,m,\rho)$ and four covariance structures: moving average (top left), exchangeable with positive correlation (top right), uncorrelated (bottom left), and autoregressive with negative autocorrelation (bottom right).}}
    \label{fig::nonlinear_sensitivity_covariance}
\end{figure}

\textbf{\textit{Comparison (Continued).}}
We compare our ATE estimators with three non-RL alternatives under the regular Bernoulli switchback design (see Definition~\ref{def_switch}). The first is proposed by \citet{bojinov2023design}. While it similarly relies on a hyperparameter $m$ that determines the duration of each policy application, their design differs in that, after applying a treatment for $m$ time steps, there is a 50\% chance of continuing with the same treatment or switching to the alternative. They employ multi-step importance sampling to construct their ATE estimator and show that the optimal choice of $m$ depends on the duration of the carryover effect, which is infinite in the MDP setting.

The second is the ``burn-in'' difference-in-mean estimator of \citet{hu2022switchback}, and the third is the simple importance sampling (IS) estimator of \citet{Xiong2023}. 

We begin by introducing the regular Bernoulli switchback design \citep{hu2022switchback}.

\begin{defs}\label{def_switch}
Given a time horizon $T$ and a block length $m \geq 1$ such that $K := T/m$ is a positive integer, the regular switchback design assigns treatment sequentially as follows:
\begin{equation*}
	A_{t} \mid A_{t-1} = 
	\begin{cases}
		\text{Bernoulli}(0.5) & \text{if } t = km+1 \text{ for some } k = 0, 1, \ldots, K-1, \\
		A_{t-1} & \text{otherwise}.
	\end{cases}
\end{equation*}
\end{defs}

We next present the three non-RL estimators under our notation. Under the $m$-carryover assumption, the following multi-step IS estimator is consistent for the ATE:
\begin{equation}\label{estimate:bojinov}
\begin{split}
\widehat{\ATE}^{(m)} =& \frac{1}{nT} \sum_{i=1}^n \Biggl\{ 
\sum_{t=m+1}^T \left[ 
\frac{R_{i,t} \mathbb{I}(A_{i,t-m:t} = \1_{m+1})}{\prob(A_{i,t-m:t} = \1_{m+1})} 
- \frac{R_{i,t} \mathbb{I}(A_{i,t-m:t} = \0_{m+1})}{\prob(A_{i,t-m:t} = \0_{m+1})} 
\right] \\
&+ \sum_{t=1}^m \left[ 
\frac{R_{i,t} \mathbb{I}(A_{i,1:t} = \1_t)}{\prob(A_{i,1:t} = \1_t)} 
- \frac{R_{i,t} \mathbb{I}(A_{i,1:t} = \0_t)}{\prob(A_{i,1:t} = \0_t)} 
\right] 
\Biggr\},
\end{split}
\end{equation}

where $A_{t_1:t_2} = (A_{t_1}, A_{t_1+1}, \ldots, A_{t_2})^\top$ for $1 \leq t_1 < t_2 \leq T$. As shown in \citet{bojinov2023design}, the optimal block length equals $m$ under this design.

We also include the following two estimators: the ``burn-in'' difference-in-mean estimator (with burn-in length $0 \leq b < m$) and the simple IS estimator:
\begin{equation*}
\begin{split}
\widehat{\ATE}^{m,b} &= \frac{1}{n} \sum_{i=1}^n 
\left[ 
\frac{1}{K_1 (m-b)} \sum_{k: A_{km+1} = 1} \sum_{s=b}^{m} R_{i, km+s}
- \frac{1}{K_0 (m-b)} \sum_{k: A_{km+1} = 0} \sum_{s=b}^{m} R_{i, km+s}
\right], \\
\widehat{\ATE}^{m,\text{IS}} &= \frac{4}{nT} \sum_{i=1}^n \sum_{t=1}^T \left(A_{it} - \frac{1}{2}\right) R_{it},
\end{split}
\end{equation*}
where $K_a = \left| \{ k \in \{0, \ldots, K-1\} : A_{km+1} = a \} \right|$ for $a \in \{0,1\}$.

We adopt the same experimental setup as in the linear/nonlinear DGP setting (see Subsection~\ref{subsec:linearDGP}), using a burn-in length $b = 1$, and report the log(MSE) values of various ATE estimators under different designs and choices of $m$ in Figure~\ref{fig:compare_logmse_linear_nonlinear}. The results clearly demonstrate the superior performance of our ATE estimators over those based on the alternative designs.

\textbf{\textit{Real Data-based Simulation (Continued).}} The first dataset covers the period from Dec. 5th, 2018, to Jan. 13th, 2019, with thirty minutes defined as one time unit, resulting in $T=48$. The second dataset spans from May 17th, 2019, to June 25th, 2019, with one-hour time units, leading to $T=24$. A summary of the bootstrap-assisted procedure is provided in Algorithm \ref{algo:res_bootstrap}. Specifically, for each dataset, we first fit the data based on the linear models in \eqref{linear_mdp}. We apply ridge regression to estimate the regression coefficients with the regularization parameter determined by minimizing the generalized cross-validation criterion \citep{wahba1975smoothing}.  
This yields the estimators $\{\widehat{\alpha}_t\}_t$, $\{\widehat{\beta}_t\}_t$, $\{\widehat{\phi}_t\}_t$ and $\{\widehat{\Phi}_t\}_t$. However, $\{\gamma_t\}_t$ and $\{\Gamma_t\}_t$ remain unidentifiable, since $A_t=0$ almost surely. We then calculate the residuals in the reward and state regression models based on these estimators as follows:  
\begin{equation}\label{eqn:residuals}
	\widehat{e}_{i,t}=R_{i,t}-\widehat{\alpha}_t-S_{i,t}^\top \widehat{\beta}_t, \quad \widehat{E}_{i,t}=S_{i,t+1}-\widehat{\phi}_{t}- \widehat{\Phi}_t S_{i,t}.
\end{equation}
To generate simulation data with varying sizes of treatment effect, we introduce the treatment effect ratio parameter $\lambda$ and manually set the treatment effect parameters {$\widehat{\gamma}_t=\delta_1\times (\sum_i R_{i,t}/(100\times N))$ and $\widehat{\Gamma}_t=\delta_2\times (\sum_i S_{i,t}/(100\times N))$}. The treatment effect ratio essentially corresponds to the ratio of the ATE and the baseline policy's average return. A discussion on the choice of the treatment effect ratio can be found in subsection \ref{sec:realdata}.

Finally, to create a dataset spanning $n$ days, actions are generated according to the chosen design as described in Section \ref{sec:model}. We then sample i.i.d. standard Gaussian noises $\{\xi_i\}_{i=1}^n$. For the $i$-th day, we uniformly sample an integer $I\in \{1,\dots, N\}$, set the initial state to $S_{I,1}$, and generate rewards and states according to \eqref{linear_mdp} with the estimated $\{\widehat{\alpha}_t\}_t$, $\{\widehat{\beta}_t\}_t$, $\{\widehat{\phi}_t\}_t$, $\{\widehat{\Phi}_t\}_t$, the specified $\{\widehat{\gamma}_t\}_t$ and $\{\widehat{\Gamma}_t\}_t$, and the error residuals given by $\{\xi_i \widehat{e}_{i,t}:1\le t\le T\}$ and $\{\xi_i \widehat{E}_{i,t}:1\le t\le T\}$, respectively. This ensures that the error covariance structure of the simulated data closely resembles that of the real datasets. Based on the simulated data, we similarly apply OLS, LSTD and DRL to estimate the ATE and compare their MSEs under various switchback designs. 

\begin{algorithm}[t]
	\caption{Bootstrap-based Simulation}\label{algo:res_bootstrap}
	\begin{algorithmic}[1]
		\STATE {\bfseries Input:} Real data $\left\lbrace (S_{it}, R_{it}): 1 \leq i \leq n; 1 \leq t \leq T \right\rbrace$, the adjustment parameters for the ratios $(\delta_1, \delta_2)$, $m$-switchback design, the bootstraped sample size $n$, random seed, the bootstrapped size $B$.
		\STATE {\bfseries Output:} The RSMEs, Biases, and SDs of different ATE estimators.
		\STATE {\bfseries Step 1:} Calculate the least square estimates $\left\lbrace \widehat{\alpha}_t \right\rbrace$, $\left\lbrace \widehat{\beta}_t \right\rbrace$, $\left\lbrace \widehat{\phi}_t \right\rbrace$, $\left\lbrace \widehat{\Phi}_t \right\rbrace$ in the model (\ref{linear_mdp}), treatment effect parameters $\left\lbrace \widehat{\gamma}_t \right\rbrace$, $\left\lbrace \widehat{\Gamma}_t \right\rbrace$, and the residuals of the reward model and state regression model by Equation \eqref{eqn:residuals}.
		\FOR{$b = 1$ {\bfseries to} $B$}
		\STATE Sample the number of days $n$ from $\left\lbrace 1, \ldots ,N \right\rbrace$ with replacement and generate i.i.d. normal random variables $ \xi_{i}^b \sim N(0,1)$.
		\STATE Generate pseudo rewards $\left\lbrace \widehat{R}_{i,t}^b \right\rbrace_{i,t}$ and states $\left\lbrace \widehat{S}_{i,t}^b \right\rbrace_{i,t}$ using:
		\begin{eqnarray*}
			\widehat{R}_{i,t}^b =  [1, (\widehat{S}_{i,t}^{b})^\top, A_{i,t} ] \begin{pmatrix}
				\widehat{\alpha}_t \\
				\widehat{\beta}_t \\
				\widehat{\gamma}_t
			\end{pmatrix}  + \xi_{i}^b \widehat{e}_{i,t}, \quad 
			\widehat{S}_{i,t+1}^b  =  [\widehat{\phi}_t, \widehat{\Phi}_t, \widehat{\Gamma}_t ] \begin{pmatrix}
				1 \\
				\widehat{S}_{i,t}^{b} \\
				A_{i,t}
			\end{pmatrix} + \xi_{i}^b \widehat{E}_{i,t}.
		\end{eqnarray*}
		\STATE Calculate the set of estimators $\left\lbrace \textrm{ATE}_{\text{SB}}^{(m),b} \right\rbrace_{b,m}$ by OLS, LSTD, and DRL.
		\ENDFOR
	\end{algorithmic}
\end{algorithm}

\textit{\textbf{Confidence intervals}}. Since we adopt RL-based estimators for A/B testing, existing methods developed in the reinforcement learning literature can be directly applied for CI construction \citep[see e.g.,][Section 5]{shi2025statistical}. In real-data-based simulation, we use nonparametric bootstrap to construct CIs for OLS-based ATE estimators, and report both the coverage probability (CP) and the average CI width of these CIs in Figure \ref{fig:CI_mean_widths}. It can be seen that most CPs are over 92\%, close to the nominal level. Meanwhile,
{
\begin{itemize}
\setlength\itemsep{0pt}
 \setlength\topsep{0pt}
    \item For small values of $\lambda$, more frequent policy switch reduces the average CI width.
    \item When $\lambda$ is increased to 10\%, AD produces the narrowest CI on average.
\end{itemize}}
These results verify our claim that a reduction in MSE directly translates to a shorter CI.

\begin{figure}[t]
    \centering
    \begin{minipage}{0.45\linewidth}
        \centering
\includegraphics[width=\linewidth]{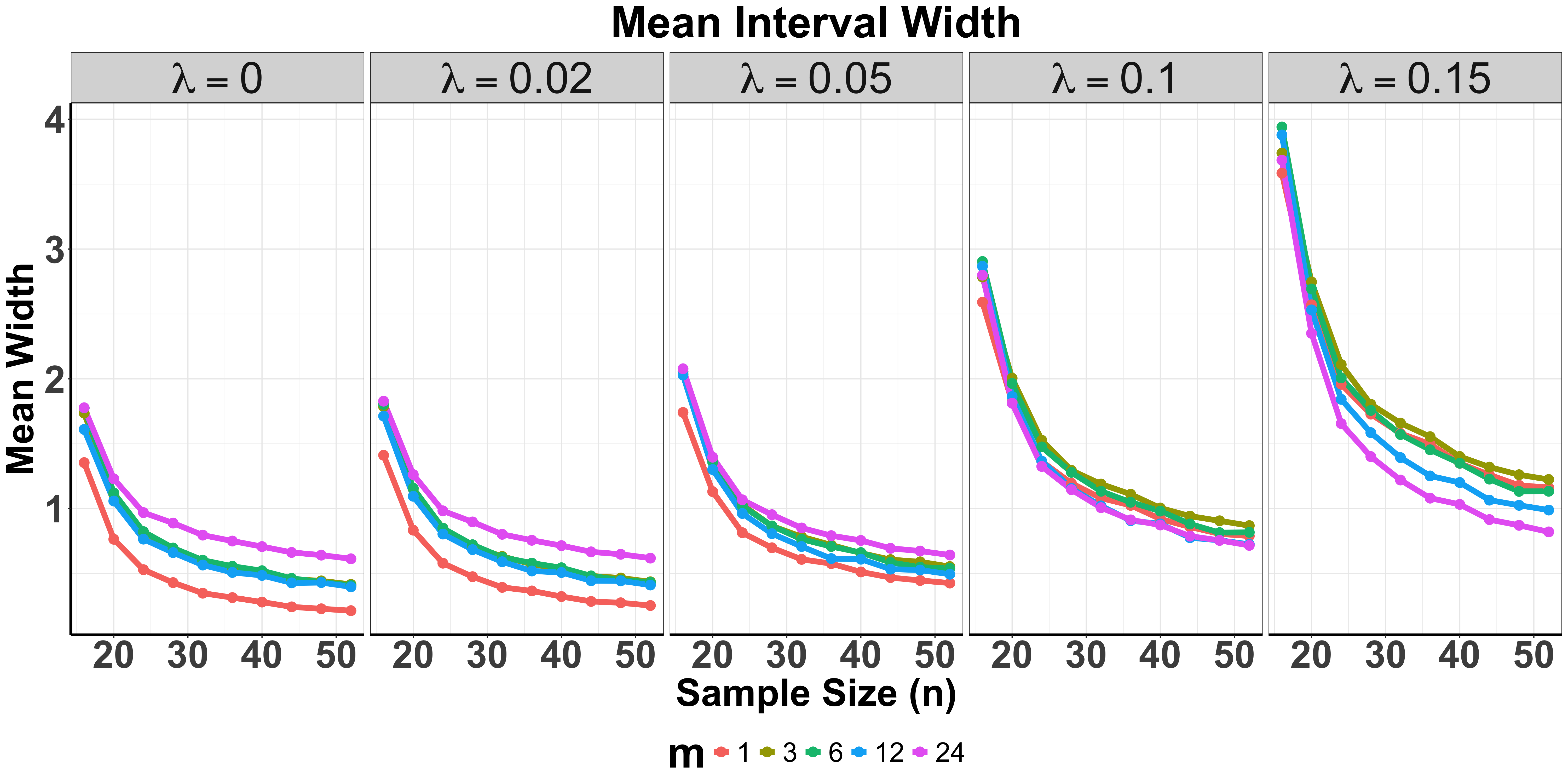}
    \end{minipage}
    \hspace{0.05cm}  
    \begin{minipage}{0.45\linewidth}
        \centering
        \includegraphics[width=\linewidth]{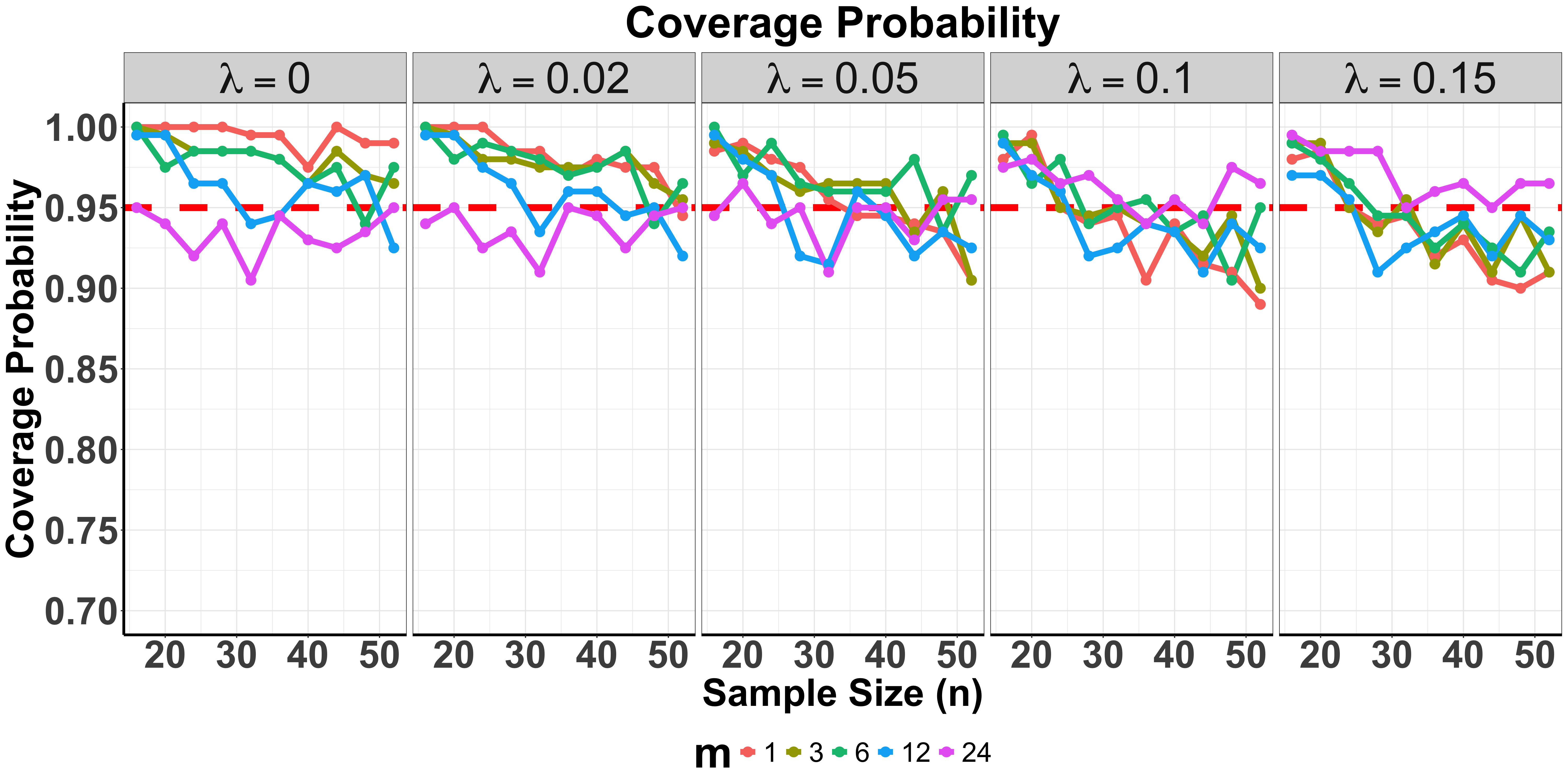}
    \end{minipage}
    \caption{Confidence intervals  in real-data-based simulation.}
    \label{fig:CI_mean_widths}
\end{figure}

\begin{figure}[t]
    \centering
    \begin{minipage}{0.45\linewidth}
        \centering
        \includegraphics[width=\linewidth]{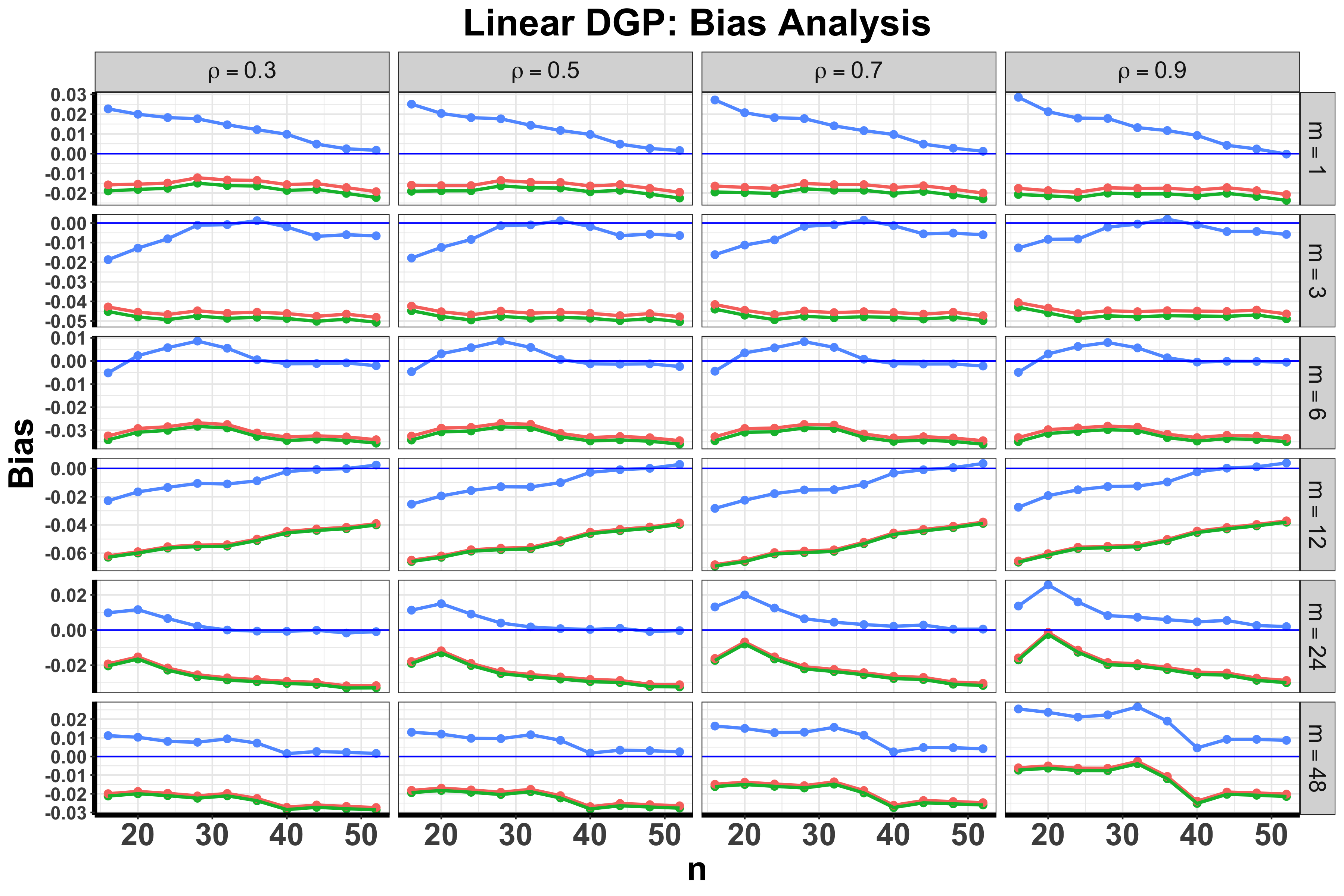}
    \end{minipage}
    \hspace{0.05\linewidth}
    \begin{minipage}{0.45\linewidth}
        \centering
        \includegraphics[width=\linewidth]{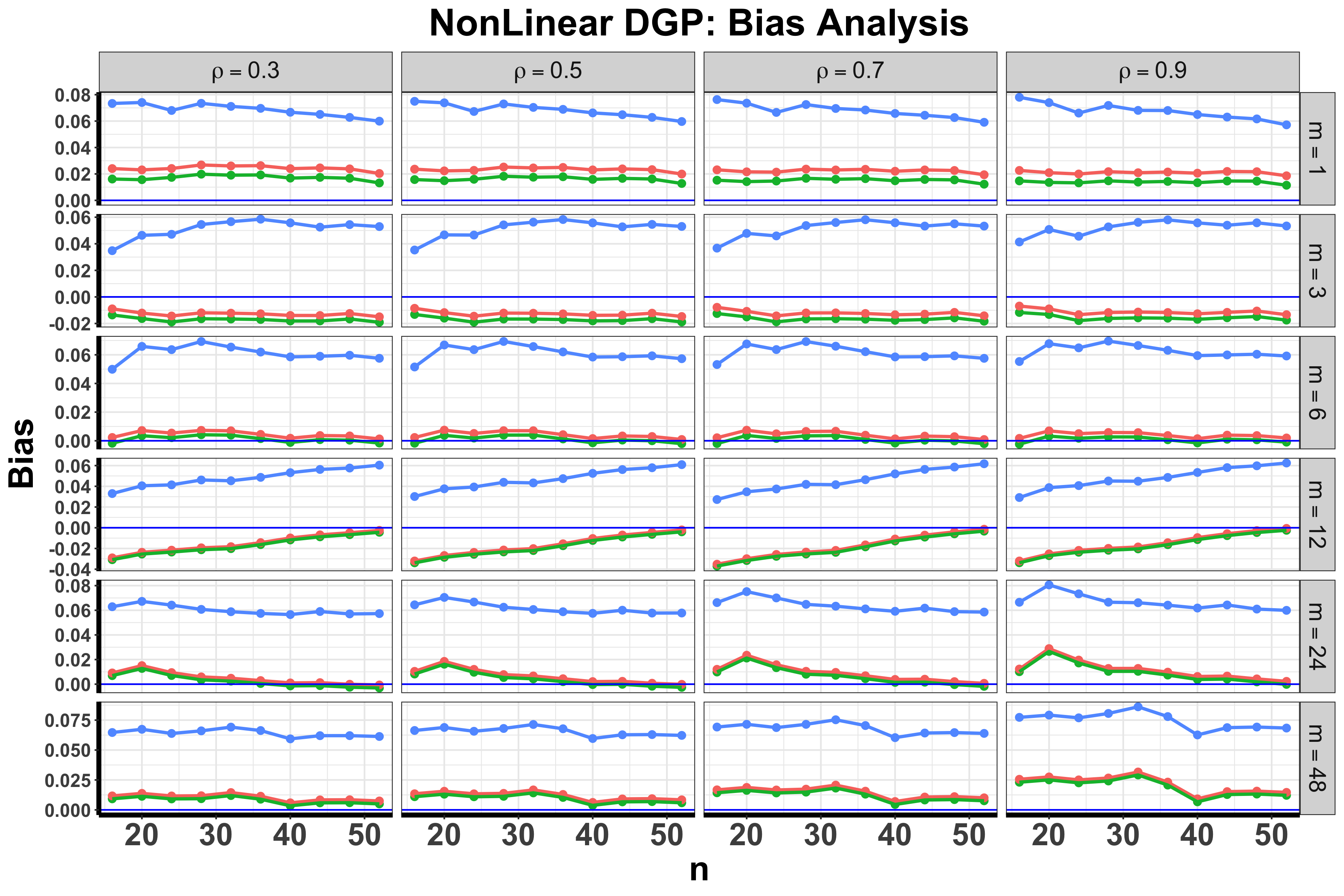}
    \end{minipage}

    \vspace{1em}  

    \begin{minipage}{0.45\linewidth}
        \centering
        \includegraphics[width=\linewidth]{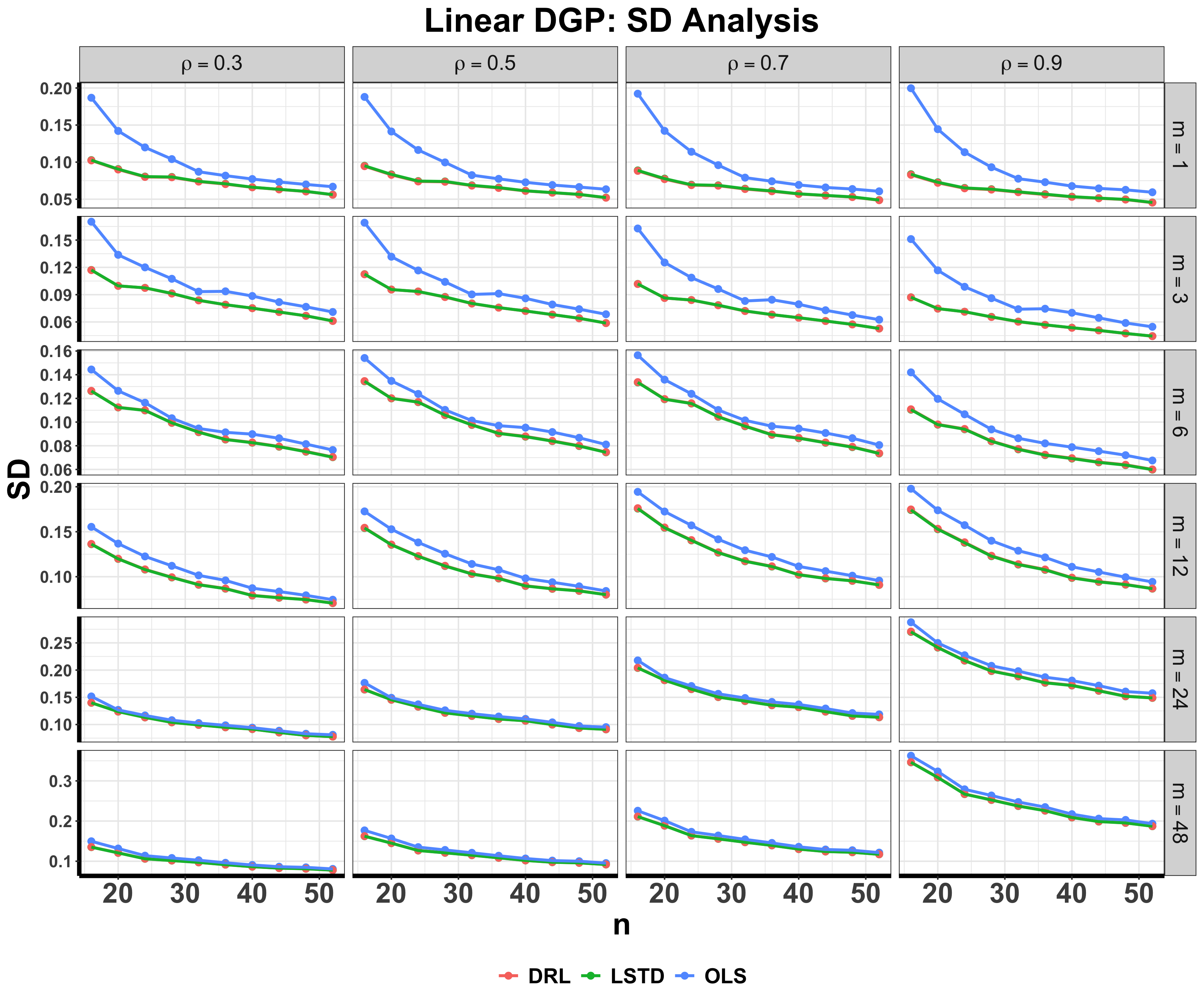}
    \end{minipage}
    \hspace{0.05\linewidth}
    \begin{minipage}{0.45\linewidth}
        \centering
        \includegraphics[width=\linewidth]{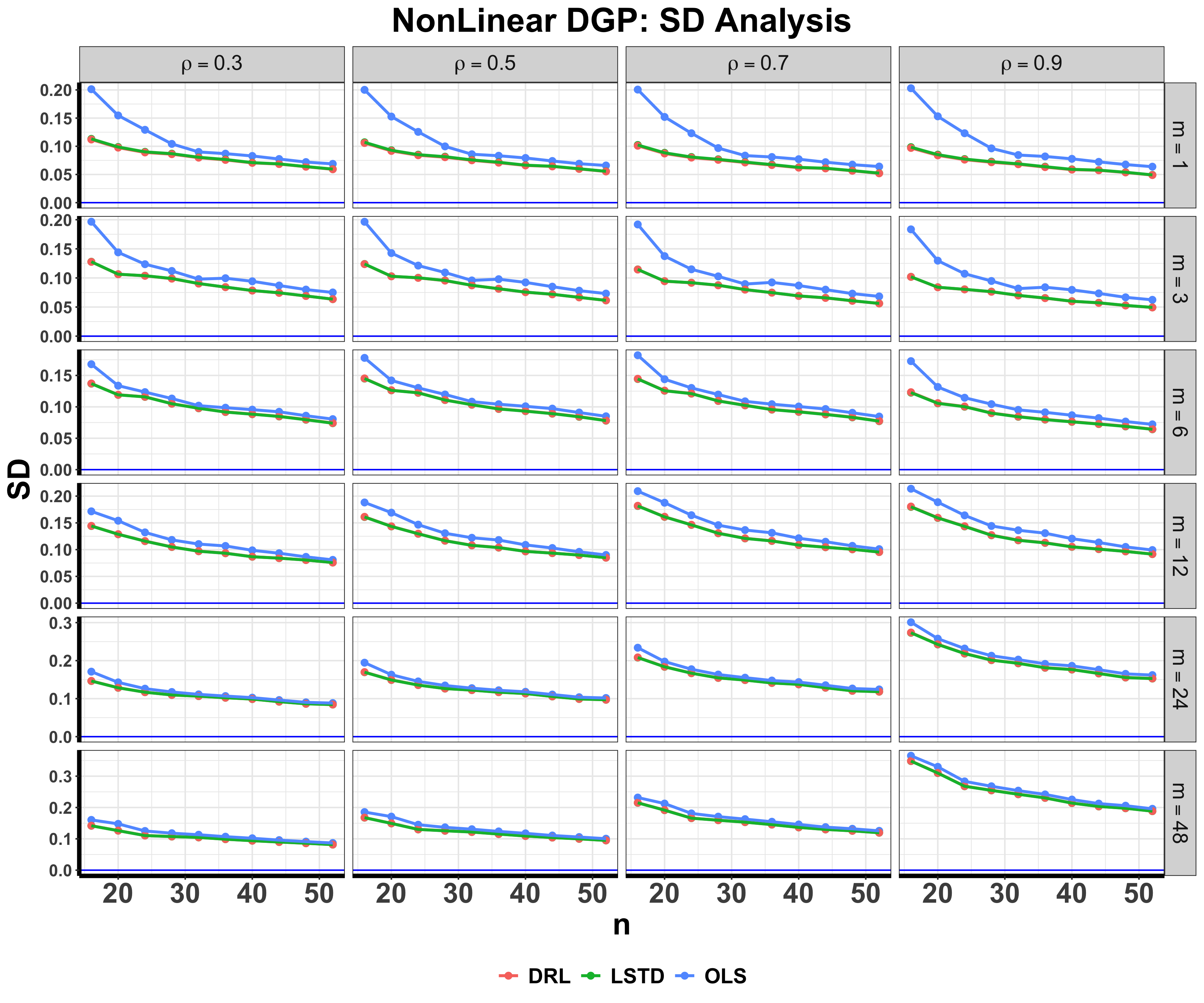}
    \end{minipage}

    \caption{\small
     Comparison of bias (top row) and standard deviation (bottom row) of different ATE estimators under varying combinations of $(n, m, \rho)$, for linear (left column) and nonlinear (right column) data-generating processes.
    }
    \label{fig:DGP_combined_sd_bias}
\end{figure}

\begin{figure}[t]
    \centering
    \begin{minipage}{0.45\linewidth}
        \centering
        \includegraphics[width=\linewidth]{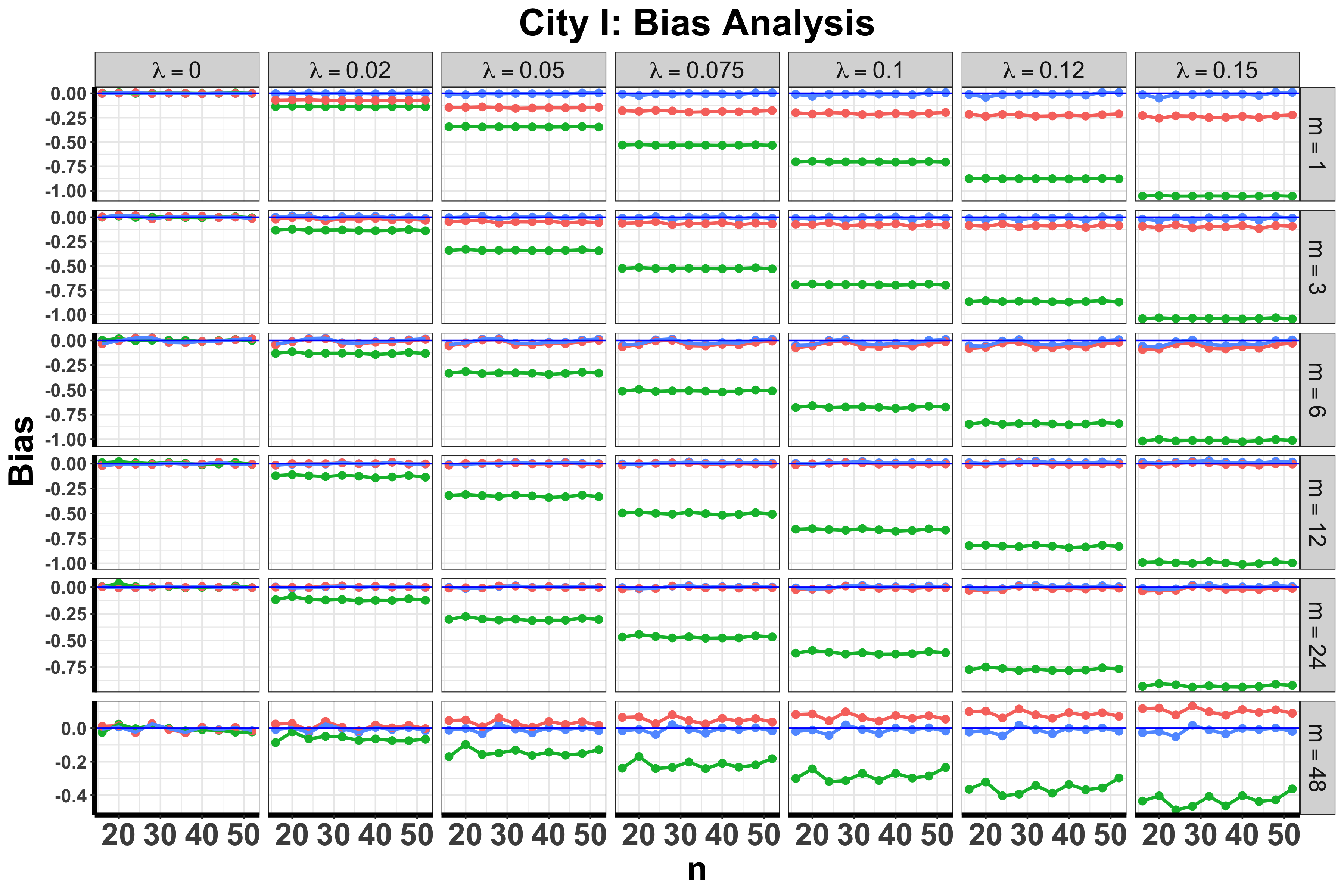}
    \end{minipage}
    \hspace{0.1cm}  
    \begin{minipage}{0.45\linewidth}
        \centering
        \includegraphics[width=\linewidth]{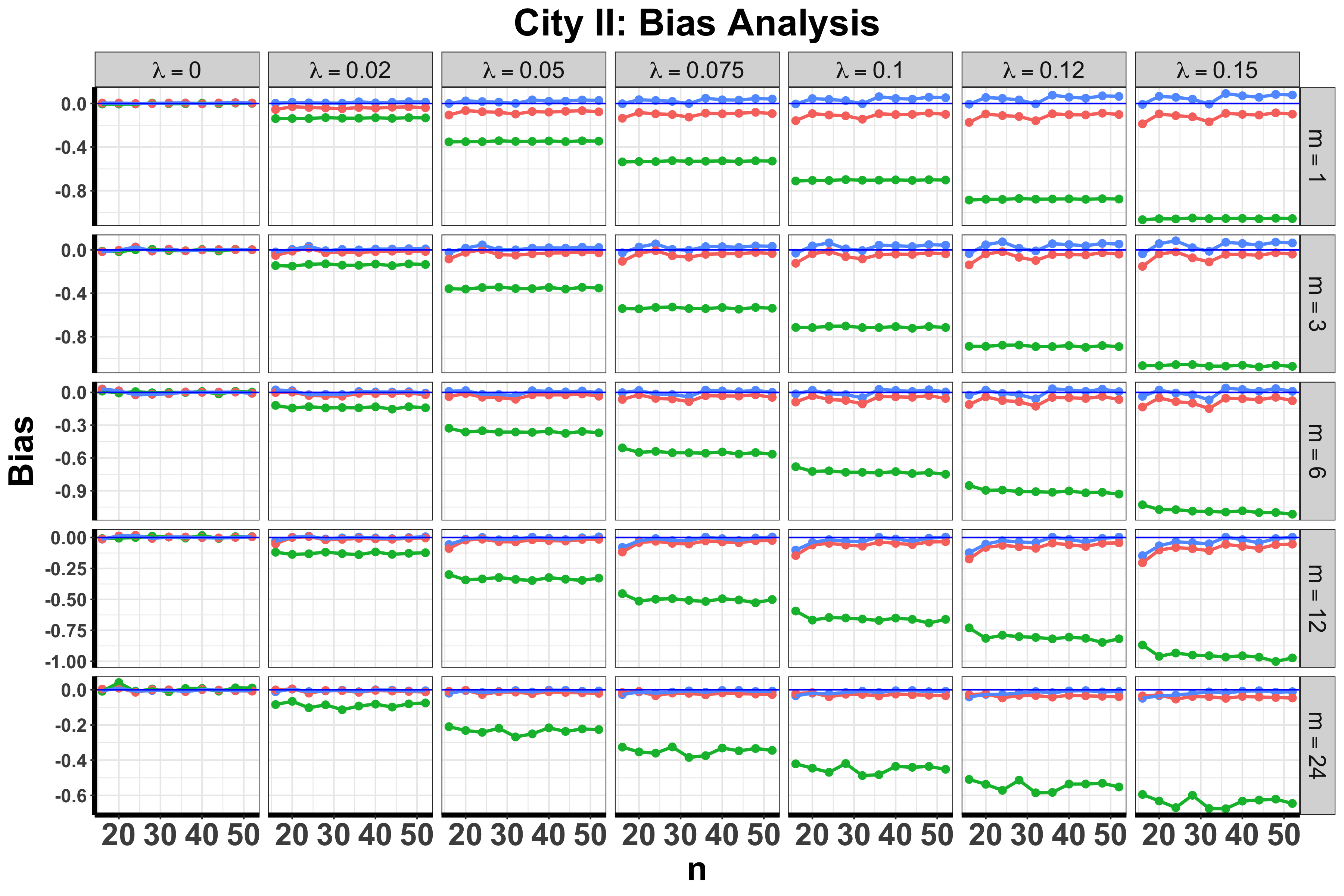}
    \end{minipage}
    \vskip \baselineskip  
    \begin{minipage}{0.45\linewidth}
        \centering
        \includegraphics[width=\linewidth]{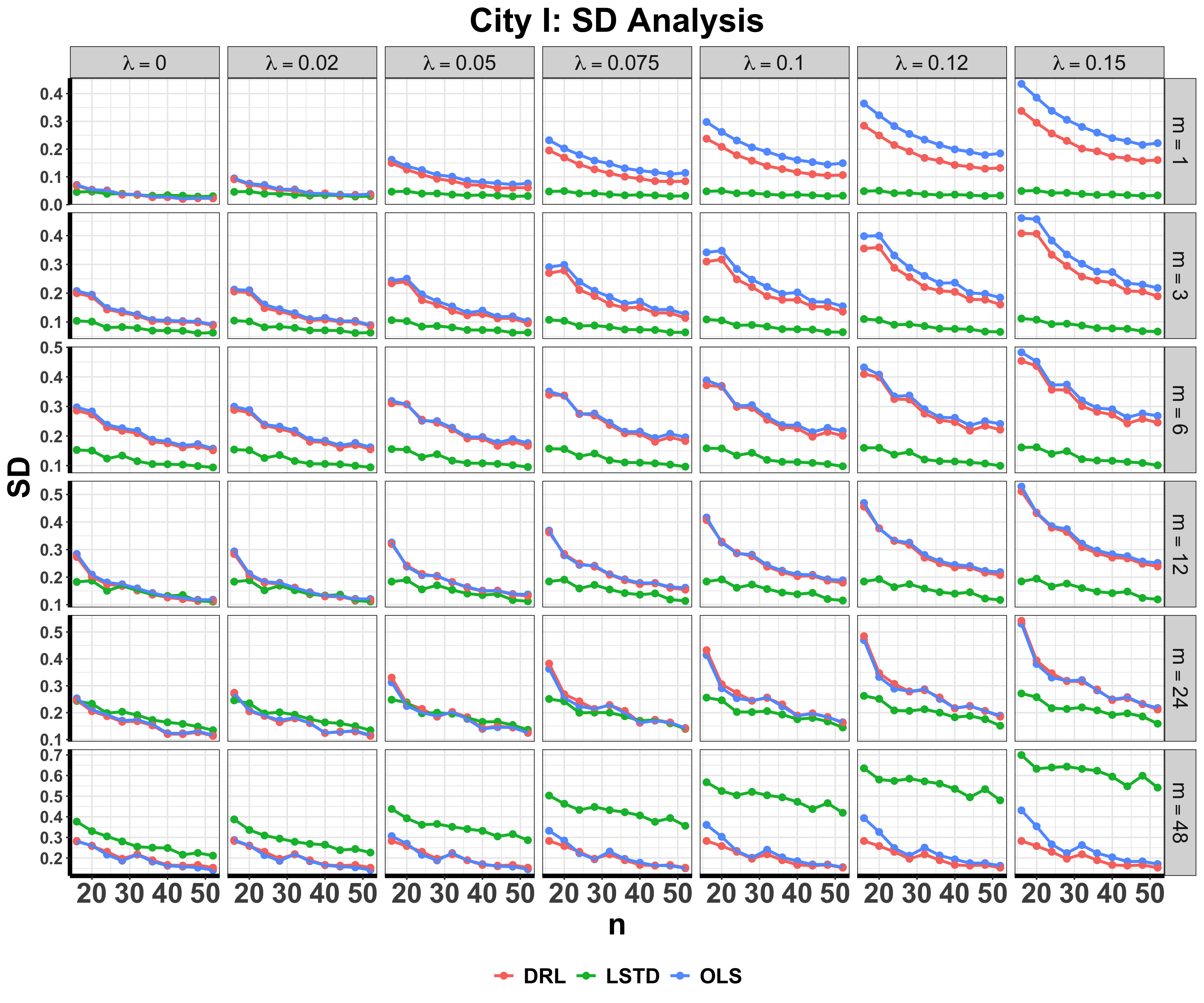}
    \end{minipage}
    \hspace{0.1cm}  
    \begin{minipage}{0.45\linewidth}
        \centering
        \includegraphics[width=\linewidth]{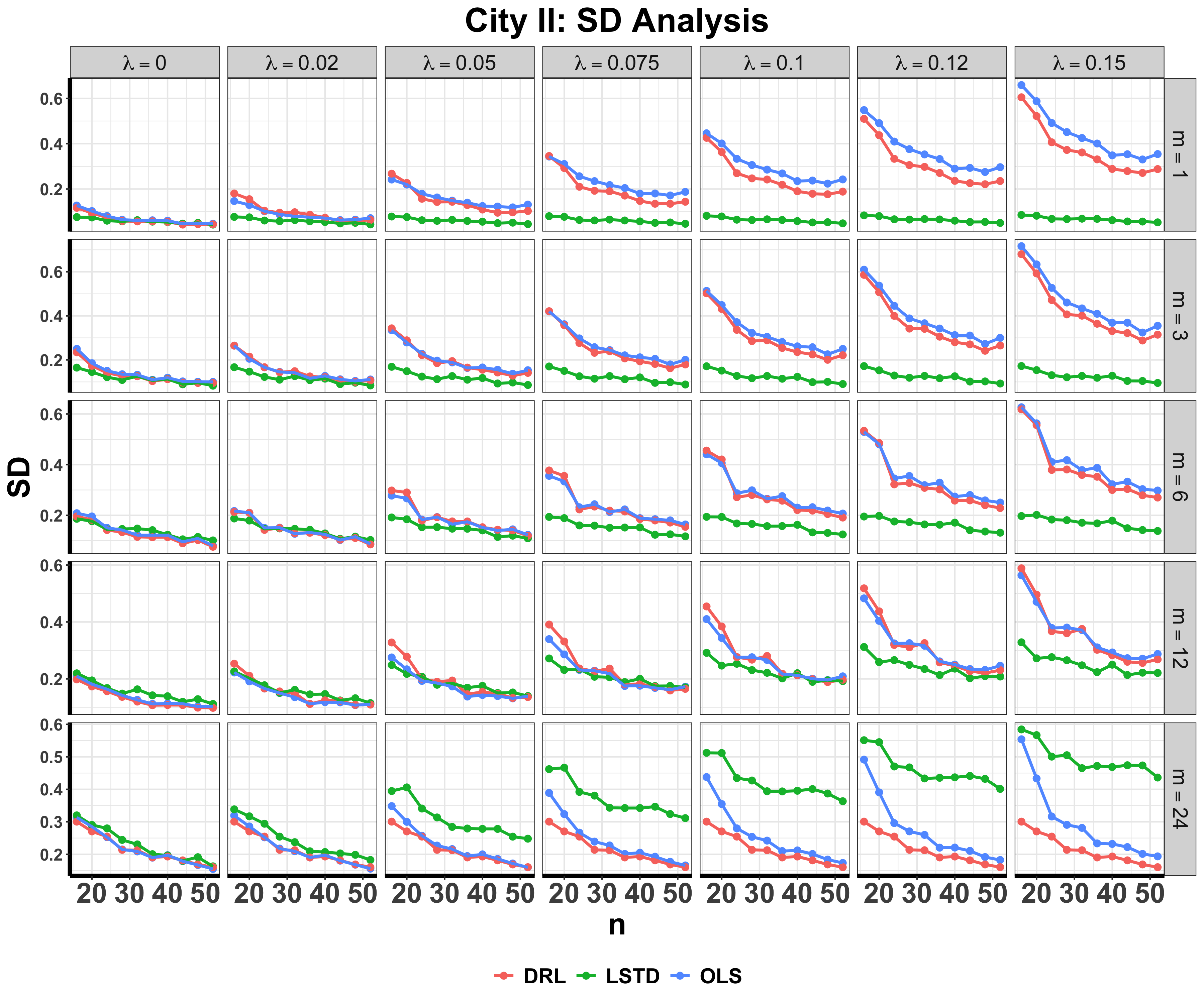}
    \end{minipage}
    
    \caption{{\small Comparison of bias (top row) and standard deviation (bottom row) of different ATE estimators under varying combinations of $(n, m, \lambda)$, based on real datasets from City I (left column) and City II (right column)}}
\label{fig:rel_data_compare_SD_bias}
\end{figure}

\textit{\textbf{Other results.}} We further visualize the biases and standard deviations of the ATE estimators in synthetic experiments in Figure \ref{fig:DGP_combined_sd_bias}. It can be seen that under the nonlinear DGP, the OLS-based ATE estimators exhibit larger absolute bias than the LSTD- and DRL-based estimators, primarily due to the misspecification of the linear model. 

Similarly, Figure \ref{fig:rel_data_compare_SD_bias} displays the biases and standard deviations of the ATE estimators across two real datasets. The standard deviations of all estimators increase, as expected, as the effect size $\lambda$ or $m$ increases. The biases of the OLS and DRL estimators remain relativ
ely stable with respect to $m$ in the setting of weak carryover effects. However, the LSTD estimator experiences a large bias with a larger effect size. Since these biases are caused by model misspecification, they remain roughly constant across different designs.


\section{Assumptions, Proofs of Theories and Corollaries}\label{sec:proof}
In this section, we first present Assumptions \ref{asmp:est} -- \ref{asmp:sievebasis} that are needed to establish Theorem \ref{thm::main}. Next, we discuss in detail the order of magnitude of the reminder term in Theorem \ref{thm::main}. Finally, we provide the proofs of our theorem and corollaries. Throughout this section, we assume without loss of generality that the state space $\mathcal{S}$ is discrete -- a typical assumption in RL  \citep[see e.g.,][]{sutton2018reinforcement}.

\subsection{Assumptions}\label{subsec:assump}
\begin{asmp}[Bounded ATE]\label{asmp:est}
   The absolute value of the OLS-based ATE estimator is bounded by $R_{\max}$.  
\end{asmp}
\begin{asmp}[Non-singular covariance matrix]\label{asmp:nonsingular}
For any $1\le t\le T$, the covariance matrix $\Cov(S_t)$ is non-singular, whose minimum eigenvalue is larger than $\max(\epsilon, \bar{c}\delta)$ for some fixed constant $\epsilon>0$ and some sufficiently large constant $\bar{c}>0$.
\end{asmp}
\begin{asmp}[Bounded regression coefficients]\label{asmp:boundedregcoef}
    For any $1\le t\le T$, $\|\Phi_t\|_2\le \rho_{\Phi}$ for some constant $0<\rho_{\Phi}<1$.
\end{asmp}
\begin{asmp}[Bounded states]\label{asmp:boundstates}
The state dimension $d$ is fixed, and the states are contained within a compact ball, i.e., there exists a constant \( C > 0 \) such that for all \( t \), the states satisfy \( \| S_t \|_2 \leq C \).
\end{asmp}
\begin{asmp}[Bounded transition functions]\label{asmp:boundtrans}
	The transition functions $\{p_t\}$ are uniformly bounded away from zero and infinity.
\end{asmp}
\begin{asmp}[Bounded temporal difference errors]\label{asmp:boundedtd}
    The absolute value of the temporal difference error $R_t+V_{t+1}(S_{t+1})-V_t(S_t)$ is of the order $O(R_{\max})$ where the big-$O$ term is uniform in $t$.
\end{asmp}
\begin{asmp}[Sieve basis functions]\label{asmp:sievebasis}
	(i) For any $t \le T$, there exists a constant $c \ge 1$ such that $	c^{-1} \le \lambda_{\min}\Big[\sum_{s\in \mathcal{S}}\varphi_t(s)\varphi_t^\top(s)\Big] 
	\le \lambda_{\max}\Big[\sum_{s\in \mathcal{S}}\varphi_t(s)\varphi_t^\top(s) \Big] \le c,$
	where $\lambda_{\min}[\cdot]$ and $\lambda_{\max}[\cdot]$ denote the minimum and maximum eigenvalues of a matrix, respectively;
	(ii) $\sup_{s \in \mathcal{S}} \|\varphi_t(s)\|_2 <\infty $. 
\end{asmp}
\begin{asmp}[Convergence rates]\label{asmp:converge}
    $\textrm{err}_\omega^2 = \max_{t,a,m} \Mean | 	\widehat{\omega}_{t,-k}^{a,m} (s)-\omega_t^{a,m}(s) |^2$ and  $\textrm{err}_v^2=R_{\max}^{-2}\max_{t,a,m} \Mean  | 	\widehat{V}_{t,-k}^{a,m} (s)-V_t^{a}(s) |^2$ are of the order $o(n^{-1/4})$. 
\end{asmp}

First, Assumptions \ref{asmp:est} and \ref{asmp:boundedtd} are mild. Given the boundedness of the rewards, it is reasonable to assume both the ATE and the temporal difference error are bounded as well. 

Second, Assumptions \ref{asmp:nonsingular} and \ref{asmp:boundedregcoef} are also mild. In particular, Assumption \ref{asmp:boundedregcoef} is the classical no unit root assumption for the state process. Both assumptions are crucial to ensure the consistency of the OLS estimators defined in Section \ref{subsec:methods}. 

Third, in many real-world applications, states are naturally bounded, making Assumption \ref{asmp:boundstates} reasonable. 

Fourth, Assumption \ref{asmp:boundtrans} is a standard condition often used to ensure the asymptotic distribution of the LSTD estimator \citep{shi2022statistical,shi2023dynamic}. This assumption is intrinsically linked to the overlap condition, which is critical for maintaining the boundedness of the density ratio $\omega_t$. The overlap condition is commonly imposed in OPE \citep[see e.g.,][]{kallus2020double,kallus2022efficiently,liao2022batch}.  

Fifth, Assumption \ref{asmp:sievebasis} is quite reasonable and is automatically satisfied when tensor product B-splines or wavelet basis functions are used for $\varphi_t$ \citep{chen2015optimal}.

Lastly, we denote the estimated value function by $\widehat{V}_{t,-k}^{a,m}$, where the superscript $m$ highlight its dependence on the $m$-switchback design -- different values of $m$ yield different value estimators. The supremum in Assumption \ref{asmp:converge} is taken with respect to $a,t$ and $m$, but not over $k$, since the trajectories are i.i.d., making the expectation invariant across different $k$. Assumptions of this type are commonly imposed in the literature for valid statistical inference of the ATE \citep[see e.g.,][]{chernozhukov_doubledebiased_2018}. 
    
\subsection{Theorem \ref{thm::main}}\label{thm:full_thm}
In the following, we provide a detailed statement of Theorem \ref{thm::main}. Its proof will be presented in the next subsection.

\textbf{Theorem \ref{thm::main}}. 
Under the given conditions, the difference in the MSE of the ATE estimator between the alternating-day design and an $m$-switchback design (where each switch duration equals $m$) is lower bounded by
\begin{eqnarray*}
    \begin{split}
        \frac{16}{nT^2} \sum_{\substack{k_2-k_1=1,3,5, \ldots\\ 0\le k_1<k_2< T/m}} \sum_{l_1,l_2=1}^{m} \sigma_{e}(l_1+k_1m, l_2+k_2 m)
        -\frac{c\delta R_{\max}^2}{n}-\text{reminder term},
    \end{split}
\end{eqnarray*}
for some constant $c>0$ and some reminder term whose order if estimator-dependent:  
\begin{itemize}
    \item For the \textbf{OLS estimator}, under Assumptions \ref{asmp:reward}-- \ref{asmp:boundstates}, its reminder term $=O(n^{-3/2} R_{\max}^2\log(nT))$.
    \item For the \textbf{DRL estimator}, under Assumptions \ref{asmp:reward}, \ref{asmp:boundtrans}, \ref{asmp:boundedtd} and \ref{asmp:converge} its reminder term is of the order
\begin{equation*}
   O\Big[\max \Big( \frac{ R_{\max}^2 \textrm{err}_\omega }{{n}}, \frac{ R_{\max}^2 \textrm{err}_v }{{n}},\frac{R_{\max}^2 \textrm{err}_\omega  \textrm{err}_v}{\sqrt{n}} \Big)
    \Big].
\end{equation*}
    \item For the \textbf{LSTD estimator}, under Assumptions \ref{asmp:reward}, \ref{asmp:boundtrans} -- \ref{asmp:converge}, its reminder term is of the order 
\begin{equation*}
    O\Big[\max \Big( \frac{ R_{\max}^2 \sqrt{\log(nT)} \textrm{err}_\omega  }{{n}}, \frac{ R_{\max}^2 \sqrt{\log(nT)} \textrm{err}_v  }{{n}},  
    \frac{R_{\max}^2 \textrm{err}_\omega \textrm{err}_v}{\sqrt{n}} \Big)
    \Big].
\end{equation*}
\end{itemize}

\subsection{Proof of Theorem \ref{thm::main}-OLS and Corollaries \ref{cor1}--\ref{cor3}}

In this subsection, we first consider the OLS estimator and prove all the corollaries. The proofs for the LSTD and DRL estimators are given in the next subsection. Our proof heavily relies on the strength of carryover effects $\delta$,  which implies that $\|\Gamma_t\|=O(\delta)$ under the assumption of linear model \eqref{linear_mdp}. A key ingredient of the proof of Theorem \ref{thm::main} is the following lemma, which demonstrates that, as $\delta\to 0$, each state becomes asymptotically uncorrelated with the current action. It largely simplifies the calculation of the asymptotic variance of the OLS estimator.

\begin{lemma}\label{lemma:ols_a2}
	Suppose both Assumption \ref{asmp:boundedregcoef} and the linear model assumption in \eqref{linear_mdp} are satisfied. Then for each $t$, we have $\|\Cov(S_t, A_t)\|_2= O(\delta)$.
\end{lemma}

\begin{proof}
	According to \eqref{eqn::IS_model}, we obtain that
	\begin{equation*}
		\Cov(S_{t}, A_{t})= \Big(\Pi_{k=1}^{t-1}\Phi_k\Big)\Cov(S_1,A_t)+	\sum_{k=1}^{t-1} \Big( \Pi_{l=k+1}^{t-1} \Phi_l \Big) \Gamma_k\Cov(A_{k}, A_{t})+\sum_{k=1}^{t-1} \Big( \Pi_{l=k+1}^{t-1} \Phi_l \Big) \Cov(E_k, A_t),
	\end{equation*}
    where we recall that $E_t=S_{t+1}-\mathbb{E}(S_{t+1}|A_t,S_t)$. 
    
	For each design, $A_t$ is uniquely determined by $A_1$. More specifically, it can take one of two values: either $A_1$ or $1-A_1$. Since $A_1$ is uncorrelated with $S_1$ and $\{E_k\}_k$, the same holds true for $A_t$. It follows that
	\begin{eqnarray*}
		\Cov(S_{t}, A_{t})=\sum_{k=1}^{t-1} \Big( \Pi_{l=k+1}^{t-1} \Phi_l \Big) \Gamma_k\Cov(A_{k}, A_{t}).
	\end{eqnarray*}
	Since each $\|\Gamma_t\|$ is $O(\delta)$, and Assumption \ref{asmp:boundedregcoef} implies that $\max_{t} \| \Phi_t \| \leq \rho_{\Phi} <1$, we have
    \begin{equation*}
        \| \sum_{k=1}^{t-1} \Big( \Pi_{l=k+1}^{t-1} \Phi_l \Big)\Gamma_k\Cov(A_{k}, A_{t}) ) \|_2 \leq  \sum_{k=1}^{t-1} \| \Pi_{l=k+1}^{t-1} \Phi_l \|\cdot \|\Gamma_k\|_2 \leq \sum_{k=1}^{t-1} \Pi_{l=k+1}^{t-1} \|\Phi_l \| \cdot O(\delta) \leq \sum_{k=1}^{t-1} \rho_{\Phi}^{t-k-1} \cdot O(\delta) ,
    \end{equation*}
    we have $\| \Cov(S_{t}, A_{t}) \|_2=O( \frac{\delta (1-\rho_{\Phi}^{t-1})}{1-\rho_{\Phi}})=O(\delta)$, for any $t=1, \cdots, T$. This completes the proof of Lemma \ref{lemma:ols_a2}. 
\end{proof}

The next lemma obtains a linear representation for the OLS-based ATE estimator. 
\begin{lemma}\label{lemma:ols_a3}
Suppose Assumptions \ref{asmp:reward}, and \ref{asmp:nonsingular} -- \ref{asmp:boundstates} and the linear model assumption in \eqref{linear_mdp} are satisfied. For sufficiently large $n$ and any large constant $\kappa>0$, we have with probability at least $1-O((nT)^{-\kappa})$ that the difference between the OLS-based ATE estimator and the ATE itself can be represented by 
\begin{eqnarray}\label{eqn:lemma2errorbound}
    \begin{split}
\frac{4}{nT}\sum_{i=1}^n \sum_{t=1}^T {(A_{i,t}-\frac{1}{2})e_{i,t}}+\frac{4}{nT}\sum_{i=1}^n\sum_{t=2}^{T}  \beta_{t}^\top \Big[\sum_{k=1}^{t-1} (\Phi_{t-1} \cdots \Phi_{k+1} ) {(A_{i,k}-\frac{1}{2})}E_{i,k}\Big]\\
        O\Big( \frac{\delta}{T} \sum_{t=1}^T\Big\|\frac{1}{n}\sum_{i=1}^n \mu_{i,t}\Big\|_2\Big) + O\Big(\frac{R_{\max}\log (nT)}{n}\Big),
    \end{split}
\end{eqnarray}
where $\mu_{i,t}$ denote the vector of residuals
\begin{eqnarray*}
    \mu_{i,t}=(e_{i,t},A_{i,t}e_{i,t},S_{i,t}^\top e_{i,t},R_{\max}E_{i,t}^\top,R_{\max} A_{i,t}E_{i,t}^\top,R_{\max}\textrm{vec}(S_{i,t}E_{i,t}^\top))^\top,
\end{eqnarray*}
and vec$(\bullet)$ denotes the operator that vectorize a matrix into a row vector. 
\end{lemma}
Here, the first two terms in \eqref{eqn:lemma2errorbound} are closely related to the autocorrelation term in \eqref{eq:ATE_difference}. Indeed, \eqref{eq:ATE_difference} equals the difference in their variance between AD and an $m$-switchback design. Meanwhile, the third term is closely related to the carryover effect term in \eqref{eq:ATE_difference} and the last term is a high-order reminder term. 
\begin{proof}
For any $m$-switchback design, since the initial action is uniformly generated from $\{0,1\}$, we have $\Mean A_t=1/2$, $t=1,\ldots,T$. According to Lemma \ref{lemma:ols_a2},
\begin{eqnarray}\label{eqn:Sigma}
    \Mean\begin{pmatrix}
            1	\\
            S_t	\\
            A_t
        \end{pmatrix} (1,S_t^\top,A_t)=\left(\begin{array}{ccc}
        1 & \Mean S_t^\top & \Mean A_t \\
        \Mean S_t & \Mean S_t S_t^\top & \Mean A_t S_t \\
        \Mean A_t & \Mean A_t S_t^\top & \Mean A_t
    \end{array}
    \right) =\underbrace{\left(\begin{array}{ccc}
            1 & \Mean S_t^\top & 0.5 \\
            \Mean S_t & \Mean S_t S_t^\top & 0.5 \Mean S_t + O(\delta)\\
            0.5 & 0.5 \Mean S_t^\top + O(\delta) & 0.5
        \end{array}
        \right)}_{\Sigma_t}.
\end{eqnarray}
We first show that $\Sigma_t$ is invertible under Assumptions \ref{asmp:nonsingular} and \ref{asmp:boundstates}. 
Define
\begin{eqnarray*}
    \Sigma_t^*=\left(\begin{array}{ccc}
            1 & \Mean S_t^\top & 0.5 \\
            \Mean S_t & \Mean S_t S_t^\top & 0.5 \Mean S_t \\
            0.5 & 0.5 \Mean S_t^\top  & 0.5
        \end{array}
        \right).
\end{eqnarray*}
With some calculations, we can represent $\Sigma_t^*$ by 
\begin{eqnarray*}
    \underbrace{\left(\begin{array}{ccc}
        1 &  &\\
        \mathbb{E} S_t & 1 & \\
        0.5 & & 1
    \end{array}\right)}_{L}
    \left(\begin{array}{ccc}
            1 & &  \\
            & \Cov(S_t) &  \\
             &   & 0.25
        \end{array}
        \right)
    \underbrace{\left(\begin{array}{ccc}
        1 & \mathbb{E} S_t^\top & 0.5\\
        & 1 & \\
        & & 1
    \end{array}\right)}_{L^\top},
\end{eqnarray*}
where $L$ is a lower-triangular matrix and is hence invertible. Under Assumption \ref{asmp:nonsingular}, $\Sigma_t^*$ is invertible and its inverse satisfies
\begin{eqnarray}\label{eqn:sigmatinverse}
    \|(\Sigma_t^*)^{-1}\|_2=\sup_{a:\|a\|_2=1} a^\top \underbrace{\left(\begin{array}{ccc}
        1 & -\mathbb{E} S_t^\top & -0.5\\
        & 1 & \\
        & & 1
    \end{array}\right)
    \left(\begin{array}{ccc}
            1 & &  \\
            & \Cov^{-1}(S_t) &  \\
             &   & 4
        \end{array}
        \right)
    \left(\begin{array}{ccc}
        1 &  &\\
        -\mathbb{E} S_t & 1 & \\
        -0.5 & & 1
    \end{array}\right)}_{\Sigma_t^{*-1}}a,
\end{eqnarray}
which is of the order $O(\lambda_{\min}^{-1}[\Cov(S_t)])$ under Assumption \ref{asmp:nonsingular}. 

Consequently, the maximum eigenvalue of $(\Sigma_t^*)^{-1}$ is of the order $O(\lambda_{\min}^{-1}[\Cov(S_t)])$. Equivalently, the minimum eigenvalue of $\Sigma_t^*$ is lower bounded by $c \lambda_{\min}[\Cov(S_t)]$ for some constant $c>0$. It follows that
\begin{eqnarray}\label{eqn:minSigmat1}
    \lambda_{\min}(\Sigma_t)= \inf_{a:\|a\|_2=1} a^\top \Sigma_t a\ge \inf_{a:\|a\|_2=1} a^\top \Sigma_t^* a - \|\Sigma_t-\Sigma_t^*\|_2= c \lambda_{\min}[\Cov(S_t)] - O(\delta),
\end{eqnarray}
and the rightmost term satisfies
\begin{eqnarray}\label{eqn:minSigmat2}
    c \lambda_{\min}[\Cov(S_t)] - O(\delta)=\frac{c}{2}\lambda_{\min}[\Cov(S_t)]+\frac{c}{2}\lambda_{\min}[\Cov(S_t)]- O(\delta)\ge \frac{c\epsilon}{2}+\frac{c\bar{c}\delta}{2}-O(\delta)\ge \frac{c\epsilon}{2},
\end{eqnarray}
under Assumption \ref{asmp:nonsingular}, which requires that $\lambda_{\min}[\Cov(S_t)]\ge \max(\epsilon,\bar{c} \delta)$ for some sufficiently large constant $\bar{c}>0$. This proves the invertibility of $\Sigma_t$. 

We next analyze the ATE estimator. Recall from \eqref{ate_est_formula} that
\begin{eqnarray*}
    \textrm{ATE}=\frac{1}{T}\sum_{t=1}^T \gamma_t+\frac{1}{T}\sum_{t=2}^T \beta_t^\top \Big[ \sum_{k=1}^{t-1} (\Phi_{t-1}\Phi_{t-2}\ldots\Phi_{k+1}) \Gamma_k \Big]
\end{eqnarray*}
The OLS-based estimator is constructed by plugging the estimated $\{\gamma_t,\beta_t,\Phi_t,\Gamma_t\}_t$ into this expression. Applying Taylor's expansion to the OLS-based estimator around the true parameter values yields that
\begin{eqnarray}\label{eqn:someeqn}
\begin{split}
    &\textrm{ATE}_{\textrm{SB}}^{(m)}-\textrm{ATE}=
    \underbrace{\frac{1}{T}\sum_{t=1}^T (\widehat{\gamma}_t-\gamma_t)}_{I_1}+ \underbrace{\frac{1}{T}\sum_{t=2}^T \beta_t^\top \Big[ \sum_{k=1}^{t-1}(\Phi_{t-1}\ldots\Phi_{k+1})(\widehat{\Gamma}_k-\Gamma_k) \Big]}_{I_2}\\
    +&\underbrace{\frac{1}{T}\sum_{t=2}^T (\widehat{\beta}_t-\beta_t)^\top\Big[ \sum_{k=1}^{t-1}(\Phi_{t-1}\ldots\Phi_{k+1}) \Gamma_k \Big]}_{I_3}+\underbrace{\sum_{t=2}^{T-1}\frac{d\textrm{IE}}{d\Phi_{t}}(\widehat{\Phi}_t-\Phi_t)}_{I_4}+ \textrm{second-order term},
\end{split}
\end{eqnarray}
where IE is a shorthand for the indirect/delayed effect $T^{-1}\sum_{t=2}^T \beta_t^\top[\sum_{k=1}^{t-1}(\Phi_{t-1}\ldots\Phi_{k+1})\Gamma_k]$. 
    
Notice that the first four terms on the RHS of \eqref{eqn:someeqn} (e.g., $I_1$-$I_4$) are the first-order terms, which we now analyze.

\textbf{Analysis of $I_1$}. By definition, we have
\begin{eqnarray}\label{eqn:thetahat-theta}
    \left(
    \begin{array}{c}
        \widehat{\phi}_t-\phi_t  \\
        \widehat{\beta}_t-\beta_t \\
        \widehat{\gamma}_t-\gamma_t
    \end{array}
    \right)=\widehat{\Sigma}_t^{-1} \frac{1}{n}\sum_{i=1}^n \left(
    \begin{array}{c}
        e_{i,t}  \\
        S_{i,t}e_{i,t} \\
        A_{i,t}e_{i,t}
    \end{array}
    \right).
\end{eqnarray}
Together with \eqref{eqn:Sigma}, one can show under Assumptions \ref{asmp:reward} and \ref{asmp:nonsingular} that
\begin{eqnarray*}
    \left(
    \begin{array}{c}
        \widehat{\phi}_t-\phi_t  \\
        \widehat{\beta}_t-\beta_t \\
        \widehat{\gamma}_t-\gamma_t
    \end{array}
    \right)=\underbrace{\Sigma_t^{*-1} \frac{1}{n}\sum_{i=1}^n \left(
    \begin{array}{c}
        e_{i,t}  \\
        S_{i,t}e_{i,t} \\
        A_{i,t}e_{i,t}
    \end{array}
    \right)}_{\textrm{leading term}}+  \underbrace{[\widehat{\Sigma}_t^{-1}-\Sigma_t^{*-1}]\frac{1}{n}\sum_{i=1}^n \left(
    \begin{array}{c}
        e_{i,t}  \\
        S_{i,t}e_{i,t} \\
        A_{i,t}e_{i,t}
    \end{array}
    \right)}_{\textrm{reminder term}}.
\end{eqnarray*}
Consider the leading term first. Recall that \eqref{eqn:sigmatinverse} obtains a closed-form expression for $\Sigma_t^{*-1}$. With some calculations, it can be shown that the last row of $\Sigma_t^{*-1}$ equals $(-2, 0, 4)$. This leads to
\begin{eqnarray}\label{eqn:I1lead}
    \textrm{the last entry of the leading term}=\frac{4}{n}\sum_{i=1}^n (A_{i,t}-\frac{1}{2})e_{i,t}. 
\end{eqnarray}
Next, consider the reminder term. Note that $(\widehat{\Sigma}_t^{-1}-\Sigma_t^{*-1})=\widehat{\Sigma}_t^{-1} (\Sigma_t^*-\widehat{\Sigma}_t)\Sigma_t^{*-1}$. The matrix in the middle can be decomposed into the sum of $\Sigma_t^*-\Sigma_t$ and $\Sigma_t-\widehat{\Sigma}_t$. The spectral norm of the first difference, $\Sigma_t^*-\Sigma_t$, is of the order $O(\delta)$. Below, we aim to apply the matrix Bernstein's inequality \citep{tropp2012user} to upper bound the spectral norm of the second difference. 

Denote $Z_{it}=(1,S_{it}^\top, A_{it})^\top \in \mathbb{R}^{d+2}$, and $X_{it}=Z_{it}Z_{it}^\top \in \mathbb{R}^{(d+2) \times (d+2)}$, $X_{it}^*=X_{it}-\Mean X_{it} \in \mathbb{R}^{(d+2) \times (d+2)}$. Clearly, $\Sigma_t=\Mean X_{it}$, $\widehat{\Sigma}_t=n^{-1}\sum_{i=1}^n X_{it}$. Under the Assumptions \ref{asmp:nonsingular} and \ref{asmp:boundstates}, we have $\max_t(\|\widehat{\Sigma}_t \|_2, \| \Sigma_t\|_2 ) \leq c_{\Sigma}$ and  $\max_{i,t} \| X_{it}^*\|_2 \leq c_{\Sigma}$, for some finite constant $0<c_{\Sigma}<\infty$. Similarly, let $\nu^2=\max_{t} \| |\sum_{i=1}^n \Mean (X_{it}^*)^2 \|$. We have $\nu^2=O(n)$. Applying the matrix Bernstein's inequality allows us to obtain the following high probability tail bound for $ \max_t \|n^{-1} \sum_{i=1}^n X_{it}^* \|_2$,  
\begin{eqnarray}\label{eqn:somematrixBernstein}
    \mathbb{P}\Big(\max_t \|\frac{1}{n}\sum_{i=1}^n X_{it}^* \|_2 \leq  \tau \Big) \geq 1- 2 T (d+2)\exp\Big(-\frac{n^2 \tau^2 /2}{\nu^2 + \frac{n \tau c_{\Sigma}}{3}}\Big), 
\end{eqnarray} 
 for every $\tau>0$. As the state dimension $d$ is fixed, by setting $\tau$ to be proportional to $n^{-1/2}\kappa \sqrt{\log (nT)}$, the above inequality holds with high probability, at least $1-O((nT)^{-\kappa})$, for any sufficient large $\kappa>0$. 

 To summarize, on the event where \eqref{eqn:somematrixBernstein} holds with $\tau$ proportional to $n^{-1/2}\kappa \sqrt{\log (nT)}$, we have $\|\widehat{\Sigma}_t-\Sigma_t^*\|_2=O(\delta+n^{-1/2}\kappa \sqrt{\log (nT)})$.  Given a sufficiently large $n$, using similar arguments to the proofs of \eqref{eqn:minSigmat1} and \eqref{eqn:minSigmat2}, we can show that 
 \begin{eqnarray}\label{eqn:hatsigmalowerbound}
    \lambda_{\min}(\widehat{\Sigma}_t^{-1})\ge \frac{c\epsilon}{4}. 
 \end{eqnarray}
Since
\begin{eqnarray*}
     \widehat{\Sigma}_t^{-1}-\Sigma_t^{*-1}=\Sigma_t^{*-1} [\Sigma_t^*-\Sigma_t ]\widehat{\Sigma}_t^{-1}+\Sigma_t^{*-1} [\Sigma_t-\widehat{\Sigma}_t]\widehat{\Sigma}_t^{-1},
\end{eqnarray*}
the reminder term can be similarly decomposed into 
\begin{eqnarray*}
    \Sigma_t^{*-1} [\Sigma_t^*-\Sigma_t ]\widehat{\Sigma}_t^{-1}\frac{1}{n}\sum_{i=1}^n \left(
    \begin{array}{c}
        e_{i,t}  \\
        S_{i,t}e_{i,t} \\
        A_{i,t}e_{i,t}
    \end{array}
    \right)+\Sigma_t^{*-1} [\Sigma_t-\widehat{\Sigma}_t]\widehat{\Sigma}_t^{-1}\frac{1}{n}\sum_{i=1}^n \left(
    \begin{array}{c}
        e_{i,t}  \\
        S_{i,t}e_{i,t} \\
        A_{i,t}e_{i,t}
    \end{array}
    \right).
\end{eqnarray*}
According to \eqref{eqn:thetahat-theta}, \eqref{eqn:hatsigmalowerbound}, the boundedness of $\|(\Sigma_t^*)^{-1}\|_2$ (see \eqref{eqn:sigmatinverse}) and that $\|\Sigma_t^*-\Sigma_t\|_2=O(\delta)$, the first term is of the order $O(\delta \|n^{-1}\sum_{i=1}^n Z_{i,t}e_{i,t}\|_2)$. 

Consider the second term. By definition,
\begin{eqnarray*}
    (\alpha_t,\beta_t^\top,\gamma_t)^\top=\Sigma_t^{-1}\mathbb{E}(1,S_t,A_t)^\top r(S_t,A_t). 
\end{eqnarray*}
Combining \eqref{eqn:minSigmat1} with \eqref{eqn:minSigmat2} yields the boundedness of $\|\Sigma_t^{-1}\|_2$. This together with the bounded states and bounded rewards assumptions leads to 
\begin{eqnarray}\label{eqn:boundedbeta}
    \sup_t \|(\alpha_t,\beta_t^\top,\gamma_t)\|_2=O(R_{\max}). 
\end{eqnarray}
This together with the bounded rewards and states assumptions yield that $\max_t |e_t|=O(R_{\max})$. Using similar arguments to those for establishing \eqref{eqn:somematrixBernstein}, it can be shown that $\|n^{-1}\sum_{i=1}^n Z_{it}e_{it}\|_2$ is of the order $n^{-1/2}\kappa R_{\max} \sqrt{\log (nT)}$ with probability at least $1-O((nT)^{-\kappa})$. This together with \eqref{eqn:somematrixBernstein}, \eqref{eqn:hatsigmalowerbound} and the boundedness of $\|(\Sigma_t^*)^{-1}\|_2$ yields that the second term is of the order $O(n^{-1}R_{\max}\log (nT))$, with probability at least $1-O((nT)^{-\kappa})$. 

As such, we obtain
\begin{eqnarray*}
    \textrm{the reminder term}=O\Big(
    \|\frac{\delta}{n}\sum_{i=1}^nZ_{i,t}e_{i,t}\|_2\Big)+O\Big(\frac{R_{\max}\log(nT)}{n}\Big),
\end{eqnarray*}
which together with \eqref{eqn:I1lead} yields that
\begin{eqnarray}\label{eqn:I1}
    I_1=\frac{4}{nT}\sum_{i=1}^n\sum_{t=1}^T (A_{i,t}-\frac{1}{2})e_{i,t}+O\Big(\frac{\delta}{T}\sum_{t=1}^T \Big\|\frac{1}{n}\sum_{i=1}^n Z_{i,t}e_{i,t}\Big\|_2\Big)+O\Big(\frac{R_{\max}\log(nT)}{n}\Big).
\end{eqnarray}

\textbf{Analysis of }$I_2$. The analysis of $I_2$ is very similar to that of $I_1$. Specifically, using similar arguments to the proof of \eqref{eqn:I1}, it can be shown that 
\begin{eqnarray}\label{eqn:gammak}
    \widehat{\Gamma}_t-\Gamma_t=\frac{4}{n}\sum_{i=1}^n (A_{i,t}-\frac{1}{2})E_{i,t}+O\Big(\frac{\delta}{T}\sum_{t=1}^T \Big\|\frac{1}{n}\sum_{i=1}^n \textrm{vec}(Z_{i,t}E_{i,t}^\top)\Big\|_2\Big)+O\Big(\frac{\log(nT)}{n}\Big).
\end{eqnarray}
By \eqref{eqn:boundedbeta}, $\sup_t \|\beta_t\|_2=O(R_{\max})$. Since $\max_t \|\Phi_t\|_2\le \rho_{\Phi}$, using similar arguments to the proof of Lemma \ref{lemma:ols_a2}, we obtain that $\sup_t \|\beta_t^\top \sum_{k=1}^{t-1} (\Phi_{t-1}\ldots\Phi_{k+1})\|_2=O(R_{\max})$. This together with \eqref{eqn:gammak} yields that
\begin{eqnarray}\label{eqn:I2}
\begin{split}
    &I_2=\frac{4}{nT}\sum_{i=1}^n\sum_{t=2}^{T}  \beta_{t}^\top \Big[\sum_{k=1}^{t-1} (\Phi_{t-1} \cdots \Phi_{k+1} ) {(A_{i,k}-\frac{1}{2})}E_{i,k}\Big]\\ +&O\Big(\frac{\delta R_{\max}}{T}\sum_{t=1}^T \Big\|\frac{1}{n}\sum_{i=1}^n  \textrm{vec}(Z_{i,t}E_{i,t}^\top)\Big\|_2\Big)
    +O\Big(\frac{R_{\max}\log(nT)}{n}\Big).
\end{split}
\end{eqnarray}
 
 \textbf{Analysis of }$I_3$. Since $\| \Gamma_t\|_2=O(\delta)$ and $\max_t \|\Phi_t\|_2\le \rho_{\Phi}$, using similar arguments to the proof of Lemma \ref{lemma:ols_a2}, it can be shown that $\sup_t \|\sum_{k=1}^{t-1} (\Phi_{t-1}\ldots\Phi_{k+1}) \Gamma_k\|_2=O(\delta)$. It follows that
\begin{eqnarray}\label{eqn:I30}
    I_3=O\Big(\frac{\delta}{T}\sum_{t=2}^T \|\widehat{\beta}_t-\beta_t\|_2\Big).
\end{eqnarray}
According to \eqref{eqn:hatsigmalowerbound}, $\|\widehat{\beta}_t-\beta_t\|_2$ is of the order of magnitude $O(\|n^{-1}\sum_{i=1}^n Z_{i,t}e_{i,t}\|_2)$. This together with \eqref{eqn:I30} leads to
\begin{eqnarray}\label{eqn:I3}
    I_3=O\Big(\frac{\delta}{T}\sum_{t=1}^T \|\frac{1}{n}\sum_{i=1}^n Z_{i,t}e_{i,t}\|_2\Big).
\end{eqnarray}

\textbf{Analysis of }$I_4$. 
Using similar arguments to the proof of the proof of Lemma \ref{lemma:ols_a2}, one can show that $\max_t \|\frac{d\textrm{IE}}{d\Phi_{t}} \|_2=O(T^{-1} \delta R_{\max})$. 
Meanwhile, similar to the analysis of $I_3$, we can show that $\|\widehat{\Phi}_t-\Phi_t\|_2$ is of the same order of magnitude to $\|n^{-1}\sum_{i=1}^n Z_{i,t}E_{i,t}\|_2$. As such, we obtain that
\begin{eqnarray}\label{eqn:I4}
    I_4=O\Big(\frac{\delta R_{\max}}{T}\sum_{t=1}^T \|\frac{1}{n}\sum_{i=1}^n \textrm{vec}(Z_{i,t}E_{i,t}^\top)\|_2\Big). 
\end{eqnarray}

\textbf{Second-order terms}. Using similar arguments to the proof of \eqref{eqn:somematrixBernstein}, it can be shown that these second-order terms are of the order $O(n^{-1}\log(nT)R_{\max})$, with probability at least $1-O((nT)^{-\kappa})$. Together with \eqref{eqn:someeqn}, \eqref{eqn:I1}, \eqref{eqn:I2}, \eqref{eqn:I3} and \eqref{eqn:I4}, we conclude the proof of Lemma \ref{lemma:ols_a3}. 
\end{proof}

\textbf{Proof of Theorem \ref{thm::main}}-OLS
\begin{proof}
To ease notation, let $c_t$ denote the $d$-dimensional vector $[\sum_{k=t}^{T-1}\beta_{k+1}^\top (\Phi_k\cdots \Phi_{t+1})]^\top$. According to the Lemma \ref{lemma:ols_a3}, with a sufficiently large $\kappa >0$, 
	the asymptotic MSE of the ATE estimator is given by
	\begin{eqnarray}
		\label{eqn:decompose_ate_var}
		\begin{aligned}
			\textrm{MSE}(\ATE_{\text{SB}}^{(m)} ) &= \mathbb{E}(\ATE_{\text{SB}}^{(m)}-\ATE)^2 \prob(\textrm{All the high probability bounds in Lemma \ref{lemma:ols_a3} hold})\\
			&+\mathbb{E}(\ATE_{\text{SB}}^{(m)}-\ATE)^2 \times \prob(\textrm{One of the high probability bounds in Lemma \ref{lemma:ols_a3} does not hold})\\
			&= \mathbb{E}(\ATE_{\text{SB}}^{(m)}-\ATE)^2 \times \Big[1-O((nT)^{-\kappa})\Big]
			+\mathbb{E}(\ATE_{\text{SB}}^{(m)}-\ATE)^2 \times \Big[O((nT)^{-\kappa})\Big].
		\end{aligned}
	\end{eqnarray} 
When the high probability bounds do not holds, our ATE estimator can be bounded by $| \ATE_{\text{SB}}^{(m)} |\leq R_{\max}$ under Assumption \ref{asmp:est}. Since $\kappa$ can be made arbitrarily large, the second term in the last line can be made arbitrarily small. Consequently, in the rest of the proof, we focus on the case where all high probability bounds in Lemma \ref{lemma:ols_a3} hold. All the expectations below are calculated by explicitly assuming that these bounds hold. 

Let $I$ denote the third term in 
\eqref{eqn:lemma2errorbound} that is of the order $O(T^{-1}\delta\sum_{t=1}^T \|n^{-1}\sum_{i=1}^n \mu_{i,t}\|_2)$ and $I^*$ denote the second-order term in \eqref{eqn:lemma2errorbound} that is of the order $O(n^{-1}R_{\max}\log(nT))$. Since the residual process $\{e_t\}_t$ is independent of all state-action pairs, it is also independent of $\{E_t\}_t$. Consequently, we obtain that
\begin{eqnarray}
    \label{eqn:someimmediatestep}
    \begin{aligned}
   &\textrm{MSE}(\ATE_{\text{SB}}^{(m)}) = \frac{16}{nT^2}\Var\Big[\sum_{t=1}^T {(A_t-\frac{1}{2})e_t} \Big] +\frac{16}{nT^2}\Var\Big[\sum_{t=1}^{T-1} {(A_t-\frac{1}{2})c_{t}^\top E_t} \Big] \\
   +&\mathbb{E}(I+I^*)^2+2\mathbb{E} (I+I^*) \Big[\frac{4}{nT}\sum_{i=1}^n \sum_{t=1}^T {(A_{i,t}-\frac{1}{2})e_{i,t}}\Big] + 2\mathbb{E}(I+I^*)\Big[\frac{4}{nT}\sum_{i=1}^n\sum_{t=1}^{T-1}  {(A_{i,t}-\frac{1}{2})} c_{t}^\top E_{i,t}\Big].
    \end{aligned}
\end{eqnarray}
Let us first focus on the second line of \eqref{eqn:someimmediatestep}. According to Cauchy-Schwarz inequality, the first term on the second line can be upper bounded by $2\mathbb{E}I^2 + 2\mathbb{E}(I^*)^2$. Using Cauchy-Schwarz inequality again, we have
\begin{eqnarray*}
    2\mathbb{E}I^2=O\Big(\mathbb{E}\Big|\frac{\delta}{T}\sum_{t=1}^T \Big\|\frac{1}{n}\sum_{i=1}^n \mu_{i,t}\Big\|_2\Big|^2\Big)=O\Big(\frac{\delta^2}{T}\sum_{t=1}^T \mathbb{E} \Big\|\frac{1}{n}\sum_{i=1}^n \mu_{i,t}\Big\|_2^2\Big). 
\end{eqnarray*}
We have shown that $\max_t |e_t|=O(R_{\max})$ in the analysis of $I_1$. Using similar arguments, it can be shown that $\max_t \|E_t\|_2=O(1)$ under Assumption \ref{asmp:boundstates}. With some calculations, we can obtain that $\mathbb{E} I^2=O(n^{-1}\delta^2 R_{\max}^2)$. Meanwhile, according to Lemma \ref{lemma:ols_a3}, we have that $\mathbb{E} (I^*)^2=O(n^{-2}R_{\max}^2\log^2 (nT))$. Consequently, we have
\begin{eqnarray}\label{IIstar}
    \mathbb{E}(I+I^*)^2=O\Big(\frac{\delta^2 R_{\max}^2}{n}\Big)+O\Big(\frac{R_{\max}^2 \log^2(nT)}{n^2}\Big).
\end{eqnarray}
Similarly, using Cauchy-Schwarz inequality, the second and third terms in the second line of \eqref{eqn:someimmediatestep} can be upper bounded by
\begin{eqnarray*}
    \sqrt{\mathbb{E}(I+I^*)^2 \mathbb{E} \Big|\frac{4}{nT}\sum_{i=1}^n \sum_{t=1}^T {(A_{i,t}-\frac{1}{2})e_{i,t}}\Big|^2}\,\,\hbox{and}\,\,
    \sqrt{\mathbb{E}(I+I^*)^2 \mathbb{E} \Big|\frac{4}{nT}\sum_{i=1}^n\sum_{t=1}^{T-1}  {(A_{i,t}-\frac{1}{2})} c_{t}^\top E_{i,t}\Big|^2},
\end{eqnarray*}
respectively. Using \eqref{IIstar}, one can similarly show that the above two expressions are of the order 
\begin{eqnarray}\label{eqn:secondlineorder}
    O\Big(\frac{\delta R_{\max}^2}{n}\Big)+O\Big(\frac{R_{\max}^2\log(nT)}{n^{3/2}}\Big). 
\end{eqnarray}
Since the state space is discrete, $\delta$ -- which equals the difference in two probability mass functions -- is bounded. Given a sufficiently large $n$, $\mathbb{E}(I+I^*)^2$ is upper bounded by \eqref{eqn:secondlineorder} as well. Consequently, the second line of \eqref{eqn:someimmediatestep} is also of the order specified in \eqref{eqn:secondlineorder}. 
	
We next consider the second term on the RHS of the first line  of \eqref{eqn:someimmediatestep}. The sequence \( \{E_t\} \) forms a martingale difference sequence with respect to the filtration  $\langle\sigma(\mathcal{F}_t): t \geq 1\rangle$ where \( \mathcal{F}_t = \{S_j, A_j\}_{j \leq t} \), and is uncorrelated with the sequence \( \{A_t\}_t \) under the switchback design.
 Additionally, under both the alternating-day design and the switchback design, each $A_t$ follows a Bernoulli(0.5) random variable. Its marginal variance is given by 0.25. As such, the second term in the first line of \eqref{eqn:someimmediatestep} equals
\begin{eqnarray}\label{eqn:secondtermfirstline}
        \sum_{t=1}^{T-1} \frac{4 c_t^\top\Cov(E_t)c_t}{nT^2 }, 
\end{eqnarray}
and is design-independent. 

Finally, we analyze the first term on the RHS of the first line of \eqref{eqn:someimmediatestep}. Consider a given $m$-switchback design. For any integers $1\le t_1\le t_2\le T$, we represent them as $t_1=l_1+k_1m$ and $t_2=l_2+k_2m$ such that $1\le l_1,l_2\le m$ and $0\leq k_1, k_2 <T/m$. By definition, $\Cov(A_{t_1},A_{t_2})$ equals $0.25$ if $k_2-k_1$ is even and $-0.25$ otherwise. 
Then the first term on the RHS of \eqref{eqn:someimmediatestep} can thus be represented as
\begin{eqnarray*}
    \frac{4}{nT^2 } \sum_{l_1=1}^m \sum_{l_2=1}^m \sum_{k_1=0}^{T/m -1}  \sum_{k_2=0}^{T/m -1}   (-1)^{ |k_1-k_2| }  \sigma_{e}(l_1+k_1m,l_2+k_2m).
\end{eqnarray*}
When $m=T$, it equals
\begin{eqnarray*}
    \frac{4}{nT^2} \sum_{l_1=1}^m \sum_{l_2=1}^m \sum_{k_1=0}^{T/m -1}  \sum_{k_2=0 }^{T/m-1}   \sigma_{e}(l_1+k_1m,l_2+k_2m).
\end{eqnarray*} 
Together with \eqref{eqn:secondlineorder} and \eqref{eqn:secondtermfirstline}, we obtain that
\begin{equation*}
    \begin{split}
\textrm{MSE}(\ATE_{\text{SB}}^{(m)} ) &= \frac{4}{nT^2} \sum_{l_1=1}^m \sum_{l_2=1}^m \sum_{k_1=0}^{T/m -1}  \sum_{k_2=0 }^{T/m -1}   (-1)^{ |k_1-k_2| }  \sigma_{e}(l_1+k_1m,l_2+k_2m)\\
        &+\frac{4 }{nT^2}\sum_{t=1}^{T-1}c_t^\top\Cov(E_t)c_t + O\Big(\frac{\delta  R_{\max}^2 }{n}\Big) + O\Big(\frac{R_{\max}^2\log(nT)}{n^{3/2}}\Big).
    \end{split}
\end{equation*}
Notice that for AD, we have$\Sigma_t=\Sigma_t^*$ and the third term on the RHS of \eqref{eqn:someeqn} equals zero. This leads to the following lower bound for the difference in the MSE: 
    \begin{equation*}
		\begin{split}
			\textrm{MSE}(\ATE_{\text{AD}} ) -	\textrm{MSE}(\ATE_{\text{SB}}^{(m)} ) \geq  \frac{16}{nT^2} \sum_{\substack{k_2-k_1=1,3,5, \ldots\\ 0 \le k_1<k_2 < T/m }} \sum_{l_1=1}^{m} \sum_{l_2=1}^{m}  \sigma_{e}(l_1+k_1 m,l_2+k_2m) - \frac{c \delta R_{\max}^2  }{n} -O\Big(\frac{R_{\max}^2\log(nT)}{n^{3/2}}\Big),
		\end{split}
	\end{equation*} 
    for some constant $c>0$. The proof is hence completed.
\end{proof}

\textit{\textbf{Proof of Corollary \ref{cor1}}}.
\begin{proof}
	Recall that under the autoregressive covariance structure, $\sigma_e(t_1, t_2)=\sigma^2 \rho^{t_2-t_1}$ for some $0<\rho<1$ and any $t_1\le t_2$. 
	It follows that
	\begin{eqnarray*}
		&&\frac{16}{T^2} \sum_{\substack{k_2-k_1=1,3,5, \ldots\\ 0 \le k_1<k_2 < T/m }} \sum_{l_1=1}^{m} \sum_{l_2=1}^{m}  \sigma_{e}(l_1+k_1m,l_2+k_2m)\\
		&=&\frac{16}{nT^2} \sum_{\substack{k_2-k_1=1,3,5, \ldots\\ 0 \le k_1<k_2< T/m }} \sum_{l_1=1}^{m} \sum_{l_2=1}^{m} \sigma^2 \rho^{(k_2-k_1)m+(l_2-l_1)} \\
		&=&\frac{16}{nT^2} \sum_{\substack{k_2-k_1=1,3,5, \ldots\\ 0\le k_1<k_2< T/m }} \sigma^2 \rho^{(k_2-k_1)m} \sum_{l_1=1}^{m} \sum_{l_2=1}^{m}  \rho^{l_2-l_1}.
	\end{eqnarray*}
	With some calculations, it is immediate to see that
	\begin{equation}\label{eqn:autoregressive}
		\sum_{l_1=1}^{m} \sum_{l_2=1}^{m}  \rho^{l_2-l_1}=\sum_{l_1=1}^{m} \rho^{-l_1} \sum_{l_2=1}^{m}  \rho^{l_2}=\frac{\rho^{-1}(1-\rho^{-m})}{1-\rho^{-1}}\frac{\rho(1-\rho^m)}{1-\rho}=\frac{\rho^{1-m}(1-\rho^m)^2}{(1-\rho)^2}.
	\end{equation}
	Similarly, when $K=\frac{T}{m} $ is odd, 
	\begin{eqnarray*}
		\sum_{\substack{k_2-k_1=1,3,5, \ldots\\ 0 \le k_1<k_2< K}} \rho^{(k_2-k_1)m}=\sum_{k_2\in \{1,\ldots,K-2\}} \rho^{k_2m}+\sum_{k_2\in \{2,\ldots,K-1\}} \rho^{(k_2-1)m}+\ldots+\rho^m+\rho^m\\
		=2\frac{\rho^m-\rho^{Km}}{1-\rho^{2m}}+2\frac{\rho^m-\rho^{(K-2)m}}{1-\rho^{2m}}+\ldots+2\frac{\rho^m-\rho^{3m}}{1-\rho^{2m}}+2\frac{\rho^m-\rho^{m}}{1-\rho^{2m}}\\
		=\frac{\rho^m (K+1)}{1-\rho^{2m}}-\frac{2\rho^m(1-\rho^{(K+1)m})}{(1-\rho^{2m})^2}=\frac{\rho^m[T/m-1-(T/m+1)\rho^{2m}+2\rho^{T+m}]}{(1-\rho^{2m})^2}.
	\end{eqnarray*}
	On the other hand, when $K$ is even, we have
	\begin{eqnarray*}
		\sum_{\substack{k_2-k_1=1,3,5, \ldots\\ 0 \le k_1<k_2< K}}\rho^{(k_2-k_1)m}=\sum_{k_2\in \{1,\ldots,K-1\}} \rho^{k_2m}+\sum_{k_2\in \{2,\ldots,K-2\}} \rho^{(k_2-1)m}+\ldots+\rho^m+\rho^m\\
		=\frac{\rho^m-\rho^{(K+1)m}}{1-\rho^{2m}}+2\frac{\rho^m-\rho^{(K-1)m}}{1-\rho^{2m}}+\ldots+2\frac{\rho^m-\rho^{3m}}{1-\rho^{2m}}+2\frac{\rho^m-\rho^{m}}{1-\rho^{2m}}\\
		=\frac{\rho^m(K+1)-\rho^{(K+1)m}}{1-\rho^{2m}}-\frac{2\rho^m(1-\rho^{Km})}{(1-\rho^{2m})^2}=\frac{\rho^m[T/m-1-(T/m+1)\rho^{2m}+\rho^{T}+\rho^{T+2m}]}{(1-\rho^{2m})^2}.
	\end{eqnarray*}
	To summarize, we obtain that
	\begin{eqnarray*}
		\sum_{\substack{k_2-k_1=1,3,5, \ldots\\ 0 \le k_1<k_2< K}}\rho^{(k_2-k_1)m}=\left\{\begin{array}{ll}
			\displaystyle \frac{\rho^m[T/m-1-(T/m+1)\rho^{2m}+2\rho^{T+m}]}{(1-\rho^{2m})^2},  & K ~\textrm{is odd}  \\
			\displaystyle \frac{\rho^m[T/m-1-(T/m+1)\rho^{2m}+\rho^{T}+\rho^{T+2m}]}{(1-\rho^{2m})^2},  & K~\textrm{is even}. 
		\end{array}
		\right.
	\end{eqnarray*}
In either case, as $T\to \infty$, $T^{-1} \sum_{\substack{k_2-k_1=1,3,5, \ldots\\ 0 \le k_1<k_2< K}}\rho^{(k_2-k_1)m}$ becomes asymptotically equivalently to $\rho^m/[m(1-\rho^{2m})]$. Combining this together with \eqref{eqn:autoregressive} yields the desired result. The proof is hence completed.
\end{proof}
\textit{\textbf{Proof of Corollary \ref{cor2}}}. 
\begin{proof}
	Under the moving average covariance structure, we have for any $t_1\le t_2$ that
	\begin{eqnarray*}
		\Cov(e_{t_1}, e_{t_2})= \frac{1}{K}\Mean (\varepsilon_{t_1+1}+\varepsilon_{t_1+2} +\ldots \varepsilon_{t_1+K}) 		(\varepsilon_{t_2+1}+\varepsilon_{t_2+2} +\ldots \varepsilon_{t_2+K})  
		=\sigma^2\frac{[K-t_2+t_1]_+}{K},
	\end{eqnarray*}
	where $[z]_+=\max(z,0)$ for any $z\in \mathbb{R}$. Accordingly, we have
	\begin{eqnarray}\label{eqn:ma}
		\begin{split}
			&\frac{16}{T^2} \sum_{\substack{k_2-k_1=1,3,5, \ldots\\ 0 \le k_1<k_2< T/m }} \sum_{l_1=1}^{m} \sum_{l_2=1}^{m}  \sigma_{e}(l_1+k_1m,l_2+k_2m)\\
			=& \frac{16\sigma^2}{T^2} \sum_{\substack{k_2-k_1=1,3,5, \ldots\\ 0 \le k_1<k_2<T/m }} \sum_{l_1=1}^{m} \sum_{l_2=1}^{m} \frac{1}{K}
			[K-(k_2-k_1)m-(l_2-l_1)]_+.
		\end{split}
	\end{eqnarray} 
	Since $m\ge K$, the expression $K-(k_2-k_1)m-(l_2-l_1)$ remains positive only when $k_2-k_1=1$ and that $l_1-l_2>m-K$. It follows that the RHS of \eqref{eqn:ma} equals
	\begin{eqnarray*}
		\frac{16\sigma^2}{T^2} \sum_{\substack{k_2-k_1=1\\ 0 \le k_1<k_2< T/m }}\sum_{l_1=1}^{m} \sum_{l_2=1}^{m} \frac{[K-m+l_1-l_2]_+}{K}=\frac{16\sigma^2(T/m-1)}{T^2}\sum_{l_1=1}^{m} \sum_{l_2=1}^{m} \frac{[K-m+l_1-l_2]_+}{K}\\
		=\frac{16\sigma^2(T/m-1)}{T^2}\Big(\frac{1+2+\ldots+K-1}{K}+\frac{1+2+\ldots+K-2}{K} + \ldots +\frac{1}{K}\Big)\\
		=\frac{8\sigma^2(T/m-1)(K^2-1)}{3T^2}.
	\end{eqnarray*}
	The proof is hence completed. 
\end{proof}
\textit{\textbf{Proof of Corollary \ref{cor3}}}
\begin{proof}
	Recall that under the exchangeable covariance structure, $\sigma_e(t_1, t_2)=\sigma^2 [\rho\mathbb{I}(t_1 \neq t_2) +\mathbb{I}(t_1 = t_2)]$ for any $t_2 \ge t_1$. 
	It follows that 
	\begin{eqnarray*}
		\frac{16}{T^2} \sum_{\substack{k_2-k_1=1,3,5, \ldots\\ 0 \le k_1<k_2< T/m }} \sum_{l_1=1}^{m} \sum_{l_2=1}^{m}  \sigma_{e}(l_1+(k_1-1)m,l_2+(k_2-1)m)\\= \frac{16}{nT^2} \sum_{\substack{k_2-k_1=1,3,5, \ldots\\ 0 \le k_1<k_2 < T/m }} \sum_{l_1=1}^{m} \sum_{l_2=1}^{m} \sigma^2 \rho
		=\frac{16m^2}{nT^2} \sum_{\substack{k_2-k_1=1,3,5, \ldots\\ 0 \le k_1<k_2< T/m }}\sigma^2 \rho.
	\end{eqnarray*} 
	When $m<T$ and $T/m$ is even, the number of elements in the set $\{(k_1,k_2):k_2-k_1=1,3,5, \ldots; 0 \le k_1<k_2 < T/m\}$ is given by $T/(2m)+2[T/(2m)-1]+\ldots+2=T^2/(4m^2)$. Similarly, when $m<T$ and $T/m$ is odd, it can be shown that the aforementioned set contains $T^2/(4m^2)-1/4$ elements. It follows that the reduction in MSE equals $4n^{-1}\sigma^2\rho$ when $T/m$ is even and $4n^{-1}\sigma^2\rho(1-m^2/T^2)$ otherwise. The proof is hence completed. 
\end{proof}

\subsection{Proofs of Theorem \ref{thm::main}-DRL and LSTD } \label{seubsec:DRL_LSTD}
We first establish the proof for the DRL estimator, followed by the proof for LSTD. With a slight abuse of notation, we use 
\(\omega_t^{a,m}(s) := \frac{p_t^a(s)}{p_t^m(s \mid a)}\) to represent the IS weight of the conditional probability mass function (pmf) of the state given the action. In particular, the numerator denotes the pmf of $S_t$ under the target policy $a$ and the denominator denotes the conditional pmf of $S_t$ under the $m$-switchback design given $A_t=a$. Additionally, recall that  
\(\omega_t^{a,m}(s, a') := \frac{p_t^a(s, a')}{p_t^m(s, a')}\) denotes the IS weight of the pmf for the state-action pair.

\textit{\textbf{Proof of Theorem \ref{thm::main}-DRL }}
\begin{proof}
	Recall that the cross-fitted version of the DRL estimator is given by
	\begin{eqnarray*}
		\textrm{ATE}_{\textrm{DRL}}^{(m)}=\frac{1}{nT} \sum_{k=1}^K\sum_{i \in \mathcal{D}_k} \psi(\{S_{i,t},A_{i,t},R_{i,t}\}_t;\{\widehat{V}_{t,-k}^{a,m}\}_{t,a},\{\widehat{\omega}_{t,-k}^{a,m}\}_{t,a}).
	\end{eqnarray*}
 Similar to the proof of Theorem 6 in \citet{kallus2022efficiently}, we can show that the ATE estimator $\textrm{ATE}_{\textrm{DRL}}^{(m)}$ is asymptotically equivalent to its ``oracle'' version $\textrm{ATE}_{\textrm{DRL}}^{(m)*}$ which works as well as if the nuisance functions were known in advance. In particular, we have
\begin{eqnarray*}
\textrm{ATE}_{\textrm{DRL}}^{(m)}=\underbrace{\frac{1}{nT}\sum_{i=1}^n\psi(\{S_{i,t},A_{i,t},R_{i,t}\}_t;\{V_t^a\}_{t,a},\{\omega_t^{a,m}\}_{t,a})}_{\textrm{ATE}_{\textrm{DRL}}^{(m)*}}+\textrm{reminder term}.
\end{eqnarray*}
Notice that the oracle estimator is unbiased. Hence, we only need to compute the asymptotic variances of the oracle estimators under different switchback designs to compare their MSEs. 

The rest of the proof is divided into two parts. In Part I, we calculate the asymptotic variance of the ``oracle" version of DRL estimator. In Part II, we focus on the calculation of the reminder term.  

\textit{\textbf{Part I:}} Consider the sequence of temporal difference (TD) errors $\{\varepsilon_t^a\}$  where each $\varepsilon_t^a=R_t+V_{t+1}^a(S_{t+1})-V_t^a(S_t)$. We decompose each TD error into the sum of $e_t$ and $\epsilon_t^a=r_t(A_t,S_t)+V_{t+1}^a(S_{t+1})-V_t^a(S_t)$, which yields two residual sequences. The first sequence, $\{e_t\}_t$, is correlated over time. Under the Markov assumption and the conditional mean independence assumption (CMIA), the second sequence can be shown to form a martingale difference sequence with respect to the filtration $\langle\sigma(\mathcal{F}_t):t\ge 1\rangle$ defined in the proof of Theorem \ref{thm::main}-OLS; see e.g., the proof of Theorem 1 in \citet{shi2022statistical}. Additionally, the two error processes are mutually independent. 
This yields
\begin{eqnarray}\label{eqn:asymvariance}
\begin{split}
	n\Var(\textrm{ATE}_{\textrm{DRL}}^{(m)*})=\frac{1}{T^2}\Var\Big[V_1^1(S_1)-V_1^0(S_1)+\sum_{t=1}^T \omega_t^{1,m}(A_t,S_t) \varepsilon_t^1-\sum_{t=1}^T \omega_t^{0,m}(A_t,S_t) \varepsilon_t^0\Big]\\
	=\frac{1}{T^2}\Var\Big[V_1^1(S_1)-V_1^0(S_1)\Big]+\frac{1}{T^2}\sum_{t=1}^T\Var\Big[\sum_{a=0}^1  (-1)^{a+1}\omega_t^{a,m}(A_t,S_t)\epsilon_t^a\Big]\\+\frac{1}{T^2}\Var\Big[\sum_{a=0}^1 \sum_{t=1}^T (-1)^{a+1}\omega_t^{a,m}(A_t,S_t)e_t\Big].
\end{split}
\end{eqnarray}
Notice that the first term on the second line is design-independent. 

Since
\begin{eqnarray*}
    \omega_t^{a,m}(A_t,S_t)=\frac{p_t^a(A_t,S_t)}{p_t^m(A_t,S_t)}=\frac{p_t^a(S_t)\mathbb{I}(A_t=a)}{p_t^m(A_t,S_t)}=\frac{p_t^a(S_t)\mathbb{I}(A_t=a)}{p_t^m(S_t|A_t)p_t^m(A_t)}=\frac{p_t^a(S_t)\mathbb{I}(A_t=a)}{p_t^m(S_t|a)p_t^m(a)},
\end{eqnarray*}
where $p_t^m(a)$ denotes the probability that $A_t=a$ under the $m$-switchback design, which is equal to $1/2$ by construction. 
With some calculations, the second term on the second line can be shown to equal
\begin{eqnarray}\label{eqn:error1_densti_var}
\begin{aligned}
    \frac{1}{T^2}\sum_{a=0}^1\sum_{t=1}^T\Var\Big[\omega_t^{a,m}(A_t,S_t)\epsilon_t^a\Big]&=\frac{4}{T^2}\sum_{a=0}^1\sum_{t=1}^T \Mean \frac{[p_t^a(S_t)]^2\mathbb{I}(A_t=a)\sigma^2_{\epsilon,t}(a,S_t)}{[p_t^m(S_t|a)]^2}\\
&= \frac{2}{T^2}\sum_{a=0}^1\sum_{t=1}^T \sum_{s \in \mathcal{S}}\frac{[p_t^a(s)]^2 \sigma^2_{\epsilon,t}(a,s)}{p_t^m(s|a)},
\end{aligned}
\end{eqnarray}
where $\sigma^2_{\epsilon,t}(a,s)=\Var(\epsilon_t^a|A_t=a|S_t=s)$. 

We note that
\begin{itemize}
    \item Similar to the proof of model-based methods, according to the definition of $\delta$, it can be shown that $p_t^a(s)=p_t^m(s|a)+O(\delta)$ for any $t\ge 2$ and $p_1^a(s)=p_1^m(s|a)$. Meanwhile, by Assumption \ref{asmp:boundtrans}, the transition function is bounded away from zero. This implies that $p_t^m$ is bounded away from zero as well, for any $t\ge 2$, and hence $p_t^a(s)/p_t^m(s|a)=1+O(\delta)$ for any $t\ge 1$. 
    \item Similarly, one can also show that $\sum_{s \in \mathcal{S}} |\sigma_{\epsilon,t}^2(1,s)-\sigma_{\epsilon,t}^2(0,s)| =O(\delta R_{\max}^2)$, and $\max(\sigma_{\epsilon,t}^2(0,s),\sigma_{\epsilon,t}^2(1,s))=O(R_{\max}^2)$ under Assumption \ref{asmp:boundedtd}. 
\end{itemize}
As such, the second line of \eqref{eqn:error1_densti_var} is equivalent to
\begin{equation*}
   \begin{aligned}
        \frac{2}{T^2}\sum_{a=0}^1\sum_{t=1}^T \sum_{s \in \mathcal{S}}\frac{[p_t^a(s)]^2\sigma^2_{\epsilon,t}(a,s)}{p_t^m(s|a)}&=\frac{2}{T^2}\sum_{a=0}^1\sum_{t=1}^T \sum_{s \in \mathcal{S}} p_t^a(s)\sigma^2_{\epsilon,t}(a,s)+O(\delta R_{\max}^2)\\
        &= \frac{2}{T^2}\sum_{a=0}^1\sum_{t=1}^T \sum_{s \in \mathcal{S}}{p_t^a(s)\sigma^2_{\epsilon,t}(0,s)} +O(\delta R_{\max}^2).
   \end{aligned}
\end{equation*}

Consequently, the above expression is asymptotically design-independent, i.e., 
\begin{eqnarray}\label{eqn:asymvar2}
\frac{1}{T^2}\sum_{a=0}^1\sum_{t=1}^T\Var\Big[\omega_t^{a,m}(A_t,S_t)\epsilon_t^a\Big] = \frac{2}{T^2}\sum_{t=1}^T [\Mean^0 \sigma^2_{\epsilon,t}(0,S_t)+\Mean^1 \sigma^2_{\epsilon,t}(0,S_t)]+O(\delta R_{\max}^2).
\end{eqnarray}

Based on the above discussion, it suffices to compare the third line of \eqref{eqn:asymvariance} to compare different switchback designs. Similar to \eqref{eqn:asymvar2}, we can show that the third line is asymptotically equivalent to
\begin{eqnarray*}
\frac{4}{T^2}\sum_{a=0}^1\Var\Big[\sum_{t=1}^T\mathbb{I}(A_t=a)e_t\Big] +O(\delta R_{\max}^2).
\end{eqnarray*}
 Using similar arguments in the proof of Theorem \ref{thm::main}-OLS, we obtain that the lower bound of difference in the variance between the alternating-day design and the $m$-switchback design is given by
\begin{equation*}
\frac{16}{nT^2} \sum_{\substack{k_2-k_1=1,3,5, \ldots\\ 0\le k_1<k_2< T/m }} \sum_{l_1=1}^{m} \sum_{l_2=1}^{m}  \sigma_{e}(l_1+k_1m,l_2+k_2m) -  \frac{c \delta R_{\max}^2}{n} -\textrm{reminder term}, 
\end{equation*}
for some constant $c>0$.

\textbf{Part II:}
It suffices to upper bound the absolute value of the reminder term, i.e., the difference between 
$	\textrm{ATE}_{\textrm{DRL}}^{(m)}$ and $	\textrm{ATE}_{\textrm{DRL}}^{(m)*}$. With some calculations, the reminder term can be shown to equal the sum of the following three terms:
\begin{eqnarray*}
J_1:=\frac{1}{nT} \sum_{k=1}^K \sum_{i\in \mathcal{D}_k} [\psi(\{S_{i,t},A_{i,t},R_{i,t}\};\{V_{t}^a\}_{t,a};\{\widehat{\omega}_{t,-k}^{a,m}\}_{t,a})-\psi(\{S_{i,t},A_{i,t},R_{i,t}\};\{V_t^a\}_{t,a};\{\omega_t^{a,m}\}_{t,a})].\\
	J_2:=\frac{1}{nT} \sum_{k=1}^K \sum_{i\in \mathcal{D}_k}  [\psi(\{S_{i,t},A_{i,t},R_{i,t}\};\{\widehat{V}_{t,-k}^{a,m}\}_{t,a};\{\omega_t^{a,m}\}_{t,a})-\psi(\{S_{i,t},A_{i,t},R_{i,t}\};\{V_t^a\}_{t,a};\{\omega_t^{a,m}\}_{t,a})].\\
J_3:=\frac{2}{nT}\sum_{k=1}^K \sum_{i\in \mathcal{D}_k}\sum_{t=1}^T\sum_{a=0}^1  (-1)^{a+1} \mathbb{I}(A_{i,t}=a)[\widehat{\omega}^{a,m}_{t,-k}(S_{i,t})-\omega^{a,m}_t(S_{i,t})]\\ \nonumber
	\times [\widehat{V}_{t+1,-k}^{a,m}(S_{i,t+1})-\widehat{V}_{t,-k}^{a,m}(S_{i,t})-V_{t+1}^a(S_{i,t+1})+V_t^a(S_{i,t})].
\end{eqnarray*}
We analyze each of these terms below. 

\textbf{Analysis of $J_1$}. With some calculations, we can represent $J_1$ by 
\begin{equation*}
	\begin{aligned}
		J_1&=\frac{2}{nT}\sum_{k=1}^K \sum_{i \in \mathcal{D}_k} \sum_{t=1}^T\sum_{a=0}^1  (-1)^{a+1} \mathbb{I}(A_{i,t}=a) \left[
		\widehat{\omega}_{t,-k}^{a,m} (S_{i,t})-\omega_t^{a,m}(S_{i,t}) \right] \left[
		R_{i,t}+V_{t+1}^a(S_{i,t+1})-V_{t}^a(S_{i,t})
		\right]  \\
		&=	\frac{2}{nT}\sum_{k=1}^K  \sum_{i \in \mathcal{D}_k} \sum_{t=1}^T\sum_{a=0}^1  (-1)^{a+1} \mathbb{I}(A_{i,t}=a) \left[
		\widehat{\omega}_{t,-k}^{a,m} (S_{i,t})-\omega_t^{a,m}(S_{i,t}) \right] \varepsilon_{i,t}^a.
	\end{aligned}
\end{equation*}
It follows from the Cauchy-Schwarz inequality that
\begin{equation*}
	\begin{split}
	& \Mean J_1^2=	 \Mean \left\lbrace
		\frac{2}{nT} \sum_{k=1}^K \sum_{i \in \mathcal{D}_k} \sum_{t=1}^T\sum_{a=0}^1  (-1)^{a+1} \mathbb{I}(A_{i,t}=a) \left[
		\widehat{\omega}_{t,-k}^{a,m} (S_{i,t})-\omega_t^{a,m}(S_{i,t}) \right] \varepsilon_{i,t}^a 
		 \right\rbrace^2 \\
		&
		\leq 
		\frac{4K}{n^2T^2} \sum_{k=1}^K \Mean  \left\lbrace  \sum_{i \in \mathcal{D}_k}
		\sum_{t=1}^T\sum_{a=0}^1   \left[
		\widehat{\omega}_{t,-k}^{a,m} (S_{i,t})-\omega_t^{a,m}(S_{i,t}) \right] \varepsilon_{i,t}^a 
		\right\rbrace^2 \\
		& =	\frac{4K}{n^2T^2} \sum_{k=1}^K |\mathcal{D}_k| \Var\left\lbrace
			\sum_{t=1}^T\sum_{a=0}^1   \left[
		\widehat{\omega}_{t,-k}^{a,m} (S_{t})-\omega_t^{a,m}(S_{t}) \right] \varepsilon_{t}^a 
		 \right\rbrace \\
         &+ \frac{4K}{n^2T^2} \sum_{k=1}^K \left\lbrace
		 \Mean \left[ 
		 \sum_{i \in \mathcal{D}_k}
		 \sum_{t=1}^T\sum_{a=0}^1   \left[
		 \widehat{\omega}_{t,-k}^{a,m} (S_{i,t})-\omega_t^{a,m}(S_{i,t}) \right] \varepsilon_{i,t}^a 
		 \right]  \right\rbrace^2.
	\end{split}
\end{equation*}
Observe that the conditional mean of the temporal difference error is zero, given any state-action pair. Using Cauchy-Schwarz inequality again, we have 
\begin{equation*}
	\begin{aligned}
		\Mean J_1^2 &\le \frac{4K}{n^2T^2} \sum_{k=1}^K |\mathcal{D}_k| \Var\left\lbrace
		\sum_{t=1}^T\sum_{a=0}^1   \left[
		\widehat{\omega}_{t,-k}^{a,m} (S_{t})-\omega_t^{a,m}(S_{t}) \right] \varepsilon_{t}^a 
		\right\rbrace  \\
        &\le \frac{4K}{n^2T} \sum_{k=1}^K |\mathcal{D}_k|  \sum_{t=1}^T  \Mean \left\lbrace \sum_{a=0}^1   \left[
		\widehat{\omega}_{t,-k}^{a,m} (S_{t})-\omega_t^{a,m}(S_{t}) \right] \varepsilon_{t}^a  \right\rbrace^2 \\
		&\leq \frac{8K}{n^2T} \sum_{k=1}^K |\mathcal{D}_k|  \sum_{t=1}^T  \sum_{a=0}^1  \Mean \left\lbrace  \left[
		\widehat{\omega}_{t,-k}^{a,m} (S_{t})-\omega_t^{a,m}(S_{t}) \right] \varepsilon_{t}^a  \right\rbrace^2.
	\end{aligned}
\end{equation*}
Under Assumption \ref{asmp:boundedtd}, one can show that \( \sup_{ \substack{t,a} }|\varepsilon_t^a|=O( R_{\max})\). It follows that $\Mean (J_1^2)=O( n^{-1} R_{\max}^2  \textrm{err}_\omega^2)$. 

\textbf{Analysis of $J_2$}. Similar to the analysis of $J_1$, we can rewrite $J_2$ as 
\begin{equation*} 
\begin{split}
		& \frac{1}{nT}\sum_{k=1}^K \sum_{i \in \mathcal{D}_k}
	\left[ 
	\widehat{V}_{1,-k}^{1,m}(S_{i,1}) -V_1^1(S_{i,1})+V_{1}^0(S_{i,1})- \widehat{V}_{1,-k}^{0,m}(S_{i,1})
	\right] \\
	&+
	\frac{2}{nT}\sum_{k=1}^K \sum_{i \in \mathcal{D}_k} \sum_{t=1}^T\sum_{a=0}^1  (-1)^{a+1} \mathbb{I}(A_{i,t}=a)
	\omega_t^{a,m}(S_{i,t}) \left[
	\widehat{V}_{t+1,-k}^{a,m}(S_{i,t+1})-V_{t+1}^a(S_{i,t+1}) +V_{t}^a(S_{i,t})-\widehat{V}_{t,-k}^{a,m}(S_{i,t}) 
	\right] .
\end{split}
\end{equation*}
Using Cauchy-Schwarz inequality, its second moment can be similarly bounded by
\begin{equation}\label{eqn:diff_term2_1}
	\begin{split}
		& \frac{K}{n^2T^2}\sum_{k=1}^K  \Var \Biggl\{
		 \sum_{i \in \mathcal{D}_k}
	\Big[
		\left[ 
	\widehat{V}_{1,-k}^{1,m}(S_{i,1}) -V_1^1(S_{i,1})+V_{1}^0(S_{i,1})- \widehat{V}_{1,-k}^{0,m}(S_{i,1})
	\right]  \\
	&+ \sum_{t=1}^T\sum_{a=0}^1  2(-1)^{a+1} \mathbb{I}(A_{i,t}=a)
	\omega_t^{a,m}(S_{i,t}) \left[
	\widehat{V}_{t+1,-k}^{a,m}(S_{i,t+1})-V_{t+1}^a(S_{i,t+1}) +V_{t}^a(S_{i,t})-\widehat{V}_{t,-k}^{a,m}(S_{i,t}) 
	\right] 
	\Big]
		\Biggr\} \\
		& +\frac{K}{n^2T^2}\sum_{k=1}^K \Biggl\{
		\Mean \Big[
		 \sum_{i \in \mathcal{D}_k}
		\Big[
		\left[ 
		\widehat{V}_{1,-k}^{1,m}(S_{i,1}) -V_1^1(S_{i,1})+V_{1}^0(S_{i,1})- \widehat{V}_{1,-k}^{0,m}(S_{i,1})
		\right]  \\
		&+ \sum_{t=1}^T\sum_{a=0}^1  2(-1)^{a+1} \mathbb{I}(A_{i,t}=a)
		\omega_t^{a,m}(S_{i,t}) \left[
		\widehat{V}_{t+1,-k}^{a,m}(S_{i,t+1})-V_{t+1}^a(S_{i,t+1}) +V_{t}^a(S_{i,t})-\widehat{V}_{t,-k}^{a,m}(S_{i,t}) 
		\right] 
		\Big]
		\Biggr\}^2.
	\end{split}
\end{equation}
By the double robustness property, under correct specification of $w_t^{a,m}$, the squared bias term in \eqref{eqn:diff_term2_1} is equal to zero. Thus, it suffices to upper bound the variance term. With some calculations, we have
\begin{equation*}
	\begin{split}
	&	\frac{K}{n^2T^2} \sum_{k=1}^K |\mathcal{D}_k| \Var \Biggl\{
	\left[ 
	\widehat{V}_{1,-k}^{1,m}(S_{1}) -V_1^1(S_{1})+V_{1}^0(S_{1})- \widehat{V}_{1,-k}^{0,m}(S_{1})
	\right]  \\
	&+ \sum_{t=1}^T\sum_{a=0}^1  2(-1)^{a+1} \mathbb{I}(A_{t}=a)
	\omega_t^{a,m}(S_{t}) \left[
	\widehat{V}_{t+1,-k}^{a,m}(S_{t+1})-V_{t+1}^a(S_{t+1}) +V_{t}^a(S_{t})-\widehat{V}_{t,-k}^{a,m}(S_{t}) 
	\right] 
	\Biggr\}\\
	& \leq \frac{2K}{n^2T^2} \sum_{k=1}^K |\mathcal{D}_k| \Mean \left\lbrace  	\left[ 
	\widehat{V}_{1,-k}^{1,m}(S_{1}) -V_1^1(S_{1})+V_{1}^0(S_{1})- \widehat{V}_{1,-k}^{0,m}(S_{1})
	\right] \right\rbrace^2 \\
	&+ \frac{16K}{n^2T^2} \sum_{k=1}^K |\mathcal{D}_k| \sum_{a=0}^1    \Mean  \left\lbrace 
	\sum_{t=1}^T
	\omega_t^{a,m}(S_{t}) \left[
	\widehat{V}_{t+1,-k}^{a,m}(S_{t+1})-V_{t+1}^a(S_{t+1}) +V_{t}^a(S_{t})-\widehat{V}_{t,-k}^{a,m}(S_{t}) 
	\right] 
	\right\rbrace^2\\
	& \leq \frac{4K}{n^2T^2} \sum_{k=1}^K |\mathcal{D}_k| \sum_{a=0}^1 \Mean   	\left[ 
	\widehat{V}_{a,-k}^{a,m}(S_{1}) -V_1^a(S_{1})\right]^2 + \frac{32K}{n^2T} \sum_{k=1}^K |\mathcal{D}_k| \sum_{a=0}^1  \sum_{t=1}^T C_{\omega}^2  \Mean   
	 \left[
	\widehat{V}_{t,-k}^{a,m}(S_{t})-V_{t}^a(S_{t}) 
	\right]^2,
	\end{split}
\end{equation*}
where the two inequalities follow from the Cauchy-Schwarz inequality and $C_{\omega}$ denotes the upper bound of $w_t^{a,m}$, which is finite under Assumption 
\ref{asmp:boundtrans}. 

As such, by the definition of $\textrm{err}_v$, we have $\Mean(J_2^2)=O(n^{-1} R_{\max}^2  \textrm{err}_v^2)$.

\textbf{Analysis of $J_3$}. Similarly, using Cauchy-Schwarz inequality, $\mathbb{E} (J_3^2)$ can be upper bounded by
\begin{eqnarray}\label{eqnJ3}
\begin{split}
		& \frac{4K}{n^2T^2}\sum_{k=1}^K \Mean \left\lbrace
		\sum_{i\in \mathcal{D}_k}\sum_{t=1}^T\sum_{a=0}^1  [\widehat{\omega}^{a,m}_{t,-k}(S_{i,t})-\omega^{a,m}_t(S_{i,t})] 
 [\widehat{V}_{t+1,-k}^{a,m}(S_{i,t+1})-\widehat{V}_{t,-k}^{a,m}(S_{i,t})-V_{t+1}^a(S_{i,t+1})+V_t^a(S_{i,t})] \right\rbrace^2 \\
 & = \frac{4K}{n^2T^2}\sum_{k=1}^K |\mathcal{D}_k|  \Var \left\lbrace 
 \sum_{t=1}^T\sum_{a=0}^1  [\widehat{\omega}^{a,m}_{t,-k}(S_{t})-\omega^{a,m}_t(S_{t})] 
 [\widehat{V}_{t+1,-k}^{a,m}(S_{t+1})-\widehat{V}_{t,-k}^{a,m}(S_{t})-V_{t+1}^a(S_{t+1})+V_t^a(S_{t})] 
 \right\rbrace \\
 &+ \frac{4K}{n^2T^2}\sum_{k=1}^K \left\lbrace 
 |\mathcal{D}_k|^2 \Mean \left[
  \sum_{t=1}^T\sum_{a=0}^1  [\widehat{\omega}^{a,m}_{t,-k}(S_{t})-\omega^{a,m}_t(S_{t})] 
 [\widehat{V}_{t+1,-k}^{a,m}(S_{t+1})-\widehat{V}_{t,-k}^{a,m}(S_{t})-V_{t+1}^a(S_{t+1})+V_t^a(S_{t})] 
  \right] 
 \right\rbrace^2.
\end{split}
\end{eqnarray}
Here, the squared bias term on the third line is the dominating factor. Applying Cauchy-Schwarz inequality again, it can be upper bounded by $O(R_{\max}^2 \textrm{err}_\omega^2  \textrm{err}_v^2)$. Meanwhile, using similar arguments to the analysis of $J_2$, we can show that the second line can be upper bounded by $O(n^{-1} R_{\max}^2 \textrm{err}_v^2)$. Thus, 
$\Mean[J_3^2]=O(R_{\max}^2\textrm{err}_\omega^2\textrm{err}_v^2 +n^{-1} R_{\max}^2 \textrm{err}_v^2)$. 

To summarize, we have show that $\Mean J_1^2+\Mean J_2^2 +\Mean J_3^2$ is of the order
\begin{equation*}
\begin{split}
		O\Big[\max\Big(\frac{ R_{\max}^2 \textrm{err}_\omega^2  }{{n}}, \frac{ R_{\max}^2 \textrm{err}_v^2  }{{n}}, R_{\max}^2 \textrm{err}_\omega^2  \textrm{err}_v^2 \Big)
	\Big].
\end{split}
\end{equation*}
Using Cauchy-Schwarz inequality, we obtain that $\Mean (J_1+J_2+J_3)^2$ is of the same order of magnitude. Notice that
\begin{eqnarray*}
    \textrm{MSE}(\textrm{ATE}_{\textrm{DRL}}^{(m)})=\textrm{MSE}(\textrm{ATE}_{\textrm{DRL}}^{(m)*})+ \Mean (J_1+J_2+J_3)^2+2\Mean (\textrm{ATE}_{\textrm{DRL}}^{(m)*} (J_1+J_2+J_3)),
\end{eqnarray*}
where the absolute value of the last term can be upper bounded by 
\begin{eqnarray*}
    \sqrt{\Mean (J_1+J_2+J_3)^2 \Mean (\textrm{ATE}_{\textrm{DRL}}^{(m)*})^2}=O\Big[\max\Big(\frac{ R_{\max}^2 \textrm{err}_\omega  }{{n}}, \frac{ R_{\max}^2 \textrm{err}_v  }{{n}}, \frac{R_{\max}^2 \textrm{err}_\omega  \textrm{err}_v}{\sqrt{n}} \Big)
	\Big]
\end{eqnarray*}
As the error terms $\textrm{err}_v$ and $\textrm{err}_w$ are $o(n^{-1/4})$ and hence bounded, combining these results together yields the desired assertion.  
\end{proof}

\textit{\textbf{Proof of Theorem \ref{thm::main}-LSTD}}. 
\begin{proof}
The proof of Theorem \ref{thm::main}-LSTD relies on demonstrating that the LSTD estimator can be expressed as a DRL estimator (whose detailed form is given in Equation (\ref{eqn:LSTDDRL}) below), allowing us to apply the proof techniques from Theorem \ref{thm::main}-DRL to establish Theorem \ref{thm::main}-LSTD. We break the rest of the proof into two parts. In Part I, we establish the equivalence between the two estimators. In Part II, our aim is to calculate the order of the difference between the RL estimator $\textrm{ATE}^{(m)}_{\textrm{DRL}}$ and its ``oracle'' version $\textrm{ATE}^{(m)*}_{\textrm{DRL}}$, and then analyze its order of magnitude.

\textit{\textbf{Part I}}. Recall that the value function estimator obtained via LSTD is given by $\varphi_t^\top(s) \widehat{\theta}_{t,a,m}$. The DRL estimator, whose equivalence to LSTD we aim to establish, uses this estimated value along with a linearly parameterized $\widehat{\omega}_t^{a,m}(s)$. Specifically, for each $t\ge 1$, we set $\widehat{\omega}_t^{a,m}(s)=\varphi_t^\top(s)\widehat{\alpha}_{t,a,m}$ where the estimators $\{\widehat{\alpha}_{t,a,m}\}_t$ are computed in a forward manner. By definition
\begin{eqnarray}\label{eqn:omega1}
\omega_{1}^{a,m}(s)=\frac{p_1(s)}{p_1^m(s|a)}=\frac{p_1(s)\mathbb{P}(A_1=a)}{b_1(a|s)p_1(s)}=\frac{1}{2b_1(a|s)},
\end{eqnarray}
where $b_1$ denotes the propensity score whose oracle value is a constant function equal to 0.5 by design. To estimate $\alpha_{1,a,m}$, notice that according to \eqref{eqn:omega1}, $\Mean [\omega_{1}^{a,m}(S_1)\mathbb{I}(A_1=a)h(S_1)-h(S_1)]=0$ for any function $h$. This motivates us to compute $\widehat{\alpha}_{1,a,m}$ by solving the following estimating equation,
\begin{eqnarray*}
\frac{1}{n}\sum_{i=1}^n \Big[\varphi_1^\top(S_{i,1}) \widehat{\alpha}_{1,a,m}\mathbb{I}(A_{i,1}=a)\varphi_1(S_{i,1}) -\varphi_1(S_{i,1})\Big] =0,
\end{eqnarray*}
which yields
\begin{eqnarray}\label{eqn:alpha1}
\widehat{\alpha}_{1,a,m}=\Big[\underbrace{\frac{1}{n}\sum_{i=1}^n \varphi_1(S_{i,1})\varphi_1^\top(S_{i,1})\mathbb{I}(A_{i,1}=a)}_{\widehat{\Sigma}_{1}^a}\Big]^{-1}\Big[\frac{1}{n}\sum_{i=1}^n \varphi_1(S_{i,1})\Big].
\end{eqnarray}
Next, consider a given $t\ge 2$. Using the Bayes rule, it is straightforward to show that
\begin{eqnarray*}
\omega_t^{a,m}(s)=\frac{p_t^a(s)}{p_t^m(s|a)}=\frac{p_t^a(s)\mathbb{P}(A_t=a)}{p_t^m(a,s)}=\frac{p_t^a(s)}{2p_t^m(a,s)},
\end{eqnarray*}
where $p_t^m(a,s)$ denote the join distribution function of $(A_t,S_t)$. Consequently, using the Bayes rule again, we obtain that
\begin{eqnarray*}
\Mean [\omega_{t-1}^{a,m}(S_{t-1})|A_{t-1}=a,S_{t}]=\frac{p_t^a(S_t)}{2p_{t}^m(a,S_t)},
\end{eqnarray*}
for any $a$ and $S_t$. This motivates us to estimate $\alpha_{t,a,m}$ by solving the following estimating equation,
\begin{eqnarray*}
\frac{1}{n}\sum_{i=1}^n\Big[\varphi_{t-1}^\top(S_{i,t-1})\widehat{\alpha}_{t-1,a,m}\mathbb{I}(A_{i,t-1}=a)\varphi_t(S_{i,t})
\Big]=\frac{1}{n}\sum_{i=1}^n\Big[\varphi_{t}^\top(S_{i,t})\widehat{\alpha}_{t,a,m}\mathbb{I}(A_{i,t}=a)\varphi_t(S_{i,t})
\Big],
\end{eqnarray*}
which yields that
\begin{eqnarray}\label{eqn:alphat}
\begin{split}        \widehat{\alpha}_{t,a,m}=\Big[\underbrace{\frac{1}{n}\sum_{i=1}^n \varphi_{t}(S_{i,t})\varphi_{t}^\top(S_{i,t})\mathbb{I}(A_{i,t}=a)}_{\widehat{\Sigma}_t^a} \Big]^{-1}\Big[\underbrace{\frac{1}{n}\sum_{i=1}^n \varphi_{t}(S_{i,t})\varphi_{t-1}^\top(S_{i,t-1})\mathbb{I}(A_{i,t-1}=a)}_{\widehat{\Sigma}_{t,t-1}^a} \Big]\widehat{\alpha}_{t-1,a,m},
\end{split}
\end{eqnarray}
and hence 
\begin{eqnarray*}
\widehat{\alpha}_{t,a,m}=(\widehat{\Sigma}_t^a)^{-1}\widehat{\Sigma}_{t,t-1}^a(\widehat{\Sigma}_{t-1}^a)^{-1}\widehat{\Sigma}_{t-1,t-2}^a \ldots (\widehat{\Sigma}_1^a)^{-1} \Big[\frac{1}{n}\sum_{i=1}^n \varphi_1(S_{i,1})\Big].
\end{eqnarray*}
This leads to the following DRL estimator
\begin{eqnarray}\label{eqn:LSTDDRL}
\textrm{ATE}^{(m)}_{\textrm{DRL}}=\frac{1}{nT}\sum_{i=1}^n \psi(\{S_{i,t},A_{i,t},R_{i,t}\};\{\widehat{V}_t^{a,m}\}_{t,a};\{\widehat{\omega}_t^{a,m}\}_{t,a}).
\end{eqnarray}
We aim to show that this DRL estimator equals the LSTD estimator based on $\{\widehat{V}_t^{a,m}\}_{t,a}$. According to the definition of the estimating function $\psi$, the DRL estimator can be naturally decomposed into two parts where the first part coincides the LSTD estimator and the second part corresponds to the augmentation term,
\begin{eqnarray*}
\frac{2}{nT}\sum_{i=1}^n\sum_{t=1}^T\sum_{a=0}^1 (-1)^{a+1} \mathbb{I}(A_{i,t}=a)\widehat{\omega}_t^{a,m}(S_{i,t})[R_{i,t}+\widehat{V}_{t+1}^{a,m}(S_{i,t+1})-\widehat{V}_t^{a,m}(S_{i,t})]. 
\end{eqnarray*}
As such, it suffices to show the augmentation term equals zero, or
\begin{eqnarray}\label{eqn:aug}
\frac{1}{n}\sum_{i=1}^n\mathbb{I}(A_{i,t}=a)\widehat{\omega}_t^{a,m}(S_{i,t})[R_{i,t}+\widehat{V}_{t+1}^{a,m}(S_{i,t+1})-\widehat{V}_t^{a,m}(S_{i,t})]=0,
\end{eqnarray}
for each $a$ and $t$. By \eqref{eqn:alphat}, the left-hand-side (LHS) of \eqref{eqn:aug} equals
\begin{eqnarray*}
(\widehat{\alpha}_{t-1,a,m})^\top (\widehat{\Sigma}_{t,t-1}^{a})^{\top} \Big\{\frac{1}{n}\sum_{i=1}^n
( \widehat{\Sigma}_t^{a})^{-1}\mathbb{I}(A_{i,t}=a)\varphi_t(S_{i,t})[R_{i,t}+\widehat{V}_{t+1}^{a,m}(S_{i,t+1})-\widehat{V}_t^{a,m}(S_{i,t})]\Big\}.
\end{eqnarray*}
According to the LSTD algorithm, it is immediate to see that the statistic inside the curly brackets equals zero. This formally verifies \eqref{eqn:aug}. The proof of Part I is hence completed. 

\textit{\textbf{Part II}}. In Part II, we aim to show \eqref{eqn:LSTDDRL} is asymptotically equivalent to its oracle estimator 
\begin{eqnarray*}
\textrm{ATE}^{(m)*}_{\textrm{DRL}}=\frac{1}{nT}\sum_{i=1}^n \psi(\{S_{i,t},A_{i,t},R_{i,t}\};\{V_t^a\}_{t,a};\{\omega_t^{a,m}\}_{t,a}).
\end{eqnarray*}
The proof is very similar to that of Theorem \ref{thm::main}-DRL; for brevity, we provide only a sketch. The main difference lies in the absence of cross-splitting for this DRL estimator.

Similarly, the difference between \(\textrm{ATE}^{(m)}_{\textrm{DRL}}\) and \(\textrm{ATE}^{(m)*}_{\textrm{DRL}}\) can be expressed as the sum of the following three terms:
\begin{eqnarray*}
J_4:=\frac{1}{nT}\sum_{i=1}^n [\psi(\{S_{i,t},A_{i,t},R_{i,t}\};\{V_t^a\}_{t,a};\{\widehat{\omega}_t^{a,m}\}_{t,a})-\psi(\{S_{i,t},A_{i,t},R_{i,t}\};\{V_t^a\}_{t,a};\{\omega_t^{a,m}\}_{t,a})].\\
J_5:=\frac{1}{nT}\sum_{i=1}^n [\psi(\{S_{i,t},A_{i,t},R_{i,t}\};\{\widehat{V}_t^{a,m}\}_{t,a};\{\omega_t^{a,m}\}_{t,a})-\psi(\{S_{i,t},A_{i,t},R_{i,t}\};\{V_t^a\}_{t,a};\{\omega_t^{a,m}\}_{t,a})].\\
\label{eqn:diffthirdterm}
J_6:=\frac{2}{nT}\sum_{i=1}^n\sum_{t=1}^T\sum_{a=0}^1  (-1)^{a+1} \mathbb{I}(A_{i,t}=a)[\widehat{\omega}^{a,m}_t(S_{i,t})-\omega^{a,m}_t(S_{i,t})]\\ \nonumber
\times [\widehat{V}_{t+1}^{a,m}(S_{i,t+1})-\widehat{V}_t^{a,m}(S_{i,t})-V_{t+1}^a(S_{i,t+1})+V_t^a(S_{i,t})].
\end{eqnarray*}
\textbf{Analysis of} $J_4$. We further decompose $J_4$ into the sum of 
\begin{equation}\label{eqn:difffourth}
\frac{1}{nT}\sum_{i=1}^n [\psi(\{S_{i,t},A_{i,t},R_{i,t}\};\{V_t^a\}_{t,a};\{\omega_t^{a,m*}\}_{t,a})-\psi(\{S_{i,t},A_{i,t},R_{i,t}\};\{V_t^a\}_{t,a};\{\omega_t^{a,m}\}_{t,a})],
\end{equation}
and
\begin{equation}\label{eqn:difffifth}
\frac{1}{nT}\sum_{i=1}^n [\psi(\{S_{i,t},A_{i,t},R_{i,t}\};\{V_t^a\}_{t,a};\{\widehat{\omega}_t^{a,m}\}_{t,a})-\psi(\{S_{i,t},A_{i,t},R_{i,t}\};\{V_t^a\}_{t,a};\{\omega_t^{a,m*}\}_{t,a})],
\end{equation}
where $\omega_t^{a,m*}(s)=\varphi_t^\top(s) \alpha_{t,a,m}^*$, and $\alpha_{t,a,m}^*:=\min_{\alpha_{t,a,m} \in \mathbb{R}^L }\Mean (
\omega_t^{a,m}(S_t) -\varphi_t^\top(S_t) \alpha_{t,a,m}
)^2$ .  Then \eqref{eqn:difffourth} is equal to
\begin{equation*}
	J_{4a}:=\frac{2}{nT} \sum_{i=1}^n \sum_{t=1}^T\sum_{a=0}^1  (-1)^{a+1} \mathbb{I}(A_{i,t}=a) \left[
	{\omega}_{t}^{a,m*} (S_{i,t})-\omega_t^{a,m}(S_{i,t}) \right] \varepsilon_{i,t}^a ,
\end{equation*}
whereas \eqref{eqn:difffifth} is equal to
\begin{equation*}
	\begin{aligned}
		J_{4b}&:=\frac{2}{nT} \sum_{i=1}^n \sum_{t=1}^T\sum_{a=0}^1  (-1)^{a+1} \mathbb{I}(A_{i,t}=a) \left[
		\widehat{\omega}_{t}^{a,m} (S_{i,t})-\omega_t^{a,m*}(S_{i,t}) \right] \varepsilon_{i,t}^a.
	\end{aligned}
\end{equation*}
Note that $\Mean(J_{4a})=0$ because the conditional mean of the temporal difference error is equal to zero, given any state-action pair.  Thus, we obtain 
\begin{equation*}
	\begin{split}
		\mathbb{E}(J_{4a})^2&\leq \frac{4}{nT^2}\Var \left\lbrace \sum_{t=1}^T
		\sum_{a=0}^1 \left[
		{\omega}_{t}^{a,m*} (S_{t})-\omega_t^{a,m}(S_{t}) \right] \varepsilon_{t}^a
		 \right\rbrace 
		  \leq O\Big(\frac{8R_{\max}^2}{nT}\Big) \sum_{a=0}^1\sum_{t=1}^T \Mean  
		  \left[
		 {\omega}_{t}^{a,m*} (S_{t})-\omega_t^{a,m}(S_{t}) \right]^2,
	\end{split}
\end{equation*}
where the second inequality follows from the Cauchy-Schwarz inequality and that $\max_{t,a}|\varepsilon_t^a|=O(R_{\max})$. 
This leads to $ \Mean(J_{4a}^2)=
O(n^{-1}R_{\max}^2 \textrm{app\_err}_{\omega}^2)
$, 
where we denote $\textrm{app\_err}_{\omega}^2$ by $\max_{a,m,t} \Mean |\omega_t^{a,m}(S_t)-\varphi_t^\top(S_t)\alpha_{t,a,m}^*|^2 $. 

It remains to bound $\Mean(J_{4b}^2)$. By definition, we can rewrite $J_{4b}$ as follows: 
\begin{equation*}
	 \frac{2}{T}\sum_{t=1}^T\sum_{a=0}^1 (-1)^{a+1}(\widehat{\alpha}_{t,a,m}-\alpha_{t,a,m}^*)^\top\Big\{\frac{1}{n}\sum_{i=1}^n \varphi_t(S_{i,t})\mathbb{I}(A_{i,t}=a) \varepsilon_{i,t}^a \Big\}.
\end{equation*}
Observe that the conditional mean of the temporal difference error is zero, given any state-action pair. Consequently, the random vector enclosed within the curly brackets have a mean of
zero. Additionally, under Assumption \ref{asmp:boundedtd}, each
summand within the curly brackets is bounded by $O(n^{-1}R_{\max})$. Meanwhile, under Assumption \ref{asmp:sievebasis}-(i), we have $\max_t \|\mathbb{E} [\varphi_t(S_t)\varphi_t^\top(S_t)]\|_2=O(1)$.   Consequently, the variance of each summand within the curly
brackets is bounded by $O(n^{-2} R_{\max}^2)$.  It follows from Bernstein's inequality \citep[see e.g.,][]{van1996weak} that with probability at least $1-O(n^{-\kappa}T^{-\kappa})$ for some sufficiently large $\kappa>0$, all elements in the curly brackets are upper bounded by
$O(n^{-1/2} R_{\max}\sqrt{\log(nT)})$. As such, we have $$	\Mean(J_{4b}^2) = O\Big(\frac{ R_{\max}^2 \log(nT) \textrm{est\_err}_{\omega}^2}{n}\Big), $$ 
where we denote $\textrm{est\_err}_{\omega}^2 ={\max_{a,t,m} \Mean \| \widehat{\alpha}_{t,a,m}-\alpha_{t,a,m}^*\|_2^2} $. 

Finally, note that the definition of $\alpha_{t,a,m}^*$ yields that the approximation error $\omega_t^{a,m}(S_t)-\omega_t^{a,m*}(S_t)$ is uncorrelated with the estimation error $\omega_t^{a,m*}(S_t)-\widehat{\omega}_t^{a,m}(S_t)$. Additionally, the non-singularity assumption in Assumption \ref{asmp:sievebasis}-(ii) yields that the second moment of $\omega_t^{a,m*}(S_t)-\widehat{\omega}_t^{a,m}(S_t)$ is of the same order to that of $\widehat{\alpha}_{t,a,m}-\alpha_{t,a,m}^*$. Consequently, we have
$$\Mean(J_4^2)\le 2\Mean J_{4a}^2+2\Mean J_{4b}^2 =O( \frac{ R_{\max}^2 \log(nT) [\textrm{app\_err}_{\omega}^2 +\textrm{est\_err}_{\omega}^2 ]}{n} )=O(\frac{R_{\max}^2 \log(nT) \textrm{err}_{\omega}^2}{n}).$$

\textbf{Analysis of $J_5$}. Similar to the analysis of $J_4$, it can be shown that $$\Mean(J_5^2)=O\Big(\frac{R_{\max}^2 \log(nT) \textrm{err}_{v}^2}{n}\Big).$$

\textbf{Analysis of $J_6$}. Define $J_{6}^*$ as
\begin{eqnarray*}
\small
\begin{split}
        \frac{2}{T}\sum_{t=1}^T\sum_{a=0}^1 (-1)^{a+1}\mathbb{E} \Big\{\mathbb{I}(A_t=a)[\widehat{\omega}_t^{a,m}(S_t)-\omega_t^{a,m}(S_t)][\widehat{V}_{t+1}^{a,m}(S_{t+1})-\widehat{V}_t^{a,m}(S_t)-V_{t+1}^a(S_{t+1})+V_t^a(S_t)]|\widehat{\omega}_t^{a,m},\widehat{V}_t^{a,m},\widehat{V}_{t+1}^{a,m}\Big\}.
\end{split}
\end{eqnarray*}
Notice that $J_{6}^*$ can be viewed as the expected value of $J_6$ when the estimated IS ratio and value function are fixed.  

Using similar arguments in the analysis of $J_3$, it can be shown that $\Mean (J_{6}^*)^2=O(R_{\max}^2\textrm{err}_w^2\textrm{err}_v^2)$.

It remains to consider $J_6-J_6^*$, which we denote by $\bar{J}_6(\{\widehat{\omega}_t^{a,m}-\omega_t^{a,m}\}_{t,a,m}, \{\widehat{V}_t^{a,m}-V_t^a\}_{t,a,m})$, to highlight its dependence upon the estimated IS ratio and value function as well as their ground truths. This term can be decomposed into $\bar{J}_{6a}+\bar{J}_{6b}+\bar{J}_{6c}+\bar{J}_{6d}$ where
\begin{eqnarray*}
    \bar{J}_{6a}&=&\bar{J}_6(\{\omega_t^{a,m*}-\omega_t^{a,m}\}_{t,a,m}, \{V_t^{a,m*}-V_t^a\}_{t,a,m}),\\
    \bar{J}_{6b}&=&\bar{J}_6(\{\omega_t^{a,m*}-\omega_t^{a,m}\}_{t,a,m}, \{\widehat{V}_t^{a,m}-V_t^{a,m*}\}_{t,a,m}),\\
    \bar{J}_{6c}&=&\bar{J}_6(\{\widehat{\omega}_t^{a,m}-\omega_t^{a,m*}\}_{t,a,m}, \{V_t^{a,m*}-V_t^{a}\}_{t,a,m}),\\
    \bar{J}_{6d}&=&\bar{J}_6(\{\widehat{\omega}_t^{a,m}-\omega_t^{a,m*}\}_{t,a,m}, \{\widehat{V}_t^{a,m}-V_t^{a,m*}\}_{t,a,m}).
\end{eqnarray*}
We note that:
\begin{itemize}
    \item $\bar{J}_{6a}$ can be represented by an average of i.i.d. mean zero random variable. Using similar arguments to those for the second line of \eqref{eqnJ3}, we can show that $\Mean (\bar{J}_{6a}^2)$ is of the order $O(n^{-1}R_{\max}^2 \text{app\_err}_v^2)$. 
    \item Similar to the analysis of $J_{4b}$, we can represent $\bar{J}_{6b}$ as an average $T^{-1}\sum_{t=1}^T \bar{J}_{6b}^{(t)}$ where each $\bar{J}_{6b}^{(t)}$ can be represented by the inner product of an average of i.i.d. mean zero random vector and $(\widehat{\theta}_{t,a,m}-\theta_{t,a,m}^*)$. The bound $\Mean (\bar{J}_{6b}^2)=O(n^{-1}R_{\max}^2\log (nT) \textrm{est\_err}_v^2)$ follows similarly to that of $\Mean (J_{4b}^2)$.
    \item $\bar{J}_{6c}$ can be analyzed very similarly to $\bar{J}_{6b}$, leading to $\Mean \bar{J}_{6c}^2=O(n^{-1}R_{\max}^2 \log(nT)  \textrm{app\_err}_v^2)$. 
    \item Finally, we can express $\bar{J}_{6d}$ as an average $T^{-1}\sum_{t=1}^T \bar{J}_{6d}^{(t)}$ where each $\bar{J}_{6d}^{(t)}$ is given by 
    \begin{eqnarray*}
        (\widehat{\theta}_{t,a,m}-\theta_{t,a,m}^*)^\top \Big(\frac{1}{n}\sum_{i=1}^n \bm{M}_i\Big) (\widehat{\alpha}_{t,a,m}-\alpha_{t,a,m}^*),
    \end{eqnarray*}
    where $\bm{M}_i$s are i.i.d. mean zero matrices. Based on this identity, one can show that $\Mean (\bar{J}_{6d}^2)$ is of the order $O(n^{-1}R_{\max}^2\log (nT) \textrm{est\_err}_v^2)$, using similar arguments to those for the analysis of $\bar{J}_{6b}$.  
\end{itemize}
Combining these results, one can apply arguments similar to those for the analysis of $J_4$ to show that $\Mean (J_6-J_6^*)^2=O(n^{-1} R_{\max}^2 \log(nT) \textrm{err}_v^2)$. Together with $\Mean (J_{6}^*)^2=O(R_{\max}^2\textrm{err}_w^2\textrm{err}_v^2)$ and the Cauchy-Schwarz inequality, we obtain that
\begin{eqnarray*}
    \Mean (J_6)^2=O\Big(\frac{R_{\max}^2 \log(nT)\textrm{err}_v^2}{n}\Big)+O(R_{\max}^2\textrm{err}_v^2\textrm{err}_{\omega}^2).
\end{eqnarray*}
The rest of the proof follows similarly to that of DRL. 
\end{proof}

\end{document}